%% file: neurips_2025.tex
\documentclass{article}

\usepackage[square,numbers,compress]{natbib}
\PassOptionsToPackage{numbers, compress}{natbib}

\usepackage{color,colortbl}

\usepackage[dvipsnames]{xcolor}
\definecolor{mygreen}{HTML}{3FBC9D}

    \usepackage[final]{neurips_2025}

\usepackage[utf8]{inputenc} 
\usepackage[T1]{fontenc}    
\usepackage{hyperref}       
\usepackage{url}            
\usepackage{booktabs}       
\usepackage{subcaption}     
\usepackage{amsfonts}       
\usepackage{nicefrac}       
\usepackage{microtype}      
\usepackage{amsmath, amssymb, amsthm}
\newtheorem{theorem}{Theorem}

\usepackage{graphicx}
\usepackage{enumitem}
\usepackage{multirow}
\usepackage{wrapfig}
\usepackage{mathtools}
\usepackage{nccmath}
\usepackage{algorithm}
\usepackage{caption}
\usepackage{listings}
\usepackage{tcolorbox}

\usepackage[algo2e, ruled]{algorithm2e}
\SetKwComment{Comment}{$\triangleright$\ }{}
\SetKwProg{Init}{Initialize}{}{}

\usepackage{pifont}
\newcommand{\cmark}{{\color{mygreen}\ding{51}}}
\newcommand{\xmark}{{\color{gray}\ding{55}}}
\usepackage{array}      
\newcommand{\rot}[1]{\rotatebox[origin=l]{30}{#1}}

\newcommand{\methodshort}[1]{\textsc{Mingle}}
\newcommand{\method}[1]{\textsc{Mixture of Null-Space Gated Low-Rank Experts}}

\title{\methodshort{}: Mixture of Null-Space Gated Low-Rank Experts for Test-Time Continual Model Merging}

\author{
  Zihuan Qiu$^{1}$\quad Yi Xu$^{2}$\quad Chiyuan He$^{1}$\quad Fanman Meng$^{1}$\thanks{Corresponding author}\\
  \textbf{Linfeng Xu$^{1}$}\quad  \textbf{Qingbo Wu$^{1}$}\quad  \textbf{Hongliang Li$^{1}$} \\
  $^{1}$University of Electronic Science and Technology of China, Chengdu, China \\
  $^{2}$Dalian University of Technology, Dalian, China \\
  {\small\{zihuanqiu@std.,\;cyhe@std.,\;fmmeng@,\;lfxu@,\;qbwu@,\;hlli@\}uestc.edu.cn,\; yxu@dlut.edu.cn}
}

\begin{document}

\maketitle

\begin{abstract}
Continual model merging integrates independently fine-tuned models sequentially without access to the original training data, offering a scalable and efficient solution for continual learning.
However, existing methods face two critical challenges: parameter interference among tasks, which leads to catastrophic forgetting, and limited adaptability to evolving test distributions.
To address these issues, we introduce the task of Test-Time Continual Model Merging (TTCMM), which leverages a small set of unlabeled test samples during inference to alleviate parameter conflicts and handle distribution shifts.
We propose \methodshort{}, a novel framework for TTCMM. \methodshort{} employs a mixture-of-experts architecture with parameter-efficient, low-rank experts, which enhances adaptability to evolving test distributions while dynamically merging models to mitigate conflicts.
To further reduce forgetting, we propose Null-Space Constrained Gating, which restricts gating updates to subspaces orthogonal to prior task representations, thereby suppressing activations on old tasks and preserving past knowledge. We further introduce an Adaptive Relaxation Strategy that adjusts constraint strength dynamically based on interference signals observed during test-time adaptation, striking a balance between stability and adaptability.
Extensive experiments on standard continual merging benchmarks demonstrate that \methodshort{} achieves robust generalization, significantly reduces forgetting, and consistently surpasses previous state-of-the-art methods by 7–9\% on average across diverse task orders. Our code is available at: \url{https://github.com/zihuanqiu/MINGLE}
\end{abstract}

\setlength{\intextsep}{2pt}
\setlength{\columnsep}{8pt}

\input{sections/introduction}

\input{sections/related_work}
\input{sections/method}

\input{sections/experiments}

\section{Conclusions}
In this work, we introduced the task of test-time continual model merging (TTCMM) and proposed \methodshort{}, a novel framework for TTCMM that integrates a mixture-of-experts architecture with adaptive null-space constrained gating.
Extensive empirical evaluations show that \methodshort{} substantially improves generalization and mitigates catastrophic forgetting, consistently outperforming prior state-of-the-art approaches. These results establish TTCMM as a principled paradigm for addressing both task interference and distribution shift, and highlight the practical potential of \methodshort{} for scalable and efficient continual learning in real-world applications.
\newpage

\paragraph{Acknowledgments}
This work was supported in part by National Science and Technology Major Project (2021ZD0112001), National Natural Science Foundation of China (No.62271119, 08120002, 62071086, U23A20286), the Key Research and Development Project of Hainan Province (Grant No. ZDYF2024(LALH)003), the Fundamental Research Funds for the Central University of China (DUT No. 82232031), the Natural Science Foundation of Sichuan Province under Grant 2025ZNSFSC0475.

We thank all reviewers for taking the time to review our paper and give valuable suggestions.

\bibliography{references}

\input{sections/checklist}

\input{sections/appendix}

\end{document}

%% file: sections/introduction.tex
\section{Introduction}

Continual learning aims to incrementally adapt machine learning models to new tasks without forgetting previously learned knowledge, addressing the critical challenge of catastrophic forgetting \cite{mccloskey1989catastrophic}. However, conventional continual learning approaches typically require continuous access to original training data, raising significant concerns about privacy and substantial computational overhead due to retraining efforts, thus limiting their applicability in dynamic, data-sensitive environments.

To address these limitations, recent works have explored an alternative paradigm known as continual model merging (CMM), which sequentially integrates independently fine-tuned models directly in parameter space, without revisiting any training data \cite{liu2023tangent, porrello2025a, tang2025merging}. CMM typically operates under a "merge-to-transfer" paradigm: given a pretrained model $\theta_0$ and independently fine-tuned models $\{\theta_t\}_{t=1}^T$, a unified model is constructed sequentially by combining task-specific weight updates $\Delta\theta_t = \theta_t - \theta_0$ via weighted averaging or projection-based strategies \cite{ilharco2023editing,wortsman2022robust,tang2024parameterefficient}.

Despite its advantages in scalability, data privacy, and distributed training capabilities \cite{mcmahan2017communication,fang2024decentralised,shen2023hugginggpt}, existing CMM methods still encounter critical issues, notably severe parameter interference between tasks and limited adaptability to evolving test distributions. This parameter interference arises because, as fine-tuned models are incrementally merged, overlapping or conflicting parameter updates accumulate, resulting in severe forgetting of previously learned tasks. To mitigate this interference, recent methods introduce structural constraints such as orthogonal projection \cite{tang2025merging, wang2023orthogonal}, model linearization \cite{liu2023tangent, tang2024parameterefficient}, and pruning-based sparsification \cite{yadav2023ties, yu2024language}. However, their effectiveness diminishes as task count grows and interference becomes increasingly entangled. Moreover, models merged across tasks often fail to generalize effectively, particularly when facing unseen or shifting task conditions. These flaws result in severe forgetting of earlier tasks and substantial performance gaps compared to the upper bound achieved by individually fine-tuned models. As shown in Fig.~\ref{fig:preliminary}, TA \cite{ilharco2023editing} suffers from large performance gaps and strong forgetting, reflected by low accuracy and negative backward transfer. OPCM \cite{tang2025merging} improves over TA via orthogonalized merging but still shows notable degradation.

\begin{wrapfigure}[10]{r}{0.45\textwidth}\vspace{-15pt}
\centering
\includegraphics[width=0.43\textwidth]{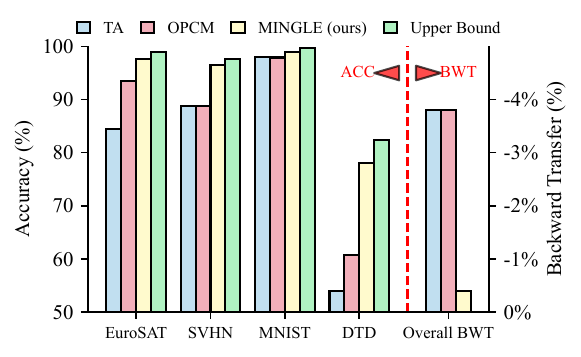}\vspace{-5pt}
\caption{After 8-task continual merging: accuracy on first four tasks and overall BWT.}
\label{fig:preliminary}
\end{wrapfigure}
To overcome these limitations, we propose a novel continual merging paradigm---\textbf{Test-Time Continual Model Merging (TTCMM)}---which explicitly introduces the concept of test-time adaptation (TTA) \cite{wang2020tent, song2023ecotta} into model merging. Unlike prior TTA-based multi-task merging methods \cite{yang2023adamerging,tang2024merging,RepresentationSurgery_ICML_2024}, which assume simultaneous availability of models and test data from all tasks, TTCMM utilizes only a small set of unlabeled samples from the current task, making it uniquely suited for realistic continual scenarios where revisiting historical data is often infeasible.

In this paper, we propose \methodshort{} (\textbf{MI}xture of \textbf{N}ull-Space \textbf{G}ated \textbf{L}ow-Rank \textbf{E}xperts), a method designed to continually merge independently fine-tuned models at test-time while preserving prior knowledge. \methodshort{} employs a mixture-of-experts architecture \cite{jacobs1991adaptive,mu2025comprehensive} composed of lightweight LoRA-based \cite{hu2022lora} experts, enabling efficient and flexible test-time adaptation. To robustly prevent interference from previously learned tasks, we introduce a novel \textbf{Null-Space Constrained Gating} mechanism, restricting gating updates to task-orthogonal subspaces. Additionally, we propose an \textbf{Adaptive Relaxation Strategy} to dynamically modulate constraint strength based on test-time interference feedback during adaptation.

Extensive experiments on standard continual learning benchmarks show that \methodshort{} consistently outperforms previous state-of-the-art approaches by 7--9\% on average, achieving robust generalization and strong resistance to catastrophic forgetting across diverse continual learning scenarios. Remarkably, these improvements are achieved entirely without any access to original training data, demonstrating the effectiveness of our TTCMM paradigm and the power of test-time adaptation in continual learning.

Our contributions are summarized as follows:\vspace{-5pt}
\begin{itemize}[leftmargin=20pt]
\item We formalize test-time continual model merging (TTCMM), a novel task that leverages unlabeled test samples to merge independently fine-tuned models.
\item We propose \methodshort{}, a TTCMM framework with Adaptive Null-Space Constrained Gating to effectively balance stability and plasticity.
\item Extensive experiments show that \methodshort{} achieves state-of-the-art performance, consistently outperforming prior methods in accuracy, robustness and resistance to forgetting.
\end{itemize}

%% file: sections/related_work.tex
\section{Related Work}
\noindent\textbf{Continual Learning.}
Continual learning (CL) seeks to mitigate catastrophic forgetting \cite{mccloskey1989catastrophic}, where learning new tasks overwrites prior knowledge. 
Regularization-based methods constrain updates with importance weights \cite{Kirkpatrick2016OvercomingCF,zenke2017continual,aljundi2018memory,jung2020continual,wu2024meta}, while distillation aligns outputs to preserve knowledge \cite{Hou2019LearningAU,Douillard2020PODNetPO,simon2021learning,qiu2023ism}. 
Replay methods store exemplars or generate surrogates with prompts, prototypes, or generators \cite{Rebuffi2016iCaRLIC,liu2021rmm,wang2022dualprompt,smith2023coda,qiu2024dual}, and dynamic architectures expand capacity via growth or ensembling \cite{lee2017overcoming,zhou2024expandable,marouf2024weightedensemblemodelsstrong}. 
Recent work leverages lightweight adapters or prompts in pre-trained models for efficient transfer \cite{yu2024boosting,huang2024class,wang2022learning}.  
Model merging offers an alternative route. Some methods remain close to conventional CL by sequentially fine-tuning and merging models to reduce forgetting \cite{marczak2024magmax,marouf2024weightedensemblemodelsstrong,fukuda2024adapter}, typically requiring training data. 
In contrast, continual model merging \cite{jin2023dataless,liu2023tangent,chitale2023task,porrello2025a,tang2025merging} merges \emph{independently fine-tuned models} without revisiting training data, enabling greater scalability and privacy.

\noindent\textbf{Model Merging.}
Early work merged models via direct parameter averaging \cite{utans1996weight,shoemake1985animating}, later refined by linear mode connectivity \cite{entezari2021role,ainsworth2022git}. 
Wortsman \textit{et al.}~\cite{wortsman2022robust} showed that weight averaging can also enhance robustness and out-of-distribution generalization. 
Task Arithmetic (TA) \cite{ilharco2023editing} views models as task vectors to be summed, but relies on weight disentanglement \cite{ortiz2023task}, often violated under standard fine-tuning, motivating structured training \cite{jin2023dataless,stoica2025model}. 
Beyond averaging, interference-aware methods reweight or sparsify parameters \cite{yadav2023ties,yu2024language}, or fuse models via distillation and clustering \cite{RepresentationSurgery_ICML_2024,wang2024lora}. 
LoRA-based tuning \cite{hu2022lora} introduces additional entanglement challenges, spurring gradient-free or retrieval-based strategies \cite{huang2024lorahub,zhao2024loraretriever,zhao2025merging}. 
More recently, dynamic merging \cite{tang2024merging,lu2024twin} adapts parameters conditioned on inputs, achieving higher flexibility and performance, but remains limited to multi-task fusion and unexplored in continual settings.

\begin{figure}[t]
    \centering
    \includegraphics[width=1\linewidth]{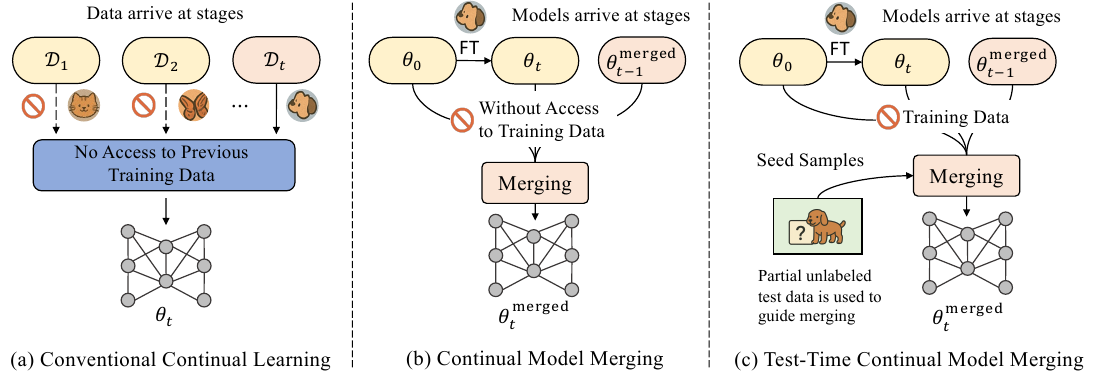}
    \caption{Comparison of three continual learning paradigms.
(a) Conventional Continual Learning trains models sequentially with data arriving in stages, without access to previous task data.
(b) Continual Model Merging continually fuses independently trained models, without access to any training data.
(c) Test-Time Continual Model Merging improves merging by leveraging a few unlabeled test samples from the current task.}\vspace{-10pt}
    \label{fig:paradigm}
\end{figure}

\noindent\textbf{Test-Time Adaptation.}
TTA adapts models at inference to mitigate distribution shift. Early approaches used self-supervised objectives \cite{sun2020test}, entropy minimization \cite{wang2020tent}, or regularized updates \cite{song2023ecotta}. Online TTA adapts continuously \cite{iwasawa2021test}, while batch-wise variants ignore temporal structure \cite{gao2023back}. To enhance stability, later work introduced confidence filtering \cite{niu2022efficient}, EMA \cite{dobler2023robust}, partial updates \cite{yuan2023robust}, test-time augmentation \cite{zhang2022memo}, and adaptive BatchNorm \cite{schneider2020improving}. For vision–language models, TTA often employs prompts or adapters \cite{shu2022test,feng2023diverse,lee2025ratta}. We draw on TTA to guide merging, aligning fused models with evolving test distributions.

\noindent\textbf{Relation to Prior Work.}
Most related to our work are MoE-Adapter \cite{yu2024boosting} and WEMOE \cite{tang2024merging}, both built on MoE architectures \cite{jacobs1991adaptive,jordan1994hierarchical}. MoE-Adapter follows conventional continual learning, embedding expert modules that are jointly trained across tasks. In contrast, we adopt a model-merging paradigm, where experts are extracted from independently fine-tuned models and inserted without further training. WEMOE incorporates test-time adaptation but targets multitask learning, assuming simultaneous access to all models and data. By contrast, \methodshort{} is tailored for the more challenging continual merging setting.

%% file: sections/method.tex
\section{\methodshort{}: Mixture of Null-Space Gated Low-Rank Experts}
\subsection{Preliminaries}

\paragraph{Problem Setting.}
We study continual learning in a model merging setting, where a sequence of task-specific models $\{\theta_1, \dots, \theta_T\}$ are independently fine-tuned from a shared pre-trained model $\theta_0$, each using a dataset $\mathcal{D}_i = \{(x_j^{(i)}, y_j^{(i)})\}$ with label space $\mathcal{C}_i \subset \mathcal{Y}$. The goal is to construct a unified model $\theta^{\text{merged}}_T$ that generalizes across the combined label space $\mathcal{C}_{1:T} = \bigcup_{i=1}^T \mathcal{C}_i$.

Unlike conventional continual learning, we assume \emph{no access} to training data during merging. All adaptation happens directly in parameter space. This paradigm is relevant in scenarios where only final fine-tuned models are retained, while original training data is discarded due to privacy, storage, or accessibility constraints.

To contextualize this, we compare with two related paradigms in Fig.~\ref{fig:paradigm}:\vspace{-5pt}
\begin{itemize}[leftmargin=20pt]
    \item \textbf{Conventional Continual Learning.} A single model $\theta$ is sequentially updated on $\mathcal{D}_1, \dots, \mathcal{D}_T$, discarding previous data. It requires direct training data access and extensive retraining.
    \item \textbf{Continual Model Merging.} A sequence of models are merged incrementally in parameter space without access to training data and earlier models: $\theta^{\text{merged}}_t = \text{Merge}(\theta^{\text{merged}}_{t-1}, \theta_t)$.
    \item \textbf{Test-Time Continual Model Merging.} An extension of the above where a small unlabeled subset \( \mathcal{D}^{\text{seed}}_t \subset \mathcal{D}^{\text{test}}_t \) (\textit{e.g.}, 5 samples per class) is available at each stage to provide lightweight task-specific guidance. We refer to \( \mathcal{D}^{\text{seed}}_t \) as the \emph{seed samples} of task $t$.
\end{itemize}

\noindent\textbf{Existing Continual Merging Strategies.}
Let $\theta_0$ denote the parameters of a pre-trained model. The corresponding task vector is defined as $\Delta\theta_t = \theta_t - \theta_0$.\vspace{-5pt}
\begin{itemize}[leftmargin=20pt]
    \item \textbf{Continual Task Arithmetic (\textsc{C. TA}).} A simple additive merge \cite{ilharco2023editing}: $\theta^{\text{merged}}_t = \theta^{\text{merged}}_{t-1} + \lambda \Delta\theta_t$, where $\lambda$ is a scalar. While training-free, it is sensitive to $\lambda$ and prone to task interference.
    
    \item \textbf{Orthogonal Projection-based Continual Merging (\textsc{OPCM}).} Tang et al.~\cite{tang2025merging} propose projecting each $\Delta\theta_t$ onto the orthogonal complement of previous directions:
    $\theta^{\text{merged}}_t = \theta_0 + \frac{1}{\lambda_t} \left[ \lambda_{t-1} \Delta\theta^{\text{merged}}_{t-1} + \mathcal{P}^{(t-1)}(\Delta\theta_t) \right]$, where $\mathcal{P}^{(t-1)}$ retains components orthogonal to previous updates. This reduces interference but ignores adaptation to task distributions.
\end{itemize}

To address these issues, we present \methodshort{}, which 
leverages $\mathcal{D}^{\text{seed}}_t$ to modulate the integration of $\theta_t$ at test-time, enhancing alignment to test distribution and mitigating task interference.

\subsection{Motivation and Theoretical Analysis}\label{sec:mingle_theory}
Most existing continual model merging methods combine fine-tuned models via static averaging, where each expert is assigned fixed coefficients,
thereby enforcing the \emph{same} mixing rule across the whole input space.
Consequently, it cannot specialize to regions where one expert is clearly
superior.  In contrast, a Mixture-of-Experts (MoE) equips every input with
a \emph{data-dependent} gate
$g(x)=(g_1(x),\dots,g_T(x))$
that selects or re-weights experts on-the-fly.
We give a formal comparison between static averaging and
dynamic MoE under a noisy-routing scenario.
\footnote{Symbols and proofs are deferred to App.~A}

\begin{theorem}[Dynamic MoE versus Static Averaging]\label{thm:main}
Let $\{(D_t,f_t)\}_{t=1}^T$ be $T$ independent tasks with priors
$P(t)$ and per-task risks $R_t(i)$.  
For any static mixture
$h_{\mathrm{static}}(x)=\sum_{i=1}^T\alpha_i\,f_i(x)$
and any hard-routed MoE
$h_{\mathrm{MoE}}(x)=f_{i^\star(x)}(x)$
with task-specific routing errors $\varepsilon_t$:
\begin{equation}
      R(h_{\mathrm{MoE}})
  \;=\;
  R_{\mathrm{ideal}}
  +\sum_{t=1}^T
    P(t)\,\varepsilon_t\bigl(R_{\text{wrong},t}-R_t(t)\bigr),
\end{equation}
where
$R_{\mathrm{ideal}}=\sum_tP(t)R_t(t)$ and
$R_{\text{wrong},t}=\frac{1}{T-1}\sum_{i\neq t}R_t(i)$.
Moreover,
\begin{enumerate}[leftmargin=*]
  \item \emph{(Perfect routing)}\;If $\varepsilon_t=0$ for all $t$, then
        $\displaystyle\inf_{g}R(h_{\mathrm{MoE}})
        \;<\;
        \inf_{\boldsymbol{\alpha}}R(h_{\mathrm{static}})$
        whenever at least two tasks disagree on their best expert.
  \item \emph{(Noisy routing)}\;
        If
        $\displaystyle
        \sum_{t}P(t)\varepsilon_t
        \bigl(R_{\text{wrong},t}-R_t(t)\bigr)
        <
        R^*_{\mathrm{static}}-R_{\mathrm{ideal}},$
        where $R^*_{\text{static}} = \inf_{\boldsymbol{\alpha}} R(h_{\text{static}})$,
        then the MoE still attains lower risk than any static mixture.
\end{enumerate}
\end{theorem}
The theory above motivates a design that (i)~keeps experts specialized and
(ii)~prevents interference between tasks.
Our \textbf{\methodshort{}} framework achieves both goals by combining\vspace{-5pt}
\begin{itemize}[leftmargin=20pt]
  \item \textbf{Low-rank experts} $f_t$ that capture task-specific
        variations with minimal parameters, and
  \item \textbf{Null-space constrained gating} that projects gradient updates
        away from subspaces spanned by previously activated features, keeping $\varepsilon_t$ small without harming earlier experts.
\end{itemize}

\subsection{Low-Rank Expert Mixture for Continual Model Merging}
We adopt MoE framework for continual model merging, in which each task \( i \)  is equipped with a low-rank expert \( f_i \) and an associated input-dependent gating function \( g_i \). These components are injected into the linear layers of the backbone (\textit{e.g.}, CLIP visual encoder). The gate \( g_i \) modulates expert activation based on the input features, allowing for fine-grained, localized task specialization.

\noindent\textbf{Mixture of Low-Rank Expert.}
When a new task $t$ arrives, a dedicated expert  \( f_t \)  and its gate \( g_t \) are appended to the model. 
The output of a given $l$-th layer\footnote{The layer index $l$ is omitted hereafter whenever it does not cause ambiguity.} can be formulated as follows:
\begin{equation} \label{eq:2}
\theta_t^{\text{merged},(l)}(X) = \theta_{t-1}^{\text{merged},(l)}(X) + g_t^{(l)}(X) \cdot f_t^{(l)}(X) = \theta_0^{(l)}(X) + \sum_{i=1}^{t} g_i^{(l)}(X) \cdot f_i^{(l)}(X) .
\end{equation}
where only the gate \( g_t \) is adaptable during testing, while all experts \( \{f_i\}_{i=1}^{t} \) and old gates \( \{g_i\}_{i=1}^{t-1} \) remain frozen to preserve prior knowledge.
To construct expert \( f_t \), we first project the task vector \(\Delta\theta_t\) onto the orthogonal complement of previously learned directions, following OPCM \cite{tang2025merging}:
\begin{equation}\label{eq:3}
    \mathcal{P}^{(t-1)}(\Delta\theta_t) 
    = \sum_{\substack{p,q=a , p\neq q}}^{m,n}\langle \Delta\theta_t, u_p^{(t-1)} v_q^{(t-1)\top} \rangle_F u_p^{(t-1)} v_q^{(t-1)\top},
\end{equation}
where \( u_p^{(t-1)} \) and \( v_q^{(t-1)} \) are the \(p\)-th and \(q\)-th singular vectors from the singular value decomposition (SVD) of previous experts $\sum_{i=1}^{t-1} f_i(X)$, and \(\alpha\) denotes the effective rank of previous experts.
This projection removes previously learned directions to mitigate interference. We then apply a rank-$r$ truncated SVD for $\mathcal{P}^{(t-1)}(\Delta\theta_t)$ to construct a low-rank expert \cite{hu2022lora}:
\begin{equation}\label{eq:4}
    f_t = BA = (\tilde{U} \tilde{\Sigma} )\tilde{V}^\top,
\end{equation}
where \( \tilde{U} \in \mathbb{R}^{d_1 \times r} \), \( \tilde
{\Sigma} \in \mathbb{R}^{r \times r} \), and \( \tilde{V} \in \mathbb{R}^{d_2 \times r} \), retaining the top $r$ singular components.
The resulting expert captures the principal directions while significantly reducing parameter overhead.

Each gating function is implemented as a linear projection:
\begin{equation}\label{eq:5}
g_t(X) = W_t^{(g)\top} X + b_t^{(g)}, 
\end{equation}
where \( W_t^{(g)} \in \mathbb{R}^{d \times 1} \) and \( b_t^{(g)} \in \mathbb{R} \) are \emph{learnable} parameters. The gating function is adapted at test time using a small number of unlabeled test data.

\noindent\textbf{Test-Time Adaptation.}\;To encourage the merged model to retain task-specific behavior, we minimize the Kullback–Leibler divergence between its prediction and that of the corresponding individual fine-tuned model $\theta_t$. We define the adaptation objective as:
\begin{equation}\label{eq:7}
L_t = \mathbb{E}_{x \sim \mathcal{D}^{\text{seed}}_t} \left[ \mathrm{KL}\left(p(x;\;\theta^{\text{merged}}_t) \,\|\, p(x;\;\theta_t)\right) \right].
\end{equation}
where $p(x;\theta)$ denotes the predictive distribution

\subsection{Adaptive Null-Space Constrained Gating for Interference Mitigation}

When merging models continually, the primary challenge of gating is to integrate new experts without disturbing prior task predictions. Consider two experts \( f_1, f_2 \) and their corresponding gates \( g_1, g_2 \). When evaluating on the first task domain \( X_1 \), the interference from \( g_2 \) can be quantified as:
\begin{equation}
    \xi(g_2) = \left\| g_1(X_1) \cdot f_1(X_1) + g_2(X_1) \cdot f_2(X_1) - f_1(X_1) \right\|^2.
\end{equation}
This measures the deviation introduced by \( g_2 \) on the domain where \( f_1 \) originally dominates. A desirable gating function should suppress \( g_2(X_1) \), ensuring predictions on \( X_1 \) remain unaffected. However, as \( X_1 \) becomes inaccessible after adaptation, this error becomes unobservable and cannot be minimized directly, resulting in prediction drift and catastrophic forgetting.

\noindent\textbf{Hard Null-Space Projection.}
After completing task~\(t\), we cache the \(l\)-th layer inputs in the seed buffer \(\mathcal{D}_t^{\text{seed}}\) and estimate their covariance \(\mathrm{Cov}_t^{(l)} \in \mathbb{R}^{d \times d}\). 
Applying truncated SVD yields the top-\(k\) dominant subspaces \(\tilde{U}_t^{(l)} \in \mathbb{R}^{d \times k}\). 
We then concatenate these with the subspaces from all previous tasks and orthonormalize:
\(
  U_{t}^{(l)} = \operatorname{orthonorm}\bigl[U_{t-1}^{(l)}|\tilde{U}_{t}^{(l)}\bigr]\in\mathbb{R}^{d\times tk}
\).
The hard projector is
\(
P_{t}
  =I-U_{t}^{(l)}\,U_{t}^{(l)\top}\in\mathbb{R}^{d\times d}.
\)
To suppress interference from tasks \(\leq t-1\), 
the gating update for task~\(t\) is
\(
W_t^{(g, l)}\leftarrow
W_t^{(g, l)}-\eta\,\nabla L_t^{(l)}P_{t-1}^{(l)}
\).
However, this projection may also discard gradient components that are informative for task~\(t\) whenever its feature support overlaps with \(\operatorname{span}(U_{t-1})\).

\begin{wrapfigure}[33]{r}{0.53\textwidth}\vspace{-5pt}
\begin{minipage}{\linewidth}
\begin{algorithm}[H]
\caption{\label{alg:merging} \methodshort{} Procedure.}
\DontPrintSemicolon
\KwIn{%
    pre-trained model $\theta_0$ and fine-tuned models $\{\theta_t\}_{t=1}^T$; seed data
    $\{\mathcal{D}^{\text{seed}}_t\}_{t=1}^T$;
    hyper-parameters $k,\beta,\gamma$; learning rate $\eta$
}
\KwOut{Merged model $\theta_T^{\text{merged}}$}

{\bf Init:} $\theta_0^{\text{merged}}\!\gets\!\theta_0$;\;
$U_0^{(l)}\!\gets\!\emptyset$\;

\For{{\bf task} $t=1$ {\bf to} $T$}{%
    \Comment{\scriptsize create low-rank experts (Eqs.~\ref{eq:3} and \ref{eq:4})}
    $\displaystyle
      f_t =
      {\small \textsc{SVD}_{\textsc{trunc}.}}~\!\bigl(\,
          \underbrace{\mathcal{P}^{(t-1)}(\theta_t-\theta_0)}_{\scriptsize \text{use } \theta_1 - \theta_0 \text{ when }t=1}
      \bigr)
      = B_tA_t$ \;  

    \Comment{\scriptsize add expert \& initialize gate (Eq.~\ref{eq:5})}
    $\theta_t^{\text{merged}}
       =
       \theta_{t-1}^{\text{merged}} + g_t\!\cdot\! f_t, \;\{W_t^{(g)}, b_t^{(g)},S^{(0)}\}\leftarrow0$
       
    \For{$m=1$ {\bf to} total iterations}{%
        $X \leftarrow \text{batch}(\mathcal{D}^{\text{seed}}_t)$
        
        $\{\nabla_{W_t^{(g)}} L,\nabla_{b_t^{(g)}} L\}\leftarrow\nabla L_t(X,\theta_t^{\text{merged}},\theta_t)$

        \If{$t>1$}{%
            \Comment{\scriptsize project gradient onto null-space (Eqs.~\ref{eq:9}-\ref{eq:12})}
            $S^{(m)}=\beta S^{(m-1)}+(1-\beta)r $
            
            $\nabla_{W_t^{(g)}} L_t\leftarrow\nabla_{W_t^{(g)}} L_t \, \widetilde{P}_{t-1}$
        }
        $W_t^{(g)}\leftarrow W_t^{(g)}-\eta\,\nabla_{W_t^{(g)}} L_t$
        
        $b_t^{(g)}\leftarrow b_t^{(g)}-\eta\,\nabla_{b_t^{(g)}} L_t$
    }

    \Comment{\scriptsize update dominant subspaces}
    $U_t
      \leftarrow
      \operatorname{orthonorm}\!\bigl[U_{t-1}\mid \tilde{U}_{t} ]$
}
\end{algorithm}
\end{minipage}
\end{wrapfigure}

\vspace{4pt}\noindent\textbf{Adaptive Null-Space Relaxation.}
To restore plasticity, we replace the \emph{all-one}
eigenvalues of \(P_{t-1}\) with \emph{soft} coefficients learned
online.

\emph{(i)~Interference statistics.}
For each column \(u_p^{(l)}\) of \(U_{t-1}^{(l)}\) we measure instantaneous
alignment:
\begin{equation}\label{eq:9}
r_p^{(l)}= \|(\nabla L_t)^\top u_p^{(l)}\|_2 \; / \; \|\nabla L_t\|_2.
\end{equation}
We maintain per-direction interference scores \(S^{(m,l)}\in\mathbb{R}^k\) (initialized to \(0\)) by applying an exponential moving average at
each iteration \(m\):
\begin{equation}\label{eq:10}
S^{(m,l)} = \beta\,S^{(m-1,l)} + (1-\beta)\,r^{(l)},
\end{equation}
which suppresses stochastic gradient noise while preserving the dominant
interference directions.

\emph{(ii)~Adaptive shrinkage.}
Each direction is attenuated by
\(
\lambda^{(m,l)}=\exp(-\gamma\,S^{(m,l)})
\)
 (\(\gamma>0, \ \lambda^{(m,l)}\in(0,1]\)).
Let \(\Lambda^{(m,l)}_{t-1}=\mathrm{diag}(\lambda_1^{(m,l)},\dots,\lambda_{(t-1)\cdot k}^{(m,l)})\).
The \emph{relaxed projector} becomes:
\begin{equation}\label{eq:11}
\widetilde{P}_{t-1}^{(m,l)}
  =U_{t-1}^{(m,l)}\,\Lambda^{(m,l)}_{t-1}\,U_{t-1}^{(m,l)\top},
\end{equation}
interpolating smoothly between no protection
(\(\Lambda=0\)) and the hard null projector
(\(\Lambda=I\)).

\emph{(iii)~Update rule.}
We finally update:
\begin{equation}\label{eq:12}
  W_t^{(g,l)} \;\leftarrow\;
  W_t^{(g,l)} - \eta\,\nabla L_t^{(l)}\,\widetilde{P}_{t-1}^{(l)}.
\end{equation}
Relaxing the projector inevitably allows more residual interference than the
hard null‐space variant, yet empirically the increase is minor and is
offset by markedly higher plasticity (Tab.~\ref{tab:ablation}), indicating a favorable \emph{stability–plasticity balance}.
The overall procedure is outlined in Algo. \ref{alg:merging}.

%% file: sections/experiments.tex
\section{Experiments}
We describe the experimental setup in Sec.~\ref{exp:setup}, followed by the main results in Sec.~\ref{exp:result} and further analysis and ablations in Sec.~\ref{exp:analysis}. Due to page limitations, detailed results are provided in the Appendix.

\subsection{Experimental Setup}\label{exp:setup}

\noindent\textbf{Datasets and Models.}\;
Following \cite{tang2025merging}, we evaluate on image-classification tasks with CLIP-ViT backbones \cite{radford2021learning}. We consider 8, 14, and 20-task groups using ViT-B/32, ViT-B/16, and ViT-L/14 models, each fine-tuned on up to 20 downstream tasks, with checkpoints from FusionBench \cite{tangFusionBenchComprehensiveBenchmark2024}. To assess order sensitivity, we repeat experiments over 10 random seeds (42–51). For comparison with conventional CL, we use the Multi-domain Task-Incremental Learning (MTIL) benchmark \cite{zheng2023preventing} with eleven vision tasks. Beyond vision, we evaluate on eight GLUE language tasks \citep{wang2019glue} with a Flan-T5-base backbone \citep{chung2024scaling}.

\noindent\textbf{Implementation Details.}\;
We insert low-rank experts into the CLIP vision encoder. Two variants are used: a full setup modifying all attention and MLP layers, and a lightweight one on \texttt{attn.qkv} and \texttt{mlp.fc1}. All experiments share a single set of \emph{global} hyper-parameters across models and task orders. Each expert has rank $r=64$; the null-space constraint uses $k=3$, $\gamma=1$, and $\beta=0.99$. Adaptation runs for 50 iterations with Adam (lr $1\text{e-4}$, batch size 16). For vision tasks we use 5 unlabeled samples per class, and for NLP tasks 100 in total, all without access to prior-task data.

\noindent\textbf{Evaluation Metrics.} 
We evaluate using average accuracy (ACC) and backward transfer (BWT) \cite{lin2022beyond}. 
ACC is the mean accuracy of the final merged model across all tasks:
\(\text{ACC} = \tfrac{1}{T} \sum_{i=1}^{T} a_i(\theta_t^{\text{merged}})\),
where \(a_i(\cdot)\) is accuracy on task \(i\). 
BWT measures forgetting by comparing performance on earlier tasks before vs.~after the final merge:
\(\text{BWT} = \tfrac{1}{T-1} \sum_{i=1}^{T-1} \big[a_i(\theta_T^{\text{merged}}) - a_i(\theta_i^{\text{merged}})\big]\).

\begin{table}[t]
\centering
\caption{Comparative results of continual merging methods, reporting average accuracy (ACC) and backward transfer (BWT) over ten task orders (mean$\pm$std). DM and DA denote method assumptions: dynamic merging or test data access. 
Best results are in bold; second-best are
underlined. \methodshort{}$^*$ denotes a lightweight variant.}
\setlength{\tabcolsep}{3pt}
\renewcommand\arraystretch{1.05}
\resizebox{1\textwidth}{!}{
\begin{tabular}{p{0.2cm}l|c|ccccccccc}
\toprule
&\multirow{2}{*}{\textbf{Method}} &   \textbf{Assump.} &\multicolumn{3}{c}{\textbf{ViT-B/32}} & \multicolumn{3}{c}{\textbf{ViT-B/16}} & \multicolumn{3}{c}{\textbf{ViT-L/14}} \\ \cmidrule[0.5pt](lr){4-6} \cmidrule[0.5pt](lr){7-9} \cmidrule[0.5pt](lr){10-12}
& &DM / DA & {8 tasks} & {14 tasks} & {20 tasks} & {8 tasks} & {14 tasks} & {20 tasks}& {8 tasks} & {14 tasks} & {20 tasks} \\
\midrule
&\textsc{Pre-Trained}  & -- \hspace{2pt}/\hspace{2pt} --    & 48.1 & 56.9 & 55.6 & 55.4 & 62.0 & 59.8 & 64.9 & 69.1 & 65.6 \\
&\textsc{Fine-Tuned}   & -- \hspace{2pt}/\hspace{2pt} --    & 90.4 & 89.3 & 89.8 &  92.4  & 91.3 & 91.6 & 94.3 & 93.4 & 93.5 \\
&\textsc{C. Fine-Tuned} & -- \hspace{2pt}/\hspace{2pt} --   & 79.8 & 67.4 & 62.6 & 82.9 & 72.2 & 68.2 & 90.0 & 70.9 & 77.7 \\
\midrule
\midrule
\multirow{9}{*}{\rotatebox[origin=c]{90}{ACC (\%) $\uparrow$}}
&\textsc{Average (SWA)} \cite{izmailov2018averaging}       & \xmark \hspace{2pt}/\hspace{2pt} \xmark        & 66.3\tiny{ $\pm$0.0} & 65.4\tiny{ $\pm$0.0} & 61.1\tiny{ $\pm$0.0} & 72.3\tiny{ $\pm$0.0} & 69.7\tiny{ $\pm$0.0} & 64.8\tiny{ $\pm$0.0} & 80.0\tiny{ $\pm$0.0} & 77.5\tiny{ $\pm$0.0} & 71.1\tiny{ $\pm$0.0} \\

&\textsc{C. Task Arithmetic} \cite{ilharco2023editing}  & \xmark \hspace{2pt}/\hspace{2pt} \xmark      & 67.5\tiny{ $\pm$0.0} & 66.5\tiny{ $\pm$0.0} & 60.0\tiny{ $\pm$0.0} & 77.1\tiny{ $\pm$0.0} & 70.9\tiny{ $\pm$0.6} & 64.2\tiny{ $\pm$0.0} & 82.1\tiny{ $\pm$0.0} & 77.9\tiny{ $\pm$0.0} & 70.3\tiny{ $\pm$0.0} \\

&\textsc{C. Ties-Merging} \cite{yadav2023ties}    & \xmark \hspace{2pt}/\hspace{2pt} \xmark       & 49.0\tiny{ $\pm$10.2} & 66.2\tiny{ $\pm$0.6} & 59.9\tiny{ $\pm$0.7} & 66.8\tiny{ $\pm$3.7} & 70.5\tiny{ $\pm$0.8} & 63.0\tiny{ $\pm$1.6} & 64.3\tiny{ $\pm$7.0} & 78.0\tiny{ $\pm$0.6} & 68.3\tiny{ $\pm$0.9} \\

&\textsc{MagMax-Ind} \cite{marczak2024magmax}      & \xmark \hspace{2pt}/\hspace{2pt} \xmark      & 70.7\tiny{ $\pm$0.0} & 67.0\tiny{ $\pm$0.0} & 61.2\tiny{ $\pm$0.0} & 76.7\tiny{ $\pm$1.8} & 67.0\tiny{ $\pm$0.0} & 62.5\tiny{ $\pm$0.0} & 83.4\tiny{ $\pm$0.0} & 71.2\tiny{ $\pm$0.0} & 71.2\tiny{ $\pm$0.0} \\

&\textsc{Consensus TA} \cite{wanglocalizing}      & \xmark \hspace{2pt}/\hspace{2pt} \xmark      & 67.1\tiny{ $\pm$0.4} & 64.1\tiny{ $\pm$0.8} & 45.8\tiny{ $\pm$1.5} & 72.8\tiny{ $\pm$0.5} & 69.0\tiny{ $\pm$0.0} & 49.9\tiny{ $\pm$1.9} & 80.4\tiny{ $\pm$0.5} &  75.0\tiny{ $\pm$1.0} & 51.3\tiny{ $\pm$2.4} \\

&\textsc{OPCM} \cite{tang2025merging}  & \xmark \hspace{2pt}/\hspace{2pt} \xmark  & 75.5\tiny{ $\pm$0.5} & 71.9\tiny{ $\pm$0.3} & \underline{65.7}\tiny{ $\pm$0.2} & 81.8\tiny{ $\pm$0.3} & 77.1\tiny{ $\pm$0.5} & 70.3\tiny{ $\pm$0.2} & 87.0\tiny{ $\pm$0.4} & 83.5\tiny{ $\pm$0.2} & 76.0\tiny{ $\pm$0.2} \\

&\textsc{C. LW AdaMerging} \cite{yang2023adamerging} & \xmark \hspace{2pt}/\hspace{2pt} \cmark   &  53.4\tiny{ $\pm$ 3.2} &  59.8\tiny{ $\pm$ 1.6} &  59.7\tiny{ $\pm$ 7.4} &  59.9\tiny{ $\pm$2.3} &  64.3\tiny{ $\pm$ 1.2} &  61.5\tiny{ $\pm$ 1.1}  &  68.8\tiny{ $\pm$ 2.9} & 73.1\tiny{ $\pm$ 5.7} &  66.9\tiny{ $\pm$ 1.1}  \\

&\textsc{C. LoRA-WEMOE} \cite{tang2024merging}  & \cmark \hspace{2pt}/\hspace{2pt} \cmark  & 68.8\tiny{ $\pm$ 7.8} & 63.8\tiny{ $\pm$ 3.4} & 49.6\tiny{ $\pm$ 15.4} & 72.6\tiny{ $\pm$ 3.7} & 67.9\tiny{ $\pm$ 2.9} & 55.0\tiny{ $\pm$ 7.0} & 75.6\tiny{ $\pm$ 7.8} & 74.0\tiny{ $\pm$ 5.0} & 56.9\tiny{ $\pm$ 19.8} \\

\rowcolor{gray!10}&
\methodshort{} (Ours) & \cmark \hspace{2pt}/\hspace{2pt} \cmark  & \textbf{85.8}\tiny{ $\pm$0.8} & \underline{81.6}\tiny{ $\pm$1.4} & \textbf{77.1}\tiny{ $\pm$2.0} & \textbf{88.3}\tiny{ $\pm$0.6} & \textbf{84.9}\tiny{ $\pm$0.8} & \textbf{81.9}\tiny{ $\pm$0.9} & \textbf{91.8}\tiny{ $\pm$0.2} & \underline{88.8}\tiny{ $\pm$0.7} & \textbf{85.5}\tiny{ $\pm$1.3} \\

\rowcolor{gray!10}&
\methodshort{}$^*$ (Ours) & \cmark \hspace{2pt}/\hspace{2pt} \cmark  & \underline{85.0}\tiny{ $\pm$0.5} & \textbf{81.7}\tiny{ $\pm$1.0} & \textbf{77.1}\tiny{ $\pm$1.3} & \underline{87.0}\tiny{ $\pm$0.6} & \underline{84.7}\tiny{ $\pm$1.0} & \underline{81.6}\tiny{ $\pm$1.3} & \underline{91.4}\tiny{ $\pm$0.3} & \textbf{89.2}\tiny{ $\pm$0.1} & \underline{83.6}\tiny{ $\pm$0.6} \\

\midrule

\multirow{9}{*}{\rotatebox[origin=c]{90}{BWT (\%) $\uparrow$}}
&\textsc{Average (SWA)} \cite{izmailov2018averaging}  & \xmark \hspace{2pt}/\hspace{2pt} \xmark          & -11.5\tiny{ $\pm$2.2} & -8.0\tiny{ $\pm$1.3} & -7.1\tiny{ $\pm$2.1} & -9.7\tiny{ $\pm$1.5} & -7.1\tiny{ $\pm$1.4} & -7.3\tiny{ $\pm$1.7} & -7.3\tiny{ $\pm$1.4} & -5.8\tiny{ $\pm$1.0} & -6.4\tiny{ $\pm$1.5} \\

&\textsc{C. Task Arithmetic} \cite{ilharco2023editing}  & \xmark \hspace{2pt}/\hspace{2pt} \xmark      & -9.6\tiny{ $\pm$1.5} & -1.3\tiny{ $\pm$1.6} & -3.4\tiny{ $\pm$1.0} & -4.2\tiny{ $\pm$1.0} & -1.3\tiny{ $\pm$0.4} & -3.6\tiny{ $\pm$0.4} & -7.1\tiny{ $\pm$0.8} & -1.8\tiny{ $\pm$0.3} & -3.3\tiny{ $\pm$0.3} \\

&\textsc{C. Ties-Merging} \cite{yadav2023ties}     & \xmark \hspace{2pt}/\hspace{2pt} \xmark      & -15.3\tiny{ $\pm$8.0} & \textbf{1.9}\tiny{ $\pm$0.6} & \underline{-1.5}\tiny{ $\pm$0.7} & -5.5\tiny{ $\pm$0.4} & \textbf{1.4}\tiny{ $\pm$0.7} & \underline{-1.5}\tiny{ $\pm$1.2} & -13.0\tiny{ $\pm$5.7} & -1.1\tiny{ $\pm$0.4} & -2.9\tiny{ $\pm$1.0} \\

&\textsc{MagMax-Ind} \cite{marczak2024magmax}     & \xmark \hspace{2pt}/\hspace{2pt} \xmark      & -8.3\tiny{ $\pm$1.3} & -7.4\tiny{ $\pm$1.4} & -7.2\tiny{ $\pm$1.6} & -6.1\tiny{ $\pm$1.3} & -7.4\tiny{ $\pm$2.0} & -8.0\tiny{ $\pm$2.2} & -5.0\tiny{ $\pm$0.8} & -6.0\tiny{ $\pm$2.1} & -6.5\tiny{ $\pm$2.1} \\

&\textsc{Consensus TA} \cite{wanglocalizing}      & \xmark \hspace{2pt}/\hspace{2pt} \xmark      & \textbf{3.8}\tiny{ $\pm$0.9} & -1.3\tiny{ $\pm$0.9} & -11.8\tiny{ $\pm$1.9} & \textbf{3.5}\tiny{ $\pm$0.6} & -1.1\tiny{ $\pm$0.8} & -11.6\tiny{ $\pm$1.3} & \textbf{2.4}\tiny{ $\pm$0.6} & -2.5\tiny{ $\pm$0.8} & -16.5\tiny{ $\pm$1.5} \\

&\textsc{OPCM} \cite{tang2025merging}  & \xmark \hspace{2pt}/\hspace{2pt} \xmark   & -6.3\tiny{ $\pm$1.1} & -6.0\tiny{ $\pm$1.0} & -7.8\tiny{ $\pm$1.5} & -4.8\tiny{ $\pm$0.7} & -5.1\tiny{ $\pm$1.4} & -6.3\tiny{ $\pm$2.2} & -2.6\tiny{ $\pm$1.0} & -4.3\tiny{ $\pm$0.7} & -6.5\tiny{ $\pm$1.8} \\

&\textsc{C. LW AdaMerging} \cite{yang2023adamerging} & \xmark \hspace{2pt}/\hspace{2pt} \cmark  & -32.5\tiny{ $\pm$3.6} & -24.1\tiny{ $\pm$1.7} & -22.7\tiny{ $\pm$4.3} & -27.8\tiny{ $\pm$2.7} & -22.1\tiny{ $\pm$1.4} & -21.4\tiny{ $\pm$1.2} & -24.3\tiny{ $\pm$3.3} & -19.6\tiny{ $\pm$1.7} & -21.7\tiny{ $\pm$1.1} \\

&\textsc{C. LoRA-WEMOE} \cite{tang2024merging} & \cmark \hspace{2pt}/\hspace{2pt} \cmark & -20.4\tiny{ $\pm$9.0} & -20.2\tiny{ $\pm$3.9} & -24.5\tiny{ $\pm$10.0} & -18.0\tiny{ $\pm$6.2} & -18.8\tiny{ $\pm$3.4} & -25.8\tiny{ $\pm$7.9} & -17.8\tiny{ $\pm$5.9} & -16.8\tiny{ $\pm$5.3} & -27.9\tiny{ $\pm$17.2} \\

\rowcolor{gray!10} &
\methodshort{} (Ours) & \cmark \hspace{2pt}/\hspace{2pt} \cmark  & -0.6\tiny{ $\pm$0.4} & -1.1\tiny{ $\pm$0.3} & -2.2\tiny{ $\pm$0.8} & -0.4\tiny{ $\pm$0.1} & -0.9\tiny{ $\pm$0.1} & -1.9\tiny{ $\pm$0.4} & -0.6\tiny{ $\pm$0.1}  & \underline{-1.0}\tiny{ $\pm$0.3} & \underline{-2.6}\tiny{ $\pm$0.9} \\

\rowcolor{gray!10} &
\methodshort{}$^*$ (Ours)  & \cmark \hspace{2pt}/\hspace{2pt} \cmark  & \underline{-0.1}\tiny{ $\pm$0.1} & \underline{-0.4}\tiny{ $\pm$0.1} & \textbf{-1.3}\tiny{ $\pm$0.6} & \underline{-0.1}\tiny{ $\pm$0.1} & \underline{-0.3}\tiny{ $\pm$0.1} & \textbf{-1.0}\tiny{ $\pm$0.4} & \underline{-0.2}\tiny{ $\pm$0.0} & \textbf{-0.4}\tiny{ $\pm$0.2} & \textbf{-1.5}\tiny{ $\pm$0.6} \\
\bottomrule
\end{tabular}}\vspace{-5pt}
\label{tab:results}
\end{table}

 \begin{table}[t]
\centering
\caption{Results of continual merging Flan-T5-base models on 8 tasks, ordered alphabetically.}
\renewcommand\arraystretch{1.05}
\setlength{\tabcolsep}{7pt}
\resizebox{1\linewidth}{!}{\begin{tabular}{l|c|cccccccc|cc}
\toprule
\textbf{Method} & \textbf{DM / DA} & \textbf{CoLA} & \textbf{MNLI} & \textbf{MRPC} & \textbf{QNLI} & \textbf{QQP}  & \textbf{RTE}  & \textbf{SST2} & \textbf{STSB} & \textbf{ACC $\uparrow$} & \textbf{BWT $\uparrow$} \\
\midrule
\textsc{Pre-trained}   & -- \hspace{2pt}/\hspace{2pt} --   & 69.1 & 56.5 & 76.2 &88.4 & 82.1 & 80.1 & 91.2 & 62.2 & 75.7 & - \\
\textsc{Individual}      & -- \hspace{2pt}/\hspace{2pt} --     & 75.0 & 83.4 & 87.5 & 91.5 & 85.4 & 85.9 & 93.6 & 88.7 & 86.4 & - \\ 
\midrule
\textsc{Task Arithmetic}  & \xmark \hspace{2pt}/\hspace{2pt} \xmark
      & 69.1 & 58.1 & 77.9 & 88.9 & 83.1 & 79.1 & 90.7 & 74.0 & 77.6 & -4.6 \\

\textsc{Ties-Merging}     & \xmark \hspace{2pt}/\hspace{2pt} \xmark
      & 39.3 & 70.0 & \underline{82.4} & 88.8 & 81.8 & 75.8 & 89.7 & \underline{76.8} & 75.6 & -6.1 \\

\textsc{OPCM}             & \xmark \hspace{2pt}/\hspace{2pt} \xmark
      & 69.9 & 72.9 & 78.7 & \underline{90.3} & \underline{83.8} & \textbf{83.0} & \underline{92.2} & 73.7 & 80.6 & \underline{-2.5} \\

\textsc{LW AdaMerging}    & \xmark \hspace{2pt}/\hspace{2pt} \cmark
      & 69.1 & 58.1 & 77.9 & 88.9 & 83.1 & 79.1 & 90.7 & 74.2 & 77.6 & -4.7 \\

\textsc{LoRA-WEMOE}       & \cmark \hspace{2pt}/\hspace{2pt} \cmark
      & \underline{71.5} & \textbf{80.6} & 78.2 & \underline{90.3} & 82.7 & \underline{80.5} & 91.3 & 76.2 & \underline{81.4} & \textbf{0.1} \\

\rowcolor{gray!10}\methodshort{} (Ours) & \cmark \hspace{2pt}/\hspace{2pt} \cmark
      & \textbf{75.0} & \underline{78.2} & \textbf{86.0} & \textbf{90.9} & \textbf{84.2} & \underline{80.5} & \textbf{92.5} & \textbf{78.8} & \textbf{83.3} & \textbf{0.1} \\

\bottomrule
\end{tabular}}
\label{flan-t5-base}
\end{table}

\subsection{Main Results}\label{exp:result}

\noindent\textbf{Overall Performance}\;
As shown in Tab.~\ref{tab:results}, \methodshort{} substantially outperforms previous continual merging methods on all CLIP backbones and task counts. It achieves the highest accuracy with backward forgetting kept near zero, demonstrating both strong forward learning and long-term stability. The lightweight variant performs on par with the full version, further underscoring the robustness of our approach.
On NLP benchmarks (Tab.~\ref{flan-t5-base}), \methodshort{} likewise attains the best overall accuracy and non-negative BWT, improving on multiple GLUE tasks while maintaining balanced performance across the suite.
Together, these results across vision and language confirm that \methodshort{} consistently delivers state-of-the-art accuracy while nearly eliminating forgetting under diverse continual scenarios.

\begin{table}[t]
  \centering
  \caption{Comparison of last accuracy (\%) with conventional CL approaches on MTIL benchmark.}
  \label{tab:mtil}
  \setlength{\tabcolsep}{4.3pt}
  \renewcommand\arraystretch{1.05}
  \resizebox{1\linewidth}{!}{%
  \begin{tabular}{llllllllllllc}
    \toprule
    \textbf{Method} &
    \rot{\small{Aircraft}}\hspace{3pt} & \rot{\small{Caltech101}}\hspace{-8pt} & \rot{\small{CIFAR100}}\hspace{-5pt} &
    \rot{\small{DTD}}\hspace{8pt} & \rot{\small{EuroSAT}}\hspace{-3pt} & \rot{\small{Flowers}} &
    \rot{\small{Food101}} & \rot{\small{MNIST}} &
    \hspace*{6pt}\rot{\small{Pets}}\hspace*{6pt} &
    \hspace*{6pt}\rot{\small{Cars}}\hspace*{6pt} & \rot{\small{SUN397}} &
    \textbf{\hspace*{3pt}Avg.\hspace*{3pt}} \\
    \midrule

    \multicolumn{13}{l}{\textbf{Conventional CL} (\emph{Sequential fine‐tuned})} \\ \cmidrule[0.5pt](lr){1-4}
    \textsc{WiSE-FT} \cite{wortsman2022robust}   & 27.2 & 90.8 & 68.0 & 68.9 & 86.9 & 74.0 & 87.6 & \textbf{99.6} & 92.6 & 77.8 & \underline{81.3} & 77.7 \\
    \textsc{ZSCL} \cite{zheng2023preventing}     & 40.6 & 92.2 & 81.3 & 70.5 & 94.8 & 90.5 & \underline{91.9} & 98.7 & \underline{93.9} & 85.3 & 80.2 & 83.6 \\
    \textsc{MoE-Adapter} \cite{yu2024boosting}   & 49.8 & 92.2 & 86.1 & 78.1 & 95.7 & 94.3 & 89.5 & 98.1 & 89.9 & 81.6 & 80.0 & 85.0 \\
    \textsc{DIKI} \cite{tang2024mind}            & 45.2 & 95.7 & \underline{86.3} & 72.9 & \textbf{98.0} & \underline{97.0} & 89.2 & \underline{99.4} & \textbf{94.2} & 81.6 & 76.6 & 85.1 \\
    \textsc{AwoForget} \cite{zheng2024adapt}     & 42.4 & 92.7 & 83.2 & 73.2 & 97.0 & 91.8 & \textbf{92.2} & 99.1 & \underline{93.9} & \textbf{87.4} & \textbf{82.6} & 85.0 \\
    \textsc{Dual-RAIL} \cite{xu2024advancing}    & \underline{52.5} & \underline{96.8} & 83.3 & \textbf{80.1} & 96.4 & \textbf{99.0} & 89.9 & 98.8 & 93.5 & \underline{85.5} & 79.2 & \textbf{86.8} \\
    \textsc{MagMax} \cite{marczak2024magmax}     & 40.2 & 96.1 & 81.1 & 72.0 & \underline{97.8} & 76.3 & 88.4 & 99.2 & 93.0 & 70.5 & 68.9 & 80.3 \\
    \rowcolor{gray!10}
    \textsc{Mingle-Seq}                          & \textbf{58.7} & \textbf{97.5} & \textbf{87.2} & \underline{79.7} & 97.3 & 87.2 & 90.1 & \textbf{99.6} & 93.0 & 80.4 & 73.3 & \underline{85.8} \\

    \midrule
    \multicolumn{13}{l}{\textbf{Continual Merging} (\emph{Independent fine‐tuned})} \\ \cmidrule[0.5pt](lr){1-4}
    \textsc{Average (SWA)} \cite{izmailov2018averaging}  & 26.5 & 92.3 & 74.3 & 48.4 & 73.7 & 74.0 & 87.1 & 84.0 & 91.2 & 67.5 & 68.5 & 71.6 \\
\textsc{C.\ TA} \cite{ilharco2023editing}            & 26.6 & 92.5 & 74.5 & 48.7 & 74.3 & 74.4 & 87.0 & 85.5 & 91.2 & 67.7 & 68.6 & 71.9 \\
\textsc{C.\ Ties} \cite{yadav2023ties}               & 30.5 & 94.0 & 74.8 & 49.8 & 71.7 & 73.8 & \underline{87.3} & 81.5 & 90.6 & 67.0 & 67.9 & 71.7 \\
\textsc{MagMax-Ind} \cite{marczak2024magmax}         & 29.9 & 93.7 & \underline{78.4} & 46.1 & 58.3 & 68.1 & 86.8 & 82.8 & 91.4 & 62.7 & 69.3 & 69.8 \\
\textsc{OPCM} \cite{tang2025merging}                 & \underline{35.7} & \underline{95.9} & 77.0 & \underline{54.6} & \underline{90.3} & \underline{76.4} & 87.1 & \underline{96.3} & \underline{93.3} & \underline{70.1} & \underline{70.5} & \underline{77.0} \\
\rowcolor{gray!10}
\methodshort{}                                       & \textbf{54.2} & \textbf{97.3} & \textbf{79.7} & \textbf{72.3} & \textbf{96.0} & \textbf{86.7} & \textbf{88.7} & \textbf{99.3} & \textbf{93.9} & \textbf{73.1} & \textbf{71.6} & \textbf{83.0} \\
    \bottomrule
  \end{tabular}}\vspace{-5pt}
\end{table}

\begin{table}[t]
  \centering
\caption{Robustness results of ViT-B/32 continually merged across 4 tasks.}        \label{tab:robust}
        \setlength{\tabcolsep}{5pt}
        \renewcommand\arraystretch{1.05}
        \resizebox{1\linewidth}{!}{%
        \begin{tabular}{p{0.15cm}lccccccccc}
        \toprule
        &\textbf{Method} \hspace{8pt}
          & \textbf{Clean} 
          & \textbf{Motion} 
          & \textbf{Impulse} 
          & \textbf{Gaussian} 
          & \textbf{Pixelate} 
          & \textbf{Spatter} 
          & \textbf{Contrast}
          & \textbf{JPEG} 
          & \textbf{Avg.}\\
        \midrule
\multirow{7}{*}{\rotatebox[origin=c]{90}{ACC (\%) $\uparrow$}} &
        \textsc{C. LW AdaMerging} \cite{yang2023adamerging} 
          &  56.0\tiny{ $\pm$5.3}
          &  47.5\tiny{ $\pm$4.4} 
          &  43.1\tiny{ $\pm$2.3} 
          &  43.3\tiny{ $\pm$3.4}
          &  18.1\tiny{ $\pm$4.7} 
          &  46.6\tiny{ $\pm$3.0} 
          &  48.9\tiny{ $\pm$4.8} 
          &  49.1\tiny{ $\pm$4.0} 
          &  44.9\\

        &\textsc{C. WEMOE} \cite{tang2024merging} 
          &   3.4\tiny{ $\pm$0.8} 
          &   3.1\tiny{ $\pm$0.4} 
          &   4.3\tiny{ $\pm$1.4} 
          &   3.4\tiny{ $\pm$1.4}
          &   3.0\tiny{ $\pm$1.6} 
          &   4.0\tiny{ $\pm$0.9} 
          &   3.3\tiny{ $\pm$0.7}
          &   4.0\tiny{ $\pm$1.2}
          &   3.6\\

        &\textsc{C. LoRA-WEMOE} \cite{tang2024merging} 
          &   78.7\tiny{ $\pm$4.5} 
          &   71.0\tiny{ $\pm$4.9} 
          &   55.0\tiny{ $\pm$3.8} 
          &   59.4\tiny{ $\pm$3.8}
          &   24.9\tiny{ $\pm$24.9} 
          &   60.5\tiny{ $\pm$3.8} 
          &   68.5\tiny{ $\pm$4.8}
          &   69.7\tiny{ $\pm$4.4}
          &   61.0\\
          
        &\textsc{C. Task Arithmetic} \cite{ilharco2023editing} 
          &  77.5\tiny{ $\pm$0.0}
          &  66.0\tiny{ $\pm$0.0}
          &  58.9\tiny{ $\pm$0.0} 
          &  59.6\tiny{ $\pm$0.0} 
          &  29.7\tiny{ $\pm$0.0} 
          &  63.5\tiny{ $\pm$0.0}
          &  66.0\tiny{ $\pm$0.0}
          &  67.8\tiny{ $\pm$0.0}
          &  61.1\\

        &\textsc{MagMax-Ind} \cite{marczak2024magmax} 
          &  79.1\tiny{ $\pm$0.0}
          &  69.0\tiny{ $\pm$0.0}
          &  60.6\tiny{ $\pm$0.0} 
          &  61.5\tiny{ $\pm$0.0}
          &  33.0\tiny{ $\pm$0.0}
          &  66.4\tiny{ $\pm$0.0}
          &  68.6\tiny{ $\pm$0.0}
          &  69.9\tiny{ $\pm$0.0}
          &  63.5\\

        &\textsc{OPCM} \cite{tang2025merging} 
          &  83.6\tiny{ $\pm$0.5}
          &  72.5\tiny{ $\pm$0.6}
          &  64.7\tiny{ $\pm$1.2}
          &  65.2\tiny{ $\pm$1.2} 
          &  35.2\tiny{ $\pm$0.6}
          &  70.5\tiny{ $\pm$0.5}
          &  72.5\tiny{ $\pm$0.6}
          &  74.4\tiny{ $\pm$0.3}
          &  67.3\\

        \rowcolor{gray!10} &
        \methodshort{} (Ours)
          &  \textbf{89.9}\tiny{ $\pm$0.4}
          &  \textbf{82.8}\tiny{ $\pm$0.8}
          &  \textbf{67.5}\tiny{ $\pm$2.0} 
          &  \textbf{70.7}\tiny{ $\pm$1.2} 
          &  \textbf{37.9}\tiny{ $\pm$0.4}
          &  \textbf{77.0}\tiny{ $\pm$0.7}
          &  \textbf{80.1}\tiny{ $\pm$0.8}
          &  \textbf{82.9}\tiny{ $\pm$0.9}
          &  \textbf{73.2}\\

        \midrule
\multirow{7}{*}{\rotatebox[origin=c]{90}{BWT (\%) $\uparrow$}} &
        \textsc{C. LW AdaMerging} \cite{yang2023adamerging} 
          & -38.0\tiny{ $\pm$7.1}
          & -37.3\tiny{ $\pm$5.9}
          & -22.2\tiny{ $\pm$3.0}
          & -25.2\tiny{ $\pm$4.5}
          & -20.8\tiny{ $\pm$6.3}
          & -28.6\tiny{ $\pm$4.0}
          & -34.7\tiny{ $\pm$6.5}
          & -36.1\tiny{ $\pm$5.3}
          & -29.5\\

        &\textsc{C. WEMOE} \cite{tang2024merging} 
          & -30.7\tiny{ $\pm$3.1}
          & -28.7\tiny{ $\pm$3.8}
          & -22.1\tiny{ $\pm$11.6}
          & -25.5\tiny{ $\pm$9.8}
          &  -8.0\tiny{ $\pm$4.5}
          & -23.4\tiny{ $\pm$11.0}
          & -27.6\tiny{ $\pm$5.6}
          & -28.6\tiny{ $\pm$5.2}
          & -24.3\\

        &\textsc{C. LoRA-WEMOE} \cite{tang2024merging} 
          &   -13.6\tiny{ $\pm$6.9} 
          &   -14.6\tiny{ $\pm$8.2} 
          &   -11.3\tiny{ $\pm$4.7} 
          &   -10.2\tiny{ $\pm$3.7}
          &   -15.6\tiny{ $\pm$8.8} 
          &   -10.8\tiny{ $\pm$3.9} 
          &   -11.2\tiny{ $\pm$8.8}
          &   -16.7\tiny{ $\pm$6.6}
          &   -13.0\\
          
        &\textsc{C. Task Arithmetic} \cite{ilharco2023editing} 
          &  -4.8\tiny{ $\pm$0.9}
          &  -6.1\tiny{ $\pm$1.2}
          &  -1.6\tiny{ $\pm$3.0}
          &  -1.6\tiny{ $\pm$1.7}
          &  -2.7\tiny{ $\pm$1.5}
          &  -3.1\tiny{ $\pm$2.5}
          &  -6.1\tiny{ $\pm$1.2}
          &  -5.1\tiny{ $\pm$0.8}
          &  -3.9\\

        &\textsc{MagMax-Ind} \cite{marczak2024magmax} 
          &  -7.7\tiny{ $\pm$0.8}
          &  -8.1\tiny{ $\pm$1.5}
          &  -6.1\tiny{ $\pm$4.9}
          &  -5.1\tiny{ $\pm$3.7}
          &  -3.5\tiny{ $\pm$3.0}
          &  -7.3\tiny{ $\pm$2.9}
          &  -8.4\tiny{ $\pm$1.6}
          &  -8.2\tiny{ $\pm$1.2}
          &  -6.8\\

        &\textsc{OPCM} \cite{tang2025merging} 
          &  -4.3\tiny{ $\pm$1.8}
          &  -4.5\tiny{ $\pm$2.8}
          &  -6.4\tiny{ $\pm$7.1}
          &  -6.1\tiny{ $\pm$4.3}
          &  -2.9\tiny{ $\pm$0.9}
          &  -6.3\tiny{ $\pm$2.9}
          &  -4.5\tiny{ $\pm$2.8}
          &  -5.7\tiny{ $\pm$1.5}
          &  -5.1\\

        \rowcolor{gray!10} &
        \methodshort{}  (Ours)
          &  \textbf{-0.2}\tiny{ $\pm$0.2}
          &   \textbf{-0.1}\tiny{ $\pm$0.4}
          &   \textbf{0.7}\tiny{ $\pm$1.0}
          &   \textbf{0.6}\tiny{ $\pm$1.1}
          &  \textbf{-0.2}\tiny{ $\pm$1.1}
          &   \textbf{0.2}\tiny{ $\pm$0.7}
          &   \textbf{0.0}\tiny{ $\pm$0.5}
          &   \textbf{0.5}\tiny{ $\pm$0.5}
          &   \textbf{-0.2}\\
        \bottomrule
        \end{tabular}
        }
\end{table}

\noindent\textbf{Comparison with Conventional CL.} Tab. \ref{tab:mtil} evaluates two CL paradigms on the MTIL benchmark: conventional CL, where each task model is fine-tuned from its immediate predecessor; and continual merging, which fine-tunes each task model independently before fusion, eliminating inter-model dependencies and enabling flexible task ordering and model reuse. Within the merging family, \methodshort{} sets a new state‑of‑the‑art, and when integrated into a sequential fine‑tuning pipeline, it matches the performance of SOTA CL methods. This demonstrates both its strength as a fusion strategy and its versatility across different training regimes.

\noindent\textbf{Robustness to Test-Time Distribution Shifts.} 
Following prior work \cite{yang2023adamerging,tang2024merging}, we evaluate \methodshort{} on seven corruptions (motion blur, impulse noise, Gaussian noise, pixelate, spatter, contrast, JPEG) and report results in Tab.~\ref{tab:robust}. It preserves high accuracy and near‑zero or even positive BWT, outperforming all baselines, whereas direct application of SOTA TTA‑based merging (WEMOE, AdaMerging) in a continual setting fails without tailored designs to continual setup.

\subsection{Ablation Results and Analysis}\label{exp:analysis}

\noindent\textbf{Ablation Study.}
We explore the contribution of each component in Tab.~\ref{tab:ablation}. Row 1 shows a fixed‑weight merging of low‑rank experts as our baseline. In Row 2, adding TTA boosts ACC substantially but at the cost of worsening BWT with more tasks. Row 3 demonstrates that freezing earlier gates curbs forgetting while retaining ACC gains. Row 4 then applies null‑space constraints, yielding further BWT improvements. Finally, Row 5 presents the full method with adaptive relaxation, which best harmonizes accuracy and long‑term stability.


\begin{table}[t]
\centering
\caption{Ablation study of \methodshort{} with CLIP ViT-B/16 over 8, 14, and 20 tasks.}
\label{tab:ablation}
\setlength{\tabcolsep}{6pt}
  \renewcommand\arraystretch{1.05}
\resizebox{1\textwidth}{!}{
\begin{tabular}{cccccccccc}
\toprule
\textbf{Test-Time} & \textbf{Frozen} & \textbf{Null-Space} & \textbf{Adaptive}
& \multicolumn{3}{c}{\textbf{ACC(\%)} ↑} 
& \multicolumn{3}{c}{\textbf{BWT(\%)} ↑} \\
\cmidrule[0.5pt](lr){5-7} \cmidrule[0.5pt](lr){8-10}
\textbf{Adaptation} & \textbf{Old Gate} & \textbf{Constrained Gate} & \textbf{Relaxation}  &{8 tasks} & {14 tasks} & {20 tasks} & {8 tasks} & {14 tasks} & {20 tasks} \\
\midrule
\xmark & --      & -- & --  
                        & 78.7\tiny{ $\pm$0.1} & 76.4\tiny{ $\pm$1.0}& 70.6\tiny{ $\pm$0.4}& -0.5\tiny{ $\pm$0.1}& -1.0\tiny{ $\pm$0.3}& -1.3\tiny{ $\pm$0.3} \\
\cmark & \xmark   & \xmark  & \xmark     & 86.4\tiny{ $\pm$5.3} & 81.7\tiny{ $\pm$2.3} & 76.7\tiny{ $\pm$1.3}
                        & -6.0\tiny{ $\pm$5.2} & -7.7\tiny{ $\pm$-3.4} & -12.8\tiny{ $\pm$1.3} \\
\cmark &\cmark & \xmark  & \xmark       
                        & 87.4\tiny{ $\pm$0.4} & 81.3\tiny{ $\pm$0.8} & 76.2\tiny{ $\pm$1.3}
                        & -2.3\tiny{ $\pm$0.5} & -4.3\tiny{ $\pm$0.8} & -6.8\tiny{ $\pm$0.9}  \\
\cmark &\cmark & \cmark  & \xmark       
                        & 86.0\tiny{ $\pm$1.5} & 83.5\tiny{ $\pm$0.9}& 78.3\tiny{ $\pm$1.7}& -0.1\tiny{ $\pm$0.1}& -0.1\tiny{ $\pm$0.1}& -0.2\tiny{ $\pm$0.1}\\
\rowcolor{gray!10}           
\cmark &  \cmark & \cmark & \cmark 
                        &88.3\tiny{ $\pm$0.6} & 84.9\tiny{ $\pm$0.8}& 81.9\tiny{ $\pm$0.9}& -0.4\tiny{ $\pm$0.1}& -0.9\tiny{ $\pm$0.1}& -1.9\tiny{ $\pm$0.4} \\
\bottomrule
\end{tabular}}
\end{table}

\begin{table}[t]
\centering
\caption{Ablation on number of adaptation steps for ViT-B/32 across 8, 14, and 20 tasks.}
\label{tab:tta_ablation}
\setlength{\tabcolsep}{11pt}
  \renewcommand\arraystretch{1.05}
\resizebox{\linewidth}{!}{
\begin{tabular}{lcccccc}
\toprule
Steps & ACC (8-task) & BWT (8-task) & ACC (14-task) & BWT (14-task) & ACC (20-task) & BWT (20-task) \\
\midrule
5   & 60.9{\tiny{$\pm$1.4}} & -0.1{\tiny{$\pm$0.2}} & 62.4{\tiny{$\pm$1.7}} & -0.3{\tiny{$\pm$0.1}} & 60.2{\tiny{$\pm$1.5}} & -0.1{\tiny{$\pm$0.2}} \\
10  & 68.6{\tiny{$\pm$1.6}} & -0.2{\tiny{$\pm$0.2}} & 68.5{\tiny{$\pm$1.6}} & -0.1{\tiny{$\pm$0.2}} & 63.5{\tiny{$\pm$1.0}} & -0.4{\tiny{$\pm$0.3}} \\
20  & 78.4{\tiny{$\pm$0.6}} & -0.2{\tiny{$\pm$0.1}} & 75.5{\tiny{$\pm$1.3}} & -0.4{\tiny{$\pm$0.1}} & 71.0{\tiny{$\pm$1.2}} & -0.4{\tiny{$\pm$0.4}} \\
\rowcolor{gray!10} 50  & 85.8{\tiny{$\pm$0.8}} & -0.6{\tiny{$\pm$0.4}} & 81.6{\tiny{$\pm$1.4}} & -1.1{\tiny{$\pm$0.3}} & 77.1{\tiny{$\pm$2.0}} & -2.2{\tiny{$\pm$0.8}} \\
\bottomrule
\end{tabular}}
\end{table}

\begin{table}[!ht]
\centering
\caption{Efficiency and sample analysis. 
(a) Expert insertion layers and rank sweep over 8 tasks on CLIP-ViT-B/32. 
(b) Wall-clock adaptation time across tasks on CLIP-ViT-B/32. 
(c) Accuracy (\%) of CLIP-ViT-B/16 under varying numbers of test samples.}
\label{tab:efficiency_combined}
\begin{minipage}{0.55\linewidth}
\centering
  \renewcommand\arraystretch{1.05}
  \setlength{\tabcolsep}{2pt}
\resizebox{\linewidth}{!}{
\begin{tabular}{lcccc}
\multicolumn{5}{c}{\textbf{(a) Expert insertion layers and rank sweep (8 tasks)}} \\
\toprule
\textbf{Configuration} & \textbf{TTA Time} & \textbf{Train. Param} & \textbf{Full Param} & \textbf{ACC(\%)} \\
\midrule
\texttt{attn.qkv\_proj} ($r=64$)  & 61 s & 27.7 k & 116.0 M & 69.9 \\
\texttt{attn.out\_proj} ($r=64$)  & 47 s & 9.3 k  & 97.0 M  & 53.9 \\
\texttt{mlp.fc1} ($r=64$)         & 48 s & 9.0 k  & 111.1 M & 82.6 \\
\texttt{mlp.fc2} ($r=64$)         & 48 s & 36.8 k & 111.3 M & 70.1 \\
\rowcolor{gray!10}\texttt{attn \& mlp} ($r=64$) & 78 s & 83.0 k & 173.1 M & 85.8 \\
\midrule
\texttt{qkv \& fc1} ($r=32$) & 65 s & 36.9 k & 113.7 M & 84.5\\
\texttt{qkv \& fc1} ($r=128$)     & 65 s & 36.9 k & 191.6 M & 85.1 \\
\texttt{qkv \& fc1} ($r=768$)     & 70 s & 36.9 k & 710.3 M & 83.7 \\
\rowcolor{gray!10}\texttt{qkv \& fc1} ($r=64$)      & 65 s & 36.9 k & 139.7 M & 85.0  \\
\bottomrule
\end{tabular}}
\end{minipage}\hfill
\begin{minipage}{0.42\linewidth}
\renewcommand\arraystretch{1.0}
\centering
\setlength{\tabcolsep}{5pt}
\resizebox{\linewidth}{!}{
\begin{tabular}{cccc}
\multicolumn{4}{c}{\textbf{(b) Wall-clock adaptation time across tasks}} \\
\toprule
\#Tasks & Adaptation steps & Total Time (s) & Avg./task (s) \\
\midrule
8  & 50 & 78  & 9.8   \\
14 & 50 & 138 & 9.9   \\
20 & 50 & 211 & 10.6  \\
\bottomrule
\end{tabular}}
\setlength{\tabcolsep}{10pt}
\resizebox{\linewidth}{!}{
\begin{tabular}{lccc}
\multicolumn{4}{c}{\textbf{(c) Results with varying samples/class}} \\
\toprule
\textbf{\# Samples/class} & \textbf{8-task} & \textbf{14-task} & \textbf{20-task} \\
\midrule
0 (Static)  & 78.7{\tiny{$\pm$0.1}} & 76.4{\tiny{$\pm$1.0}} & 70.6{\tiny{$\pm$0.4}} \\
\midrule
1   & 88.4{\tiny{$\pm$0.4}} & 84.6{\tiny{$\pm$1.1}} & 81.7{\tiny{$\pm$1.9}} \\
3   & 88.6{\tiny{$\pm$0.4}} & 84.7{\tiny{$\pm$1.0}} & 81.7{\tiny{$\pm$1.0}} \\
5   & 88.3{\tiny{$\pm$0.7}} & 84.9{\tiny{$\pm$0.8}} & 81.9{\tiny{$\pm$0.9}} \\
10  & 88.6{\tiny{$\pm$0.3}} & 85.1{\tiny{$\pm$1.1}} & 82.1{\tiny{$\pm$1.2}} \\
\bottomrule
\end{tabular}}
\end{minipage}
\end{table}

\noindent\textbf{Ablation on TTA optimization steps.}  
Tab.~\ref{tab:tta_ablation} reports accuracy and backward transfer of CLIP ViT-B/32 under different adaptation steps across 8, 14, and 20 tasks. Accuracy improves steadily with longer schedules, rising from 60--63\% at 5 steps to 77--86\% at 50 steps. Forgetting remains negligible, with BWT close to zero in all cases and only a minor drop of about 2\% in the 20-task setting at 50 steps. 
Notably, as few as 20 steps are sufficient to surpass all baselines in Tab.~1, while 50 steps yield the best accuracy with only a modest increase in cost. These results confirm that \textsc{Mingle} is both effective under tight compute budgets and scalable with additional adaptation.

\noindent\textbf{Computation and Parameter Efficiency.}
Tab.~\ref{tab:efficiency_combined} summarizes the efficiency analysis.
In part (a), inserting experts into both \texttt{attn} and \texttt{mlp} layers yields the highest accuracy, 85.8\%, but also the longest adaptation time of 78s and 83k trainable parameters. A lighter hybrid scheme \texttt{qkv \& fc1} with rank 64 reaches 85.0\% accuracy with 36.9k parameters and 65s, offering a better trade-off. The rank sweep shows that raising the rank from 32 to 64 improves accuracy from 84.5\% to 85.0\%, while larger ranks bring little or even negative gain, \textit{e.g.}, rank 768 drops to 83.7\%.
Part (b) reports wall-clock adaptation time as tasks increase: with 20 tasks the total is 211s, averaging about 10s per task. After adaptation the router remains fixed and inference is purely feedforward without TTA, enabling low-latency deployment across all tasks.
Overall, \methodshort{} achieves a strong balance of accuracy, parameter efficiency, and scalability under diverse resource budgets.

\noindent\textbf{Number of Seed Samples.}
The number of seed samples per class is crucial for TTA reliability and efficiency. 
As shown in Tab.~\ref{tab:efficiency_combined} (c), using no samples reduces to the \textit{Static} baseline, where LoRA modules are merged with fixed coefficients (0.3) without adaptation, yielding 70–79\% accuracy. 
Introducing a single sample lifts accuracy to 81–88\%, and adding more samples offers only minor gains. 
Variance across task orders decreases with more samples, making five samples per class a balanced trade-off between performance and efficiency.

\newpage

\begin{figure}[t]
    \centering
    \includegraphics[width=1\linewidth]{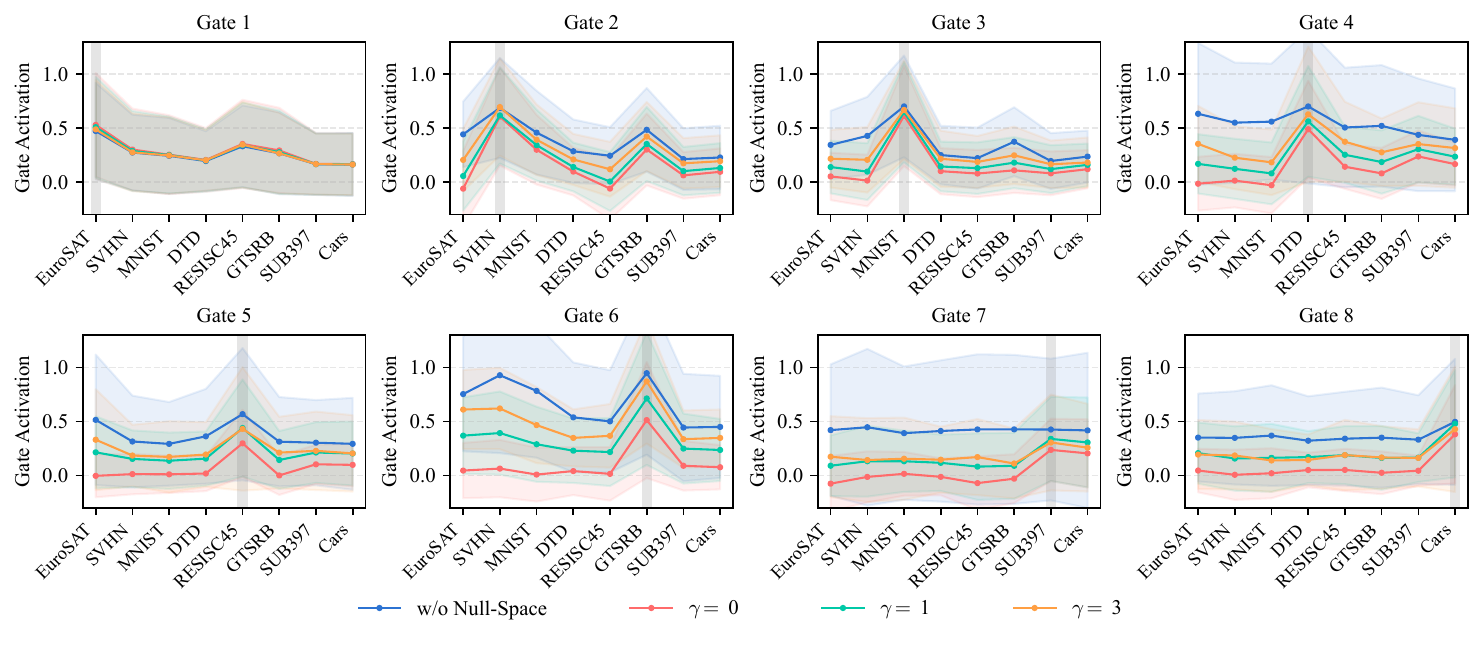}\vspace{-5pt}
    \caption{Gate activations across eight tasks under varying $\gamma$. Each subplot shows one gate; curves and shaded areas indicate mean and std across layers. Gray bars mark the gate’s training task. Lower $\gamma$ leads to stronger suppression on prior tasks.}
    \label{fig:null_gate}
\end{figure}

\begin{figure}[t]
  \centering
  \captionsetup[subfigure]{skip=0pt}
  \hspace{20pt}\begin{subfigure}[t]{0.3\textwidth}
    \centering 
    \includegraphics[width=1\linewidth]{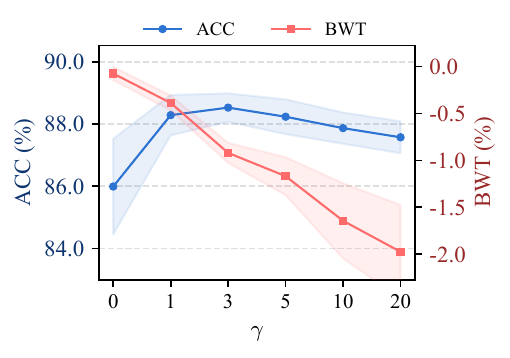}
    \caption{Vary $\gamma$, $\beta=0.99$, $k=3$}
    \label{fig:sensitivity_a}
  \end{subfigure}
  \hfill
  \begin{subfigure}[t]{0.3\textwidth}
    \centering
    \includegraphics[width=1\linewidth]{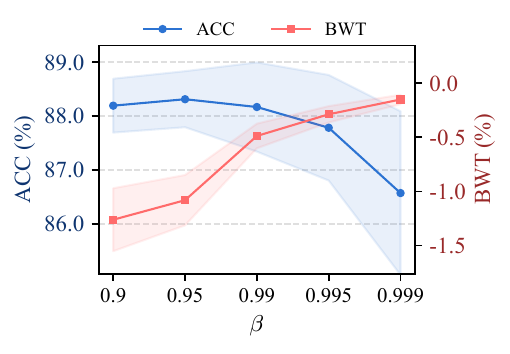}
    \caption{Vary $\beta$, $\gamma=1$, $k=3$}
    \label{fig:sensitivity_b}
  \end{subfigure}
  \hfill
  \begin{subfigure}[t]{0.3\textwidth}
    \centering
    \includegraphics[width=1\linewidth]{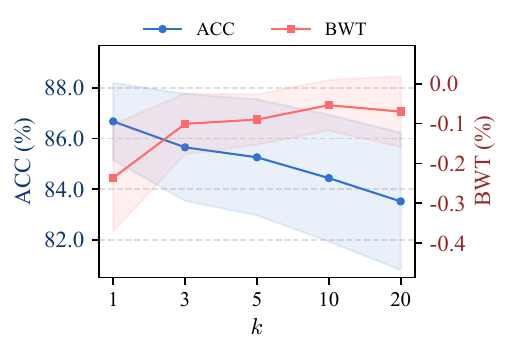}
    \caption{Vary $k$, $\gamma=0$}
    \label{fig:sensitivity_c}
  \end{subfigure}
  \caption{Sensitivity analysis of the null-space constrained gating \textit{w.r.t.} hyper-parameters $\beta$, $\gamma$, and $k$.}
  \label{fig:sensitivity}
\end{figure}

\noindent\textbf{Visualization on Gate Activations.}  
Fig.~\ref{fig:null_gate} shows that the null-space constraint suppresses gate responses on previously seen tasks, reducing forgetting, with smaller $\gamma$ giving stronger attenuation. We also observe that gate activations remain below 1.0 across tasks, even in the w/o Null-Space variant (blue curve), showing that under-activation is not solely due to the constraint. Instead, it reflects the complementary nature of experts: multiple LoRA modules capture overlapping but distinct subspaces, and softly combining them often yields better performance, especially in settings like TTCMM where task boundaries are fuzzy.

\noindent\textbf{Hyper-parameter of Gate.}
We study the effect of three key hyper-parameters: $\gamma$, $\beta$, and $k$. $\gamma$ controls the strength of null-space suppression; smaller values lead to stronger attenuation of activations on prior tasks, reducing forgetting (Fig.~\ref{fig:sensitivity_a}). As shown in Fig.~\ref{fig:sensitivity_b}, $\beta$ regulates the smoothness of the EMA used to accumulate interference signals. A moderate setting ($\beta = 0.99$) balances responsiveness and stability. Smaller $\beta$ amplifies noise sensitivity, while larger $\beta$ slows detection of interference. $k$ determines the number of principal directions retained per task. The mitigation of forgetting saturates at $k = 3$ (Fig.~\ref{fig:sensitivity_c}), indicating that a small number of task-specific directions suffices.

%% file: sections/checklist.tex
\newpage
\section*{NeurIPS Paper Checklist}

\begin{enumerate}

\item {\bf Claims}
    \item[] Question: Do the main claims made in the abstract and introduction accurately reflect the paper's contributions and scope?
    \item[] Answer: \answerYes{} 
    \item[] Justification: The abstract and introduction clearly state the main contributions, which are consistently supported by both theoretical insights and empirical evaluations presented throughout the paper.

\item {\bf Limitations}
    \item[] Question: Does the paper discuss the limitations of the work performed by the authors?
    \item[] Answer: \answerYes{} 
    \item[] Justification: The limitations of the proposed approach are explicitly discussed in Appendix.

\item {\bf Theory assumptions and proofs}
    \item[] Question: For each theoretical result, does the paper provide the full set of assumptions and a complete (and correct) proof?
    \item[] Answer: \answerYes{} 
    \item[] Justification: All theoretical claims are accompanied by clear assumptions and complete proofs in Appendix, with theorem formally stated and rigorously derived.

\item {\bf Experimental result reproducibility}
    \item[] Question: Does the paper fully disclose all the information needed to reproduce the main experimental results of the paper to the extent that it affects the main claims and/or conclusions of the paper (regardless of whether the code and data are provided or not)?
    \item[] Answer: \answerYes{} 
    \item[] Justification: Detailed descriptions of datasets, training protocols, and evaluation metrics are provided in Section~\ref{exp:setup} and Appendix, enabling faithful reproduction of the main results.

\item {\bf Open access to data and code}
    \item[] Question: Does the paper provide open access to the data and code, with sufficient instructions to faithfully reproduce the main experimental results, as described in supplemental material?
    \item[] Answer: \answerNo{} 
    \item[] Justification:The code and data are not provided at submission time, but the authors state they will release them upon acceptance.

\item {\bf Experimental setting/details}
    \item[] Question: Does the paper specify all the training and test details (e.g., data splits, hyperparameters, how they were chosen, type of optimizer, etc.) necessary to understand the results?
    \item[] Answer: \answerYes{} 
    \item[] Justification: The paper includes detailed descriptions of datasets, data splits, training procedures, hyperparameters, and evaluation protocols in both the main text and appendix.

\item {\bf Experiment statistical significance}
    \item[] Question: Does the paper report error bars suitably and correctly defined or other appropriate information about the statistical significance of the experiments?
    \item[] Answer: \answerYes{} 
    \item[] Justification: The paper reports mean and standard deviation across multiple runs with different random seeds and clearly states the sources of variability, ensuring the statistical reliability of the results.

\item {\bf Experiments compute resources}
    \item[] Question: For each experiment, does the paper provide sufficient information on the computer resources (type of compute workers, memory, time of execution) needed to reproduce the experiments?
    \item[] Answer: \answerYes{} 
    \item[] Justification: The paper specifies the type of GPUs used, training time per experiment, and overall compute requirements, providing sufficient details to assess reproducibility and resource demands.
    
\item {\bf Code of ethics}
    \item[] Question: Does the research conducted in the paper conform, in every respect, with the NeurIPS Code of Ethics \url{https://neurips.cc/public/EthicsGuidelines}?
    \item[] Answer: \answerYes{} 
    \item[] Justification: The research adheres to the NeurIPS Code of Ethics, with no identified ethical concerns related to data usage, human subjects, or societal impact.

\item {\bf Broader impacts}
    \item[] Question: Does the paper discuss both potential positive societal impacts and negative societal impacts of the work performed?
    \item[] Answer: \answerYes{} 
    \item[] Justification: The paper includes a Broader Impacts section stating that the work aims to advance machine learning and does not raise specific societal concerns requiring detailed discussion.
    
\item {\bf Safeguards}
    \item[] Question: Does the paper describe safeguards that have been put in place for responsible release of data or models that have a high risk for misuse (e.g., pretrained language models, image generators, or scraped datasets)?
    \item[] Answer: \answerNA{} 
    \item[] Justification:  The paper does not introduce models or datasets with a high risk of misuse, so safeguards are not applicable.

\item {\bf Licenses for existing assets}
    \item[] Question: Are the creators or original owners of assets (e.g., code, data, models), used in the paper, properly credited and are the license and terms of use explicitly mentioned and properly respected?
    \item[] Answer: \answerYes{} 
    \item[] Justification: All datasets and models used in the paper are properly cited, and their licenses and terms of use are respected and included where applicable.

\item {\bf New assets}
    \item[] Question: Are new assets introduced in the paper well documented and is the documentation provided alongside the assets?
    \item[] Answer: \answerNA{} 
    \item[] Justification: The paper does not release any new datasets, models, or code assets, so this question is not applicable.

\item {\bf Crowdsourcing and research with human subjects}
    \item[] Question: For crowdsourcing experiments and research with human subjects, does the paper include the full text of instructions given to participants and screenshots, if applicable, as well as details about compensation (if any)? 
    \item[] Answer: \answerNA{} 
    \item[] Justification: The paper does not involve crowdsourcing or research with human subjects, so this question is not applicable.

\item {\bf Institutional review board (IRB) approvals or equivalent for research with human subjects}
    \item[] Question: Does the paper describe potential risks incurred by study participants, whether such risks were disclosed to the subjects, and whether Institutional Review Board (IRB) approvals (or an equivalent approval/review based on the requirements of your country or institution) were obtained?
    \item[] Answer: \answerNA{} 
    \item[] Justification: The paper does not involve research with human subjects, so IRB approval is not applicable.

\item {\bf Declaration of LLM usage}
    \item[] Question: Does the paper describe the usage of LLMs if it is an important, original, or non-standard component of the core methods in this research? Note that if the LLM is used only for writing, editing, or formatting purposes and does not impact the core methodology, scientific rigorousness, or originality of the research, declaration is not required.
    \item[] Answer: \answerNA{} 
    \item[] Justification: The core methods in this research do not involve LLMs in any important, original, or non-standard way, so this question is not applicable.

\end{enumerate}

%% file: sections/appendix.tex
\newpage
\appendix
\renewcommand{\thetheorem}{A.\arabic{theorem}}
\setcounter{theorem}{0}

\section*{Appendix}


\begin{tcolorbox}[colback=gray!10,colframe=gray]
\noindent
\textbf{A \quad Theoretical Risk Comparison} \dotfill \hyperref[appdix:Theoretical Risk Comparison: Dynamic MoE vs. Static Averaging]{24} \\
\textbf{B \quad Additional Descriptions} \dotfill \hyperref[appdix:Additional Descriptions]{26} \\
\hspace{1em} B.1 \quad Details of Dataset and Task Settings \dotfill \hyperref[appdix:Details of Dataset and Task Settings]{26} \\
\hspace{1em} B.2 \quad Details of Downstream Models \dotfill \hyperref[appdix:Details of Downstream Models]{27} \\
\hspace{1em} B.3 \quad Details of Baselines \dotfill \hyperref[appdix:Details of Baselines]{27} \\
\hspace{1em} B.4 \quad Details of Baseline Hyper-parameters \dotfill \hyperref[appdix:Details of Baseline Hyper-parameters]{29} \\
\hspace{1em} B.5 \quad Comparison of Assumptions and Requirements \dotfill \hyperref[appdix:Comparison of Assumptions and Requirements]{29} \\
\textbf{C \quad Additional Results} \dotfill \hyperref[appdix:Additional Results]{30} \\
\hspace{1em} C.1 \quad Detailed Overall Performance Results \dotfill \hyperref[appdix:Details of Overall Performance]{30} \\
\hspace{1em} C.2 \quad Accuracy Trends Across Sequential Tasks \dotfill \hyperref[appdix:Accuracy Across Sequential Task]{30} \\
\hspace{1em} C.3 \quad Detailed Results Under Distribution Shifts \dotfill \hyperref[appdix:Detail Analysis of Distribution Shift]{30} \\
\hspace{1em} C.4 \quad Inference Efficiency and Parameter Overhead \dotfill \hyperref[appdix:Inference Efficiency and Parameter Overhead]{34} \\
\hspace{1em} C.5 \quad Forward Transfer Analysis \dotfill \hyperref[appdix:Forward Transfer Analysis]{34} \\
\hspace{1em} C.6 \quad Additional  Visualizations of Gate Activations and Relaxation Effect \dotfill \hyperref[appdix:Hyper-parameter Analysis of Gate]{35} \\
\textbf{D \quad Discussions} \dotfill \hyperref[appdix:Discussions]{38} \\
\hspace{1em} D.1 \quad Use of Unlabeled Adaptation Samples \dotfill \hyperref[appdix:Use of Unlabeled Adaptation Samples]{38} \\
\hspace{1em} D.1 \quad Relation to Rehearsal-Free Continual Learning \dotfill \hyperref[appdix:Relation to Rehearsal-Free Continual Learning]{38} \\
\hspace{1em} D.3 \quad Limitations \dotfill \hyperref[appdix:Limitations]{38} \\
\hspace{1em} D.4 \quad Broader Impacts \dotfill \hyperref[appdix:Broader Impacts]{38}
\end{tcolorbox}

\section{Theoretical Risk Comparison: Dynamic MoE vs. Static Averaging}
\label{appdix:Theoretical Risk Comparison: Dynamic MoE vs. Static Averaging}

\paragraph{Problem Setup and Definitions}
Consider $T$ independent tasks, each associated with a data distribution $D_t$ for $t = 1, \dots, T$. For each task $t$, a pre-trained expert model $f_t(x)$ outputs a probability distribution over classes, trained specifically on $D_t$. The overall data distribution $D$ is a mixture of these tasks, where an example $(x, y)$ is drawn from task $t$ with prior probability $P(t)$, and then $(x, y) \sim D_t$. The expected risk of a predictive model $h(x)$ is defined as:
\begin{equation}
R(h) = \mathbb{E}_{(x,y) \sim D} \big[ \ell(h(x), y) \big] = \sum_{t=1}^T P(t) \, \mathbb{E}_{(x,y) \sim D_t} \big[ \ell(h(x), y) \big],
\end{equation}
where $\ell(h(x), y)$ is a loss function (\textit{e.g.}, cross-entropy or 0-1 loss).

We compare two methods to combine the experts into a final prediction $h(x)$:
\begin{itemize}[leftmargin=20pt]
    \item \textbf{Static Averaging}: Defined as $h_{\text{static}}(x) = \sum_{i=1}^T \alpha_i f_i(x)$, where $\boldsymbol{\alpha} = (\alpha_1, \dots, \alpha_T)$ is a fixed weight vector independent of $x$, typically with $\alpha_i \geq 0$ and $\sum_i \alpha_i = 1$ for probability outputs.
    \item \textbf{Dynamic Mixture-of-Experts (MoE)}: Defined as $h_{\text{MoE}}(x) = f_{i^*(x)}(x)$, where $i^*(x) = \arg\max_i g_i(x)$ and $g(x) = (g_1(x), \dots, g_T(x))$ is a gating function that selects one expert per input (hard routing). The gating is subject to routing noise, modeled below.
\end{itemize}

\paragraph{Routing Noise Model}For each input $x$ drawn from $D_t$, the true task is $t$, and the ideal expert is $f_t$. The gating selects the correct expert $i^*(x) = t$ with probability $1 - \varepsilon_t$, and an incorrect expert $i^*(x) \neq t$ with probability $\varepsilon_t = P(i^*(x) \neq t \mid x \sim D_t)$, the task-specific routing error rate. On error, the gating selects a random expert from $\{1, \dots, T\} \setminus \{t\}$ uniformly. Define:
\begin{itemize}[leftmargin=20pt]
    \item $R_t(i) = \mathbb{E}_{(x,y) \sim D_t} [ \ell(f_i(x), y) ]$, the risk of expert $i$ on task $t$.
    \item $R_{\text{ideal}} = \sum_{t=1}^T P(t) R_t(t)$, the risk with perfect routing.
    \item $R_{\text{wrong}, t} = \frac{1}{T-1} \sum_{i \neq t} R_t(i)$, the average risk of incorrect experts on task $t$.
    \item $\varepsilon = \sum_{t=1}^T P(t) \varepsilon_t$, the overall routing error rate.
\end{itemize}

\vspace{10pt}
\begin{tcolorbox}[colback=gray!10,colframe=gray]
\begin{theorem}[Dynamic MoE versus Static Averaging]
Let $\{(D_t,f_t)\}_{t=1}^T$ be $T$ independent tasks with priors
$P(t)$ and per-task risks $R_t(i)$.  
For any static mixture
$h_{\mathrm{static}}(x)=\sum_{i=1}^T\alpha_i\,f_i(x)$
and any hard-routed MoE
$h_{\mathrm{MoE}}(x)=f_{i^\star(x)}(x)$
with task-specific routing errors $\varepsilon_t$:
\begin{equation}
      R(h_{\mathrm{MoE}})
  \;=\;
  R_{\mathrm{ideal}}
  +\sum_{t=1}^T
    P(t)\,\varepsilon_t\bigl(R_{\text{wrong},t}-R_t(t)\bigr),
\end{equation}
where
$R_{\mathrm{ideal}}=\sum_tP(t)R_t(t)$ and
$R_{\text{wrong},t}=\frac{1}{T-1}\sum_{i\neq t}R_t(i)$.
Moreover,
\begin{enumerate}[leftmargin=10pt]
  \item \emph{(Perfect routing)}\;If $\varepsilon_t=0$ for all $t$, then
        $\displaystyle\inf_{g}R(h_{\mathrm{MoE}})
        \;<\;
        \inf_{\boldsymbol{\alpha}}R(h_{\mathrm{static}})$
        whenever at least two tasks disagree on their best expert.
  \item \emph{(Noisy routing)}\;
        If
        $\displaystyle
        \sum_{t}P(t)\varepsilon_t
        \bigl(R_{\text{wrong},t}-R_t(t)\bigr)
        <
        R^*_{\mathrm{static}}-R_{\mathrm{ideal}},$
        where $R^*_{\text{static}} = \inf_{\boldsymbol{\alpha}} R(h_{\text{static}})$,
        then the MoE still attains lower risk than any static mixture.
\end{enumerate}
\end{theorem}
\end{tcolorbox}

\begin{proof}
The proof proceeds in three parts: (1) deriving the MoE risk with routing noise, (2) proving the optimal gating case, and (3) establishing the condition for MoE superiority under routing noise.

\textbf{Step 1: MoE Risk with Routing Noise}

The MoE prediction is $h_{\text{MoE}}(x) = f_{i^*(x)}(x)$, where $i^*(x) = \arg\max_i g_i(x)$. The expected risk is:
\begin{equation}
R(h_{\text{MoE}}) = \mathbb{E}_{(x,y) \sim D} [ \ell(f_{i^*(x)}(x), y) ] = \sum_{t=1}^T P(t) \, \mathbb{E}_{(x,y) \sim D_t} [ \ell(f_{i^*(x)}(x), y) ].
\end{equation}
For task $t$, condition on routing correctness:
\begin{itemize}[leftmargin=20pt]
    \item \textbf{Correct routing} ($i^*(x) = t$): Probability $1 - \varepsilon_t$, risk $R_t(t)$.
    \item \textbf{Incorrect routing} ($i^*(x) \neq t$): Probability $\varepsilon_t$, selects a random expert from $\{1, \dots, T\} \setminus \{t\}$, with average risk $R_{\text{wrong}, t} = \frac{1}{T-1} \sum_{i \neq t} R_t(i)$.
\end{itemize}
The expected risk on task $t$ is:
\begin{equation}
\mathbb{E}_{(x,y) \sim D_t} [ \ell(f_{i^*(x)}(x), y) ] = (1 - \varepsilon_t) R_t(t) + \varepsilon_t R_{\text{wrong}, t}.
\end{equation}
Thus, the total risk is:
\begin{equation}
R(h_{\text{MoE}}) = \sum_{t=1}^T P(t) \big[ (1 - \varepsilon_t) R_t(t) + \varepsilon_t R_{\text{wrong}, t} \big].
\end{equation}
Rewrite:
\begin{equation}
R(h_{\text{MoE}}) = \sum_{t=1}^T P(t) R_t(t) + \sum_{t=1}^T P(t) \varepsilon_t (R_{\text{wrong}, t} - R_t(t)) = R_{\text{ideal}} + \sum_{t=1}^T P(t) \varepsilon_t \delta_t,
\end{equation}
where $\delta_t = R_{\text{wrong}, t} - R_t(t) > 0$ is the risk increase due to misrouting on task $t$.

\textbf{Step 2: Optimal Gating (No Routing Noise)}

Assume an oracle gating function with $\varepsilon_t = 0$ for all $t$, so $i^*(x) = t$ for all $x \sim D_t$. Then:
\begin{equation}
R(h_{\text{MoE}}) = \sum_{t=1}^T P(t) R_t(t) = R_{\text{ideal}}.
\end{equation}
Define hypothesis classes:
\begin{equation}
\mathcal{H}_{\text{static}} = \left\{ x \mapsto \sum_{i=1}^T \alpha_i f_i(x) \mid \boldsymbol{\alpha} \in \mathbb{R}^T \right\},
\end{equation}
\begin{equation}
\mathcal{H}_{\text{MoE}} = \left\{ x \mapsto f_{i^*(x)}(x) \mid i^*(x) = \arg\max_i g_i(x), g: \mathcal{X} \to \mathbb{R}^T \right\}.
\end{equation}
Any static model $h_{\text{static}}(x) = \sum_{i=1}^T \alpha_i f_i(x)$ can be approximated by an MoE with $g(x)$ assigning constant weights, so $\mathcal{H}_{\text{static}} \subseteq \mathcal{H}_{\text{MoE}}$. Thus:
\begin{equation}
R^*_{\text{MoE}} = \inf_{g(\cdot)} R(h_{\text{MoE}}) \leq \inf_{\boldsymbol{\alpha}} R(h_{\text{static}}) = R^*_{\text{static}}.
\end{equation}
Under task heterogeneity ($R_t(t) < R_t(s)$ and $R_s(s) < R_s(t)$ for some $t \neq s$), the ideal MoE routes each $x \sim D_t$ to $f_t$, achieving:
\begin{equation}
R_{\text{ideal}} = \sum_{t=1}^T P(t) R_t(t).
\end{equation}
For static averaging:
\begin{equation}
R(h_{\text{static}}) = \sum_{t=1}^T P(t) \mathbb{E}_{(x,y) \sim D_t} \left[ \ell\left( \sum_{i=1}^T \alpha_i f_i(x), y \right) \right].
\end{equation}
Since $\ell$ is convex (\textit{e.g.}, cross-entropy), Jensen’s inequality implies:
\begin{equation}
\mathbb{E}_{D_t} \left[ \ell\left( \sum_i \alpha_i f_i(x), y \right) \right] \geq R_t(t),
\end{equation}
with strict inequality unless $\alpha_t = 1$ and $\alpha_i = 0$ for $i \neq t$, which cannot hold for all tasks simultaneously under heterogeneity. Thus:
\begin{equation}
R^*_{\text{static}} > R_{\text{ideal}} = R^*_{\text{MoE}}.
\end{equation}

\textbf{Step 3: MoE Superiority with Routing Noise}

Let $\gamma = R^*_{\text{static}} - R_{\text{ideal}} > 0$ under task heterogeneity. The MoE outperforms the static model if:
\begin{equation}
R(h_{\text{MoE}}) < R^*_{\text{static}}.
\end{equation}
Substitute:
\begin{equation}
R_{\text{ideal}} + \sum_{t=1}^T P(t) \varepsilon_t \delta_t < R_{\text{ideal}} + \gamma.
\end{equation}
Thus:
\begin{equation}
\sum_{t=1}^T P(t) \varepsilon_t \delta_t < \gamma = R^*_{\text{static}} - R_{\text{ideal}}.
\end{equation}
Since $\delta_t = R_{\text{wrong}, t} - R_t(t)$, the condition is:
\begin{equation}
\sum_{t=1}^T P(t) \varepsilon_t (R_{\text{wrong}, t} - R_t(t)) < R^*_{\text{static}} - R_{\text{ideal}}.
\end{equation}
If this holds, the MoE’s risk, despite routing noise, remains below the best static risk.
\end{proof}

\textbf{Conclusion}: The MoE risk is $R_{\text{ideal}} + \sum_{t=1}^T P(t) \varepsilon_t (R_{\text{wrong}, t} - R_t(t))$, and it outperforms static averaging when routing noise is sufficiently small relative to the static model’s suboptimality. The optimal gating case confirms $R^*_{\text{MoE}} \leq R^*_{\text{static}}$, with strict inequality under task heterogeneity.

\section{Additional Descriptions}
\label{appdix:Additional Descriptions}

\subsection{Details of Dataset and Task Settings}
\label{appdix:Details of Dataset and Task Settings}
\paragraph{Dataset Details}
Following prior works \cite{tang2025merging}, we evaluate continual model merging on twenty publicly available image classification datasets, including
SUN397~\cite{xiao2010sun}, %
Stanford Cars~\cite{krause20133d}, %
RESISC45~\cite{cheng2017remote}, %
EuroSAT~\cite{helber2019eurosat}, %
SVHN~\cite{netzer2011reading}, %
GTSRB~\cite{stallkamp2012man}, %
MNIST~\cite{lecun1998mnist}, %
DTD~\cite{cimpoi2014describing}, %
Flowers102~\cite{nilsback2008automated}, %
PCAM~\cite{veeling2018rotation}, %
FER2013~\cite{goodfellow2013challenges}, %
Oxford-IIIT Pet~\cite{parkhi2012cats}, %
STL-10~\cite{coates2011analysis}, %
CIFAR-100 and CIFAR-10~\cite{krizhevsky2009learning}, %
Food-101~\cite{bossard2014food}, %
Fashion-MNIST~\cite{xiao2017fashion}, %
EMNIST~\cite{cohen2017emnist}, %
KMNIST~\cite{clanuwat2018deep}, %
and Rendered SST-2~\cite{socher2013recursive}.

\begin{table}[t]
\centering
\small
\caption{Extended downstream datasets used in our experiments.}
\setlength{\tabcolsep}{4pt}
\renewcommand\arraystretch{1.05}
\begin{tabular}{lcccc}
\toprule
\textbf{Dataset} & \textbf{\#Classes} & \textbf{\#Train (k)} & \textbf{\#Test (k)} & \textbf{Task} \\
\midrule
SUN397                     & 287 & 19.9 & 19.9 & Scene category \\
Stanford Cars              & 196 &  8.1 &  8.0 & Car series \\
RESISC45                   &  45 & 18.9 &  6.3 & Remote–sensing scene \\
EuroSAT                    &  10 & 21.6 &  2.7 & Satellite land-use \\
SVHN                       &  10 & 73.3 & 26.0 & Digit recognition \\
GTSRB                      &  43 & 39.2 & 12.6 & Traffic sign \\
MNIST                      &  10 & 60   & 10   & Hand-written digit \\
DTD                        &  47 &  3.8 &  1.9 & Texture recognition \\
Flowers102                 & 102 &  1.0 &  6.1 & Flower species \\
PCAM                       &   2 & 262  & 32.8 & Tumour detection \\
FER2013                    &   7 & 28.7 &  3.6 & Facial emotion \\
Oxford IIIT Pet            &  37 &  3.7 &  3.7 & Animal species \\
STL10                      &  10 &  5   &  8   & Object recognition \\
CIFAR-100                  & 100 & 50   & 10   & Natural object \\
CIFAR-10                   &  10 & 50   & 10   & Natural object \\
Food101                    & 101 & 75.8 & 25.3 & Food type \\
Fashion-MNIST              &  10 & 60   & 10   & Fashion product \\
EMNIST (digits)            &  10 & 60   & 10   & Hand-written digit \\
KMNIST                     &  10 & 60   & 10   & Kuzushiji character \\
Rendered SST-2             &   2 &  6.9 &  1.8 & Rendered sentiment \\
\bottomrule
\end{tabular}
\end{table}

\paragraph{Task Grouping}
We group the 20 datasets into three progressive task sets and evaluate the merged models using average accuracy (ACC) and backward transfer (BWT) metrics. For each task group, we perform 10 experiments using different task sequences (listed in Tab.~ \ref{tab:appendix_dataset_order}), and report both the mean and standard deviation of the results to ensure robustness and consistency.

\begin{itemize}[leftmargin=20pt]
    \item \textbf{8-task group}: 
    (1) SUN397, (2) Stanford Cars, (3) RESISC45, (4) EuroSAT, (5) SVHN, (6) GTSRB, (7) MNIST, (8) DTD.
    
    \item \textbf{14-task group}: 
    (1) SUN397, (2) Stanford Cars, (3) RESISC45, (4) EuroSAT, (5) SVHN, (6) GTSRB, (7) MNIST, (8) DTD, (9) Flowers102, (10) PCAM, (11) FER2013, (12) OxfordIIITPet, (13) STL10, (14) CIFAR100.
    
    \item \textbf{20-task group}: 
    (1) SUN397, (2) Stanford Cars, (3) RESISC45, (4) EuroSAT, (5) SVHN, (6) GTSRB, (7) MNIST, (8) DTD, (9) Flowers102, (10) PCAM, (11) FER2013, (12) OxfordIIITPet, (13) STL10, (14) CIFAR100, (15) CIFAR10, (16) Food101, (17) FashionMNIST, (18) EMNIST, (19) KMNIST, (20) RenderedSST2.
\end{itemize}

\newcommand{\arr}{$\,\rightarrow\,$}   
\begin{table}[h]
\centering
\setlength{\tabcolsep}{3pt}
\renewcommand{\arraystretch}{1.3}
\caption{Dataset orderings used for experiments in each task group.}
\label{tab:task_orderings}
\resizebox{1\textwidth}{!}{
\begin{tabular}{c|c|l}
\toprule
\textbf{Group} & \textbf{Order} & \textbf{Dataset Order (by ID)} \\
\midrule
\multirow{10}{*}{8 tasks}
& 1  & (04\arr05\arr07\arr08\arr03\arr06\arr01\arr02)\\
& 2  & (07\arr08\arr05\arr04\arr02\arr06\arr03\arr01)\\
& 3  & (03\arr06\arr04\arr02\arr01\arr08\arr05\arr07)\\
& 4  & (06\arr08\arr02\arr01\arr03\arr07\arr04\arr05)\\
& 5  & (07\arr06\arr03\arr08\arr05\arr01\arr04\arr02)\\
& 6  & (07\arr02\arr03\arr08\arr05\arr04\arr01\arr06)\\
& 7  & (07\arr01\arr04\arr03\arr08\arr05\arr02\arr06)\\
& 8  & (08\arr05\arr06\arr07\arr01\arr04\arr03\arr02)\\
& 9  & (01\arr04\arr05\arr02\arr06\arr03\arr07\arr08)\\
& 10 & (08\arr03\arr01\arr02\arr06\arr05\arr07\arr04)\\
\midrule
\multirow{10}{*}{14 tasks}
& 1  & (09\arr13\arr08\arr07\arr14\arr12\arr06\arr03\arr10\arr04\arr05\arr01\arr02\arr11)\\
& 2  & (09\arr10\arr11\arr14\arr07\arr13\arr04\arr02\arr06\arr08\arr03\arr12\arr05\arr01)\\
& 3  & (05\arr08\arr12\arr06\arr11\arr01\arr10\arr04\arr14\arr03\arr02\arr13\arr09\arr07)\\
& 4  & (03\arr10\arr09\arr12\arr04\arr13\arr01\arr06\arr11\arr02\arr14\arr08\arr07\arr05)\\
& 5  & (08\arr14\arr09\arr06\arr12\arr13\arr05\arr03\arr04\arr11\arr10\arr01\arr07\arr02)\\
& 6  & (03\arr12\arr13\arr01\arr11\arr04\arr10\arr05\arr14\arr08\arr09\arr07\arr02\arr06)\\
& 7  & (07\arr01\arr12\arr10\arr02\arr08\arr13\arr04\arr05\arr11\arr14\arr03\arr06\arr09)\\
& 8  & (05\arr12\arr04\arr11\arr03\arr08\arr10\arr01\arr09\arr13\arr14\arr07\arr06\arr02)\\
& 9  & (10\arr07\arr09\arr02\arr03\arr13\arr01\arr12\arr14\arr04\arr11\arr06\arr05\arr08)\\
& 10 & (01\arr02\arr11\arr06\arr08\arr12\arr07\arr05\arr10\arr14\arr03\arr13\arr09\arr04)\\
\midrule
\multirow{10}{*}{20 tasks}
& 1  & (20\arr06\arr15\arr05\arr10\arr14\arr16\arr19\arr07\arr13\arr18\arr11\arr02\arr12\arr03\arr17\arr08\arr09\arr01\arr04)\\
& 2  & (09\arr14\arr06\arr03\arr07\arr04\arr18\arr01\arr17\arr19\arr08\arr20\arr13\arr16\arr11\arr12\arr15\arr05\arr10\arr02)\\
& 3  & (09\arr15\arr16\arr11\arr03\arr13\arr08\arr10\arr12\arr02\arr20\arr01\arr05\arr19\arr07\arr06\arr04\arr18\arr17\arr14)\\
& 4  & (17\arr04\arr11\arr19\arr18\arr10\arr07\arr15\arr12\arr13\arr08\arr02\arr01\arr06\arr05\arr03\arr20\arr16\arr14\arr09)\\
& 5  & (14\arr16\arr04\arr20\arr15\arr17\arr07\arr11\arr06\arr18\arr12\arr01\arr19\arr09\arr10\arr05\arr08\arr02\arr13\arr03)\\
& 6  & (02\arr06\arr17\arr04\arr19\arr18\arr08\arr16\arr20\arr01\arr10\arr13\arr07\arr09\arr05\arr11\arr15\arr14\arr03\arr12)\\
& 7  & (19\arr01\arr09\arr14\arr06\arr20\arr17\arr04\arr08\arr02\arr15\arr03\arr16\arr13\arr12\arr07\arr10\arr05\arr11\arr18)\\
& 8  & (15\arr07\arr08\arr02\arr10\arr06\arr17\arr20\arr05\arr19\arr16\arr01\arr18\arr09\arr13\arr11\arr04\arr14\arr12\arr03)\\
& 9  & (10\arr05\arr07\arr11\arr01\arr03\arr17\arr15\arr18\arr04\arr14\arr19\arr02\arr06\arr13\arr20\arr08\arr12\arr09\arr16)\\
& 10 & (01\arr11\arr02\arr15\arr03\arr10\arr12\arr19\arr16\arr13\arr07\arr05\arr09\arr04\arr14\arr20\arr06\arr18\arr17\arr08)\\
\bottomrule
\end{tabular}
}
\label{tab:appendix_dataset_order}
\end{table}

\subsection{Details of Downstream Models}
\label{appdix:Details of Downstream Models}

In this section, we present the evaluation setup for pre-trained and fine-tuned models. As shown in Tab. \ref{tab:single_model_vit}, we evaluate the zero-shot accuracy of the original CLIP-ViT models and the performance of fine-tuned models on the test sets of various downstream tasks. The fine-tuned checkpoints are obtained directly from Hugging Face (\url{https://huggingface.co/tanganke}), where each model has been fine-tuned on task-specific training data using a standard protocol. The visual encoder is updated during fine-tuning, while the classification head is fixed and initialized from the pre-trained text encoder. The fine-tuning setup follows a standard configuration: cross-entropy loss, Adam optimizer, cosine annealing learning rate schedule with a peak learning rate of 1e-5, batch size 128, and 4000 training steps.

\begin{table}[tb]
  \centering
  \setlength{\tabcolsep}{4pt}
  \caption{Test set accuracy of the pre-trained model and individual fine-tuned models on different downstream tasks.}
  \label{tab:single_model_vit}
  \renewcommand\arraystretch{1.05}
  \resizebox{\textwidth}{!}{
  \begin{tabular}{lcccccccccc}
    \toprule
    \textbf{Model}       & \rot{\small{SUN397}} & \rot{\small{Cars}} & \rot{\small{RESISC45}} & \rot{\small{EuroSAT}} & \rot{\small{SVHN}} & \rot{\small{GTSRB}} & \rot{\small{MNIST}} & \rot{\small{DTD}} & \rot{\small{Flowers102}}\hspace{-5pt} & \rot{\small{PCAM}} \\
    \midrule
    \multicolumn{11}{l}{\textbf{CLIP-ViT-B/32}} \\ \cmidrule[0.5pt](lr){1-2}
    Pre-trained & 63.2    & 59.6   & 60.3     & 45.0    & 31.6   & 32.5   & 48.3   & 44.2   & 66.4       & 60.6 \\
    Fine-tuned  & 74.9    & 78.5   & 95.1     & 99.1    & 97.3   & 98.9   & 99.6   & 79.7   & 88.6       & 88.0 \\ \midrule
    \multicolumn{11}{l}{\textbf{CLIP-ViT-B/16}} \\ \cmidrule[0.5pt](lr){1-2}
    Pre-trained & 65.5    & 64.7   & 66.4     & 54.1    & 52.0   & 43.5   & 51.7   & 45.0   & 71.3       & 54.0 \\
    Fine-tuned  & 78.9    & 85.9   & 96.6     & 99.0    & 97.6   & 99.0   & 99.7   & 82.3   & 94.9       & 90.6 \\ \midrule
    \multicolumn{11}{l}{\textbf{CLIP-ViT-L/14}} \\ \cmidrule[0.5pt](lr){1-2}
    Pre-trained & 68.2    & 77.9   & 71.3     & 61.2    & 58.4   & 50.5   & 76.3   & 55.5   & 79.2       & 51.2 \\
    Fine-tuned  & 82.8    & 92.8   & 97.4     & 99.1    & 97.9   & 99.2   & 99.8   & 85.5   & 97.7       & 91.1 \\
\midrule
    \textbf{Model}       &  \rot{\small{FER2013}} & \rot{\small{OxfordIIITPet}}\hspace{-15pt} & \rot{\small{STL10}} & \rot{\small{CIFAR100}} & \rot{\small{CIFAR10}} & \rot{\small{Food101}} & \rot{\small{FashionMNIST}}\hspace{-10pt} & \rot{\small{EMNIST}} & \rot{\small{KMNIST}} & \rot{\small{RenderedSST2}} \\
    \midrule
    \multicolumn{11}{l}{\textbf{CLIP-ViT-B/32}} \\ \cmidrule[0.5pt](lr){1-2}
    Pre-trained & 41.3    & 83.3          & 97.1  & 63.7     & 89.8    & 82.4    & 63.0         & 12.0   & 10.0   & 58.6         \\
    Fine-tuned  & 71.6    & 92.5          & 97.5  & 88.4     & 97.6    & 88.4    & 94.7         & 95.6   & 98.2   & 71.3         \\ \midrule
    \multicolumn{11}{l}{\textbf{CLIP-ViT-B/16}} \\ \cmidrule[0.5pt](lr){1-2}
    Pre-trained & 46.4    & 88.4          & 98.3  & 66.3     & 90.8    & 87.0    & 67.3         & 12.4   & 11.2   & 60.6         \\
    Fine-tuned  & 72.8    & 94.5          & 98.2  & 88.8     & 98.3    & 91.9    & 94.5         & 95.3   & 98.1   & 75.7         \\ \midrule
    \multicolumn{11}{l}{\textbf{CLIP-ViT-L/14}} \\ \cmidrule[0.5pt](lr){1-2}
    Pre-trained & 50.0    & 93.2          & 99.4  & 75.1     & 95.6    & 91.2    & 67.0         & 12.3   & 9.7    & 68.9         \\
    Fine-tuned  & 75.9    & 95.7          & 99.2  & 93.0     & 99.1    & 94.8    & 95.3         & 95.4   & 98.3   & 80.5         \\
    \bottomrule
  \end{tabular}
  }
\end{table}

\subsection{Details of Baselines}
\label{appdix:Details of Baselines}
Our experiments involve the following comparison methods and our method:
\begin{itemize}[leftmargin=20pt]
    \item \textbf{Stochastic Weight Averaging (SWA).} A simple model averaging technique to stabilize optimization and improve generalization~\cite{izmailov2018averaging}. At each step $t$, the model parameters are averaged across previous checkpoint: $\theta^{\text{SWA}}_t = \frac{1}{t} \left[ \theta^{\text{SWA}}_{t-1}(t-1)+\theta^{\text{SWA}}_t \right]$. This approach can be interpreted as a form of uniform model ensembling. While conceptually straightforward, SWA treats all checkpoints equally and does not account for inter-task conflicts.

    \item \textbf{Continual Task Arithmetic (C. TA).} A training-free merging strategy that linearly combines task-specific fine-tuned models with a shared pre-trained model~\cite{ilharco2023editing}. It computes the merged parameters as $\theta^{\text{merged}}_t = \theta^{\text{merged}}_{t-1} + \lambda (\theta_t - \theta_0)$, where $\lambda$ is a scaling factor. TA is computationally efficient and easy to apply, but sensitive to $\lambda$ and prone to destructive interference when merging dissimilar tasks.

    \item \textbf{Continual Ties-Merging (C. Ties).} An extension of Task Arithmetic that reduces parameter-level redundancy and sign conflicts during model merging~\cite{yadav2023ties}. For task $t$, the difference vector $\Delta\theta_t = \theta_t - \theta_0$ is trimmed and sign-normalized to obtain $\Delta\theta^{\text{Ties}}_t=\mathrm{Ties}{\left(\Delta\theta^{\text{Ties}}_{t-1}, \Delta\theta_{t}\right)}$, and the merged model is given by $\theta^{\text{merged}}_t = \theta^{\text{merged}}_{t-1} + \lambda \Delta\theta^{\text{Ties}}_t$. 

    \item \textbf{Orthogonal Projection-based Continual Merging 
     (OPCM).}  A projection-based scheme to mitigate task interference by enforcing orthogonality between parameter updates \cite{tang2025merging}. Specifically, each $\Delta\theta_{t}$ is projected onto the orthogonal complement of the subspace spanned by previous updates: $\theta^{\text{merged}}_t = \theta_0 + \frac{1}{\lambda_t} \left[ \lambda_{t-1} \Delta\theta^{\text{merged}}_{t-1} + \mathcal{P}^{(t-1)}(\Delta\theta_t) \right]$, where $\mathcal{P}^{(t-1)}$ denotes the orthogonal projection.

     \item \textbf{Maximum Magnitude Selection (MagMax).}  
     An extension of Task Arithmetic that, for each parameter dimension, selects the update with the larger absolute value: $\Delta\theta^{\text{MagMax}}_t=\mathrm{MagMax}{\left(\Delta\theta^{\text{MagMax}}_{t-1}, \Delta\theta_{t}\right)}$, and the merged model is given by $\theta^{\text{merged}}_t = \theta^{\text{merged}}_{t-1} + \lambda \Delta\theta^{\text{MagMax}}_t$. 
     
\end{itemize}

\subsection{Details of Baseline Hyper-parameters}
\label{appdix:Details of Baseline Hyper-parameters}
Tab.~\ref{tab:hyper-parameters} summarizes the hyper-parameters for all baseline methods under different task configurations (8, 14, 20 tasks). 
\emph{Top-k} denotes the pruning ratio, \emph{TALL} the TALL mask threshold, and \emph{Cons.} the consensus mask threshold. 
The column \emph{LR} is the learning rate, while \emph{Steps} indicates the number of adaptation steps. 
\emph{$r$} represents the LoRA rank, and the last column jointly reports the null dimension ($k$), EMA decay ($\beta$), and relaxation coefficient ($\gamma$). 

\begin{table}[htbp]
  \centering
  \caption{hyper-parameter settings for all baselines.}
  \label{tab:hyper-parameters} 
  \renewcommand\arraystretch{1.08}
    \renewcommand\tabcolsep{9pt}  
    \resizebox{\textwidth}{!}{
  \begin{tabular}{llcccccccc}
    \toprule
    Method            & Tasks & $\lambda$ & Top-k & TALL & Cons. & LR & Steps & $r$ & $k/\beta/\gamma$ \\
    \midrule
    \multirow{2}{*}{\textsc{Task Arithmetic}}   & 8       & 0.3    & -   & -   & -  & -  & -  & -  & -  \\
                                       & 14/20   & 0.1   & -   & -   & -  & -  & -  & -  & -  \\
     
    \multirow{2}{*}{\textsc{Ties-Merging}}      & 8       & 0.3    & 20  & -   & -  & -  & -  & -  & -  \\
                                       & 14/20   & 0.1    & 20  & -   & -  & -  & -  & -  & -  \\
    \textsc{Consensus TA}                         & 8/14/20 & 0.1  & -   & 0.2 & 2  & -  & -  & -  & -  \\

    \textsc{LW. AdaMerging}                     & 8/14/20 & 0.3   & -   & -   & -  & 1e-4& 50 & -  & -  \\
    \textsc{WEMOE}                              & 8/14/20 & 0.3   & -   & -   & -  & 1e-4& 50 & -  & -  \\
    \textsc{LoRA-WEMOE}                         & 8/14/20 & 0.3   & -   & -   & -  & 1e-4& 50 & 64 & -  \\
    \textsc{MINGLE-Static}                      & 8/14/20 & 0.3   & -   & -   & -  & -  & -  & -  & -  \\
    \textbf{\methodshort{} (Ours)}              & 8/14/20 & -     & -   & -   & -  & 1e-4& 50 & 64 & 3/0.99/1.0 \\
    \bottomrule
  \end{tabular}}
\end{table}

\subsection{Comparison of Assumptions and Requirements}
\label{appdix:Comparison of Assumptions and Requirements}
Tab.~\ref{tab:requirements} summarizes the assumptions and resource requirements of all baseline methods. 
We report whether each method requires storing intermediate activations, introduces additional parameters (for storage or inference), and incurs extra test-time computation. 
Our method only maintains a fixed-size covariance matrix instead of full activations, leading to constant memory regardless of test set size. 
Although LoRA experts are stored, the router merges them into a single model per input, so the effective inference cost matches that of a standard individual model.

\begin{table}[htbp]
  \centering
  \caption{Comparison of baseline assumptions and requirements.}
  \label{tab:requirements} 
    \renewcommand\tabcolsep{8pt}  
    \resizebox{\textwidth}{!}{
  \begin{tabular}{lcccc}
    \toprule
    \multirow{2}{*}{Method}            & \multirow{2}{*}{Save Activations}  & Extra Parameters& Extra Parameters & \multirow{2}{*}{Test-time Compute} \\
                &  & (Storage) & (Inference) &   \\
    \midrule
    \textsc{Task Arithmetic}   & No               & No                         & No                           & No                \\
    \textsc{Ties-Merging}      & No               & No                         & No                           & No                \\
    \textsc{MAGMAX-Ind}        & No               & No                         & No                           & No                \\
    \textsc{OPCM}              & No               & No                         & No                           & No                \\
    \textsc{Consensus TA}      & No               & Yes                        & No                           & No                \\
    \textsc{LW. AdaMerging}    & No               & No                         & No                           & Yes               \\
    \textsc{WEMOE}             & No               & Yes                        & No                           & Yes               \\
    \textsc{MINGLE-Static}     & No               & No                         & No                           & No                \\
    \textbf{\methodshort{} (Ours)} & No\textsuperscript{1} & Yes\textsuperscript{2} & No\textsuperscript{2} & Yes \\
    \bottomrule
  \end{tabular}}
  
  \vspace{0.5em}
  \footnotesize{
  \textsuperscript{1} Only a fixed-size covariance matrix is maintained, resulting in constant memory regardless of test set size. \\
  \textsuperscript{2} LoRA experts are stored, but the router merges them into a single model per input, making the effective inference size equivalent to a standard individual model.
  }
\end{table}

\section{Additional Results}
\label{appdix:Additional Results}
In this section, we provide additional experimental results to support the findings reported in the main paper. Specifically, we include: (1) detailed overall performance results (\ref{appdix:Details of Overall Performance}); (2) accuracy trends across sequential tasks (\ref{appdix:Accuracy Across Sequential Task}); (3) detailed results under distribution shifts (\ref{appdix:Detail Analysis of Distribution Shift}); and (4) extended visualizations of gate activations and hyper-parameter effects (\ref{appdix:Hyper-parameter Analysis of Gate}).

\subsection{Detailed Overall Performance Results}
\label{appdix:Details of Overall Performance}
Tab.~\ref{tab:continual_results} expands on the average results in Tab.~\ref{tab:results} by reporting per-task average accuracy after continually merging 20 tasks. We compare six methods, SWA, Task Arithmetic, Ties-Merging, MagMax-IND, OPCM, and our proposed \methodshort{} across three CLIP-ViT backbones (B/32, B/16, L/14). \methodshort{} achieves the highest accuracy on most tasks. These fine-grained results reinforce the main paper’s findings, highlighting \methodshort{}’s ability to improve performance on continual model merging.

\begin{table}[!htbp]
\centering
\caption{Test set accuracy comparisons on different downstream tasks. }
\renewcommand\arraystretch{1.02}
\setlength{\tabcolsep}{3pt}
\small
\resizebox{1\textwidth}{!}{
\begin{tabular}{lccccccccccc}
\toprule
\textbf{Model} & \rot{\small{SUN397}} & \rot{\small{Cars}} & \rot{\small{RESISC45}} & \rot{\small{EuroSAT}} & \rot{\small{SVHN}} & \rot{\small{GTSRB}} & \rot{\small{MNIST}} & \rot{\small{DTD}} & \rot{\small{Flowers102}}\hspace{-5pt} & \rot{\small{PCAM}} \\
\midrule
\multicolumn{11}{l}{\textbf{ViT-B/32}} \\
\textsc{C. Fine-Tuned} & 53.9 & 38.2 & 64.7 & 98.7 & 45.4 & 34.4 & 86.7 & 58.4 & 57.5 & 67.7 \\
\textsc{Average (SWA)} & 64.2 & 59.6 & 64.8 & 60.9 & 47.3 & 43.1 & 71.8 & 46.4 & 66.5 & 63.9 \\
\textsc{C.TA} & 62.0 & 53.7 & 60.9 & 58.1 & 48.5 & 48.9 & 79.4 & 46.1 & 61.1 & 73.4 \\
\textsc{C.TIES} & 62.5 & 49.1 & 55.8 & 50.9 & 54.6 & 49.3 & 82.0 & 46.7 & 58.5 & 69.9 \\
\textsc{MagMax-Ind} &63.6 & 53.1 & 59.7 & 49.1 & 53.8 & 53.1 & 79.8 & 43.2 & 56.9 & 75.1 \\
\textsc{Consensus TA}  & 37.0&25.2&35.2&36.7& 37.3 &44.1&80.6&30.3& 33.5 &59.2 \\
\textsc{C. LW AdaMerging} &63.1 & 60.0 & 63.5 & 60.1 & 35.6 & 32.1 & 51.8 & 45.4 & 66.6 & 60.2 \\
\textsc{C. LoRA-WEMOE} &51.4 & 45.8 & 63.3 & 43.5 & 42.9 & 34.6 & 58.9 & 46.5 & 47.5 & 60.1 \\
\textsc{OPCM} & 64.4 & 51.1 & 66.0 & 71.7 & 66.1 & 56.0 & 90.2 & 40.4 & 64.9 & 80.2 \\
\textsc{Mingle (Ours)} & 67.8 & 58.3 & 83.5 & 90.0 & 82.9 & 91.8 & 98.0 & 65.3 & 74.0 & 66.9  \\
\textsc{Mingle$^*$ (Ours)} & 68.8 & 64.2 & 83.8 & 91.1 & 82.4 & 89.0 & 96.9 & 62.8 & 76.7 & 72.8 \\
\midrule
\multicolumn{11}{l}{\textbf{ViT-B/16}} \\
\textsc{C. Fine-Tuned} & 62.7 & 58.0 & 67.6 & 99.1 & 46.0 & 29.2 & 93.9 & 61.9 & 64.1 & 75.2 \\
\textsc{Average (SWA)} & 67.1 & 64.6 & 69.3 & 63.4 & 62.4 & 52.7 & 80.7 & 46.6 & 71.8 & 63.1 \\
\textsc{C.TA} & 65.8 & 57.5 & 63.8 & 59.5 & 64.7 & 54.0 & 88.0 & 45.3 & 67.5 & 67.1 \\
\textsc{C.TIES} & 64.2 & 52.9 & 60.9 & 53.0 & 62.8 & 48.8 & 88.4 & 45.0 & 61.3 & 68.5 \\
\textsc{MagMax-Ind} &65.8 & 51.8 & 57.8 & 42.6 & 54.4 & 43.7 & 83.0 & 42.8 & 60.4 & 69.8\\
\textsc{Consensus TA} & 42.6&24.8&30.4&34.4&47.6&42.2&79.9&30.6&36.2&74.3\\
\textsc{C. LW AdaMerging} &65.5 & 65.7 & 69.8 & 59.4 & 50.1 & 44.2 & 61.1 & 47.1 & 71.8 & 57.9\\
\textsc{C. LoRA-WEMOE} &62.7 & 60.2 & 69.4 & 37.7 & 52.1 & 39.9 & 63.1 & 45.3 & 64.3 & 51.7\\
\textsc{OPCM} & 67.9 & 55.9 & 73.7 & 77.5 & 74.4 & 63.2 & 94.1 & 49.2 & 72.3 & 79.6 \\
\textsc{Mingle (Ours)} &71.5 & 64.9 & 85.3 & 90.0 & 87.5 & 90.1 & 97.1 & 62.7 & 82.6 & 80.6\\
\textsc{Mingle$^*$ (Ours)} &72.0 & 72.1 & 87.9 & 93.3 & 87.1 & 89.2 & 97.4 & 62.5 & 86.8 & 76.4\\
\midrule
\multicolumn{11}{l}{\textbf{ViT-L/14}} \\
\textsc{C. Fine-Tuned} & 69.5 & 73.6 & 78.3 & 99.2 & 59.3 & 49.3 & 98.6 & 69.7 & 83.2 & 78.3 \\
\textsc{Average (SWA)} & 70.7 & 77.7 & 76.4 & 75.3 & 69.5 & 62.1 & 93.7 & 57.7 & 80.0 & 73.6 \\
\textsc{C.TA} & 70.4 & 74.1 & 73.9 & 66.3 & 69.9 & 65.6 & 95.1 & 56.6 & 78.6 & 70.4 \\
\textsc{C.TIES} & 69.7 & 70.3 & 65.3 & 47.9 & 76.1 & 63.6 & 94.7 & 54.4 & 77.9 & 72.3 \\
\textsc{MagMax-Ind} &73.1 & 73.7 & 75.6 & 64.6 & 73.7 & 68.8 & 94.6 & 56.1 & 78.0 & 71.7 \\
\textsc{Consensus TA} & 50.7 & 39.1 & 31.7 & 36.4 & 39.4 & 44.9 & 88.5 & 33.8 & 45.7 & 62.5\\
\textsc{C. LW AdaMerging} &68.8 & 78.6 & 75.9 & 65.7 & 58.3 & 51.6 & 79.9 & 57.4 & 80.6 & 52.4 \\
\textsc{C. LoRA-WEMOE} &62.1 & 68.1 & 68.7 & 53.2 & 47.5 & 49.4 & 69.8 & 49.1 & 66.2 & 54.2\\
\textsc{OPCM} & 73.1 & 78.3 & 82.4 & 80.2 & 80.8 & 80.4 & 97.4 & 61.6 & 84.8 & 76.3 \\
\textsc{Mingle (Ours)} & 75.9 & 83.4 & 87.8 & 88.7 & 91.1 & 94.5 & 98.4 & 70.8 & 94.8 & 75.3 \\
\textsc{Mingle$^*$ (Ours)} & 74.5 & 85.9 & 90.5 & 92.5 & 90.1 & 92.7 & 98.1 & 69.2 & 95.7 & 74.0 \\
\midrule
\textbf{Model} &  \rot{\small{FER2013}} & \rot{\small{OxfordIIITPet}}\hspace{-15pt} & \rot{\small{STL10}} & \rot{\small{CIFAR100}} & \rot{\small{CIFAR10}} & \rot{\small{Food101}} & \rot{\small{FashionMNIST}}\hspace{-10pt} & \rot{\small{EMNIST}} & \rot{\small{KMNIST}} & \rot{\small{RenderedSST2}} \\
\midrule
\multicolumn{11}{l}{\textbf{ViT-B/32}} \\
\textsc{C. Fine-Tuned} & 58.3 & 68.5 & 86.7 & 40.2 & 70.5 & 50.0 & 90.7 & 72.4 & 54.5 & 54.5 \\
\textsc{Average (SWA)} & 50.2 & 84.1 & 97.0 & 69.8 & 92.7 & 80.4 & 71.3 & 15.0 & 11.5 & 61.8 \\
\textsc{C.TA} & 51.4 & 82.3 & 94.9 & 64.6 & 91.4 & 71.9 & 73.9 & 17.8 & 12.2 & 59.9 \\
\textsc{C.TIES} & 49.5 & 81.3 & 95.2 & 63.7 & 91.2 & 70.2 & 73.7 & 17.8 & 16.9 & 59.8\\
\textsc{MagMax-Ind} & 56.5 & 79.9 & 94.6 & 68.7 & 91.9 & 73.8 & 74.3 & 18.3 & 15.4 & 63.9\\
\textsc{Consensus TA}  & 41.7&58.8&81.8&41.5&78.1&29.8&72.6&17.4&18.5&54.1\\
\textsc{C. LW AdaMerging} &43.2 & 83.7 & 96.8 & 67.0 & 89.9 & 81.6 & 63.7 & 16.8 & 10.7 & 59.1 \\
\textsc{C. LoRA-WEMOE} &44.6 & 72.5 & 86.1 & 40.1 & 63.8 & 63.8 & 48.1 & 10.3 & 12.8 & 55.7\\
\textsc{OPCM} & 58.5 & 82.9 & 95.9 & 67.6 & 92.8 & 74.0 & 76.3 & 22.4 & 18.3 & 64.6 \\
\textsc{Mingle (Ours)} & 65.0 & 85.5 & 97.0 & 72.6 & 94.1 & 81.5 & 85.4 & 50.4 & 65.2 & 67.1 \\
\textsc{Mingle$^*$ (Ours)} & 65.3 & 88.5 & 97.7 & 73.9 & 94.7 & 83.7 & 86.4 & 39.3 & 56.1 & 68.7 \\
\midrule
\multicolumn{11}{l}{\textbf{ViT-B/16}} \\
\textsc{C. Fine-Tuned} & 60.5 & 84.5 & 90.5 & 38.8 & 73.6 & 61.9 & 89.7 & 83.3 & 51.5 & 72.8 \\
\textsc{Average (SWA)} & 50.9 & 89.6 & 98.0 & 72.9 & 94.2 & 85.9 & 73.3 & 15.6 & 12.4 & 62.5 \\
\textsc{C.TA} & 50.7 & 89.3 & 97.0 & 68.0 & 93.1 & 80.3 & 75.7 & 18.1 & 16.7 & 61.8 \\
\textsc{C.TIES} & 50.4 & 87.9 & 96.3 & 63.1 & 91.7 & 78.0 & 75.0 & 23.4 & 24.9 & 61.5 \\
\textsc{MagMax-Ind} &57.7 & 88.8 & 97.5 & 71.5 & 94.4 & 81.3 & 77.2 & 24.5 & 25.0 & 59.4\\
\textsc{Consensus TA} & 45.6&76.8&87.7&44.4&82.2&38.4&72.7&18.8&30.0&58.6\\
\textsc{C. LW AdaMerging} &46.8 & 88.9 & 98.1 & 69.2 & 91.4 & 86.6 & 67.2 & 17.2 & 11.0 & 59.2\\
\textsc{C. LoRA-WEMOE} &45.6 & 91.2 & 92.3 & 41.3 & 64.3 & 78.1 & 48.0 & 23.5 & 16.6 & 52.7 \\
\textsc{OPCM} & 59.5 & 91.8 & 97.7 & 73.2 & 94.7 & 83.1 & 81.3& 26.5 & 23.4 & 66.8 \\
\textsc{Mingle (Ours)} &67.6 & 92.7 & 97.4 & 74.0 & 95.3 & 87.7 & 87.4 & 73.5 & 79.9 & 74.0\\
\textsc{Mingle$^*$ (Ours)} & 67.9 & 93.5 & 98.4 & 77.7 & 96.4 & 89.7 & 87.8 & 56.6 & 64.5 & 75.3 \\
\midrule
\multicolumn{11}{l}{\textbf{ViT-L/14}} \\
\textsc{C. Fine-Tuned} & 68.0 & 92.1 & 94.5 & 60.5 & 85.7 & 74.8 & 93.1 & 89.0 & 59.2 & 78.8 \\
\textsc{Average (SWA)} & 52.7 & 94.2 & 99.2 & 81.7 & 97.0 & 90.7 & 77.4 & 16.1 & 10.4 & 66.1 \\
\textsc{C.TA} & 55.7 & 94.2 & 98.6 & 79.1 & 96.6 & 87.6 & 80.8 & 17.6 & 10.6 & 63.6 \\
\textsc{C.TIES} & 57.6 & 93.5 & 97.8 & 74.0 & 95.6 & 84.7 & 79.7 & 20.2 & 12.6 & 58.4 \\
\textsc{MagMax-Ind} & 52.9 & 93.9 & 98.7 & 82.1 & 97.3 & 89.5 & 81.6 & 19.2 & 11.1 & 68.4\\
\textsc{Consensus TA} & 50.3 & 82.2 & 89.7 & 47.5 & 86.2 & 43.5 & 75.3 & 14.5 & 10.4 & 53.4\\
\textsc{C. LW AdaMerging} &49.2 & 93.5 & 99.3 & 77.2 & 95.8 & 91.1 & 68.2 & 18.6 & 9.8 & 66.6 \\
\textsc{C. LoRA-WEMOE} &46.3 & 84.5 & 87.6 & 52.1 & 70.5 & 73.3 & 50.0 & 18.7 & 10.9 & 56.5\\
\textsc{OPCM} & 61.8 & 95.4 & 99.2 & 83.0 & 97.8 & 90.9 & 86.0 & 26.4 & 14.7 & 71.0 \\
\textsc{Mingle (Ours)} & 67.7 & 96.0 & 98.7 & 81.4 & 97.1 & 90.6 & 90.6 & 60.7 & 88.6 & 79.8 \\
\textsc{Mingle$^*$ (Ours)} &67.9 & 96.0 & 99.4 & 84.7 & 97.8 & 92.4 & 88.8 & 53.0 & 57.1 & 75.5 \\
\bottomrule
\end{tabular}
}
\label{tab:continual_results}
\end{table}

\subsection{Accuracy Trends Across Sequential Tasks}
\label{appdix:Accuracy Across Sequential Task}
Fig.~\ref{fig:clip_task_accuracy_hierarchy_final_tuned} provides a detailed view of accuracy throughout the continual merging process across different settings, showing both the performance on the current task and on previously encountered ones. The progressive accuracy drop across columns illustrates the degree of forgetting over time. Notably, \methodshort{} consistently alleviates this degradation, demonstrating markedly reduced forgetting across the full task sequence.
Fig.~\ref{fig:acc_curve_all} further compares the average accuracy curves of \methodshort{} and baseline methods on previously seen tasks after each new model is merged, using the CLIP ViT-B/16 backbone. Results are averaged over 10 random task orderings. \methodshort{} consistently achieves the highest performance throughout the merging process, with its accuracy curve clearly dominating those of competing methods. Moreover, the narrower standard deviation bands indicate that \methodshort{} is more robust to the task orders.

\subsection{Detailed Results Under Distribution Shifts}
\label{appdix:Detail Analysis of Distribution Shift}

Tab.~\ref{tab:robustness_vitbase32} expands on Tab.~\ref{tab:robust} by reporting per-dataset accuracy under both clean test conditions and seven common corruption types: motion blur, impulse noise, Gaussian noise, pixelation, spatter, contrast shift, and JPEG compression. We evaluate six merging methods, across four downstream tasks: Cars, EuroSAT, RESISC45, and GTSRB. This detailed breakdown complements the average results in the main paper, providing a more comprehensive assessment of robustness under test-time distribution shifts.

\begin{figure}[!htbp]
  \centering

  \begin{subfigure}[b]{0.28\textwidth}
    \centering
    \includegraphics[width=\linewidth]{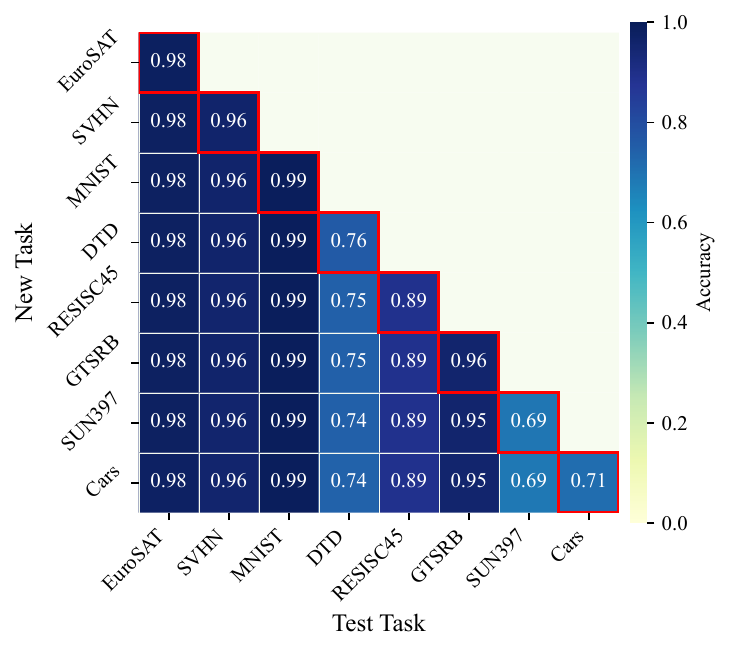}
    \caption{ViT-B/32 on 8 Tasks}
  \end{subfigure}
  \hfill
  \begin{subfigure}[b]{0.28\textwidth}
    \centering
    \includegraphics[width=\linewidth]{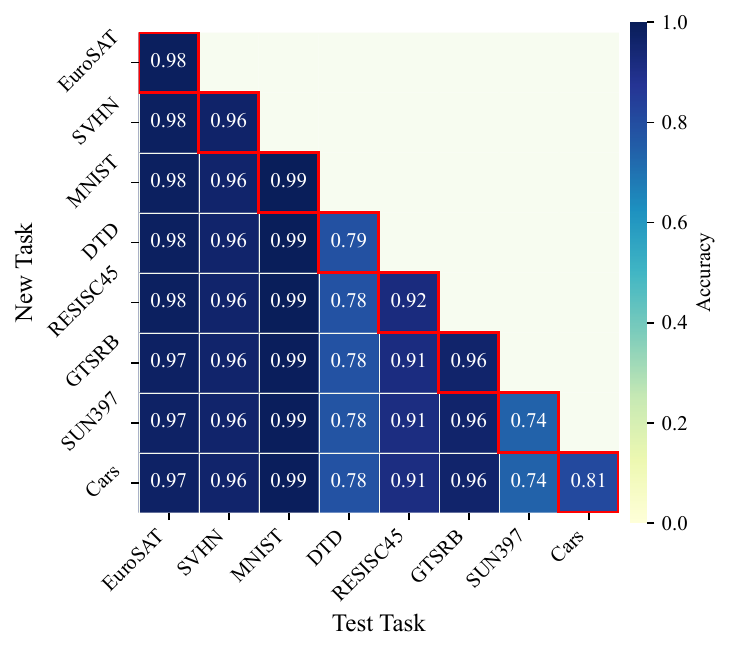}
    \caption{ViT-B/16 on 8 Tasks}
  \end{subfigure}
  \hfill
  \begin{subfigure}[b]{0.28\textwidth}
    \centering
    \includegraphics[width=\linewidth]{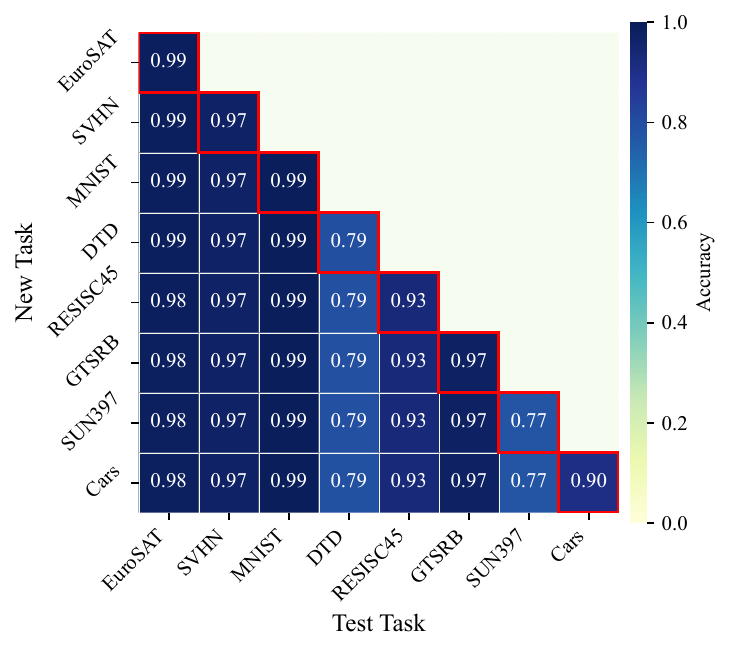}
    \caption{ViT-L/14 on 8 Tasks}
  \end{subfigure}

  \vskip1.0\baselineskip

  \begin{subfigure}[b]{0.45\textwidth}
    \centering
    \includegraphics[width=\linewidth]{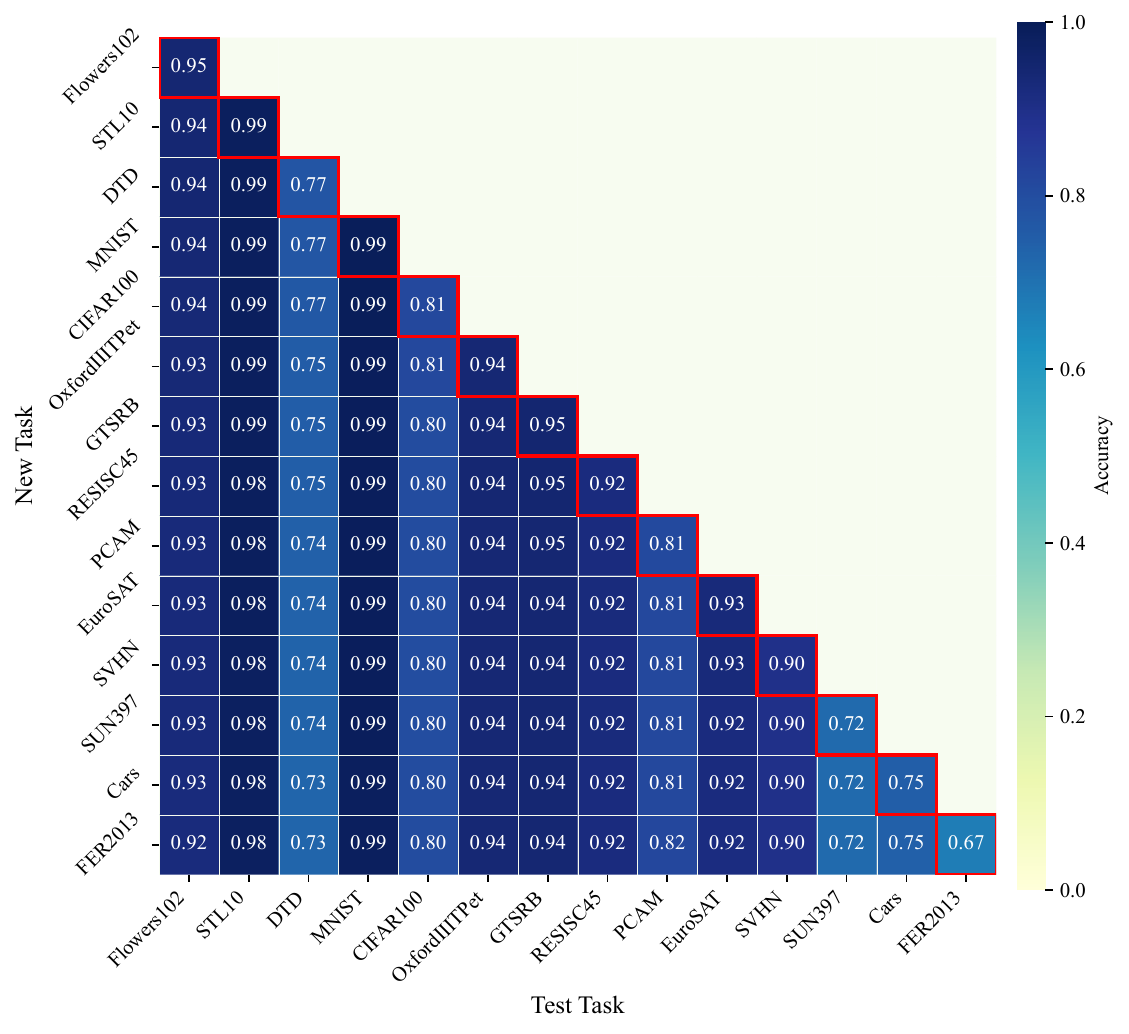}
    \caption{ViT-B/16 on 14 Tasks}
  \end{subfigure}
  \hfill
  \begin{subfigure}[b]{0.45\textwidth}
    \centering
    \includegraphics[width=\linewidth]{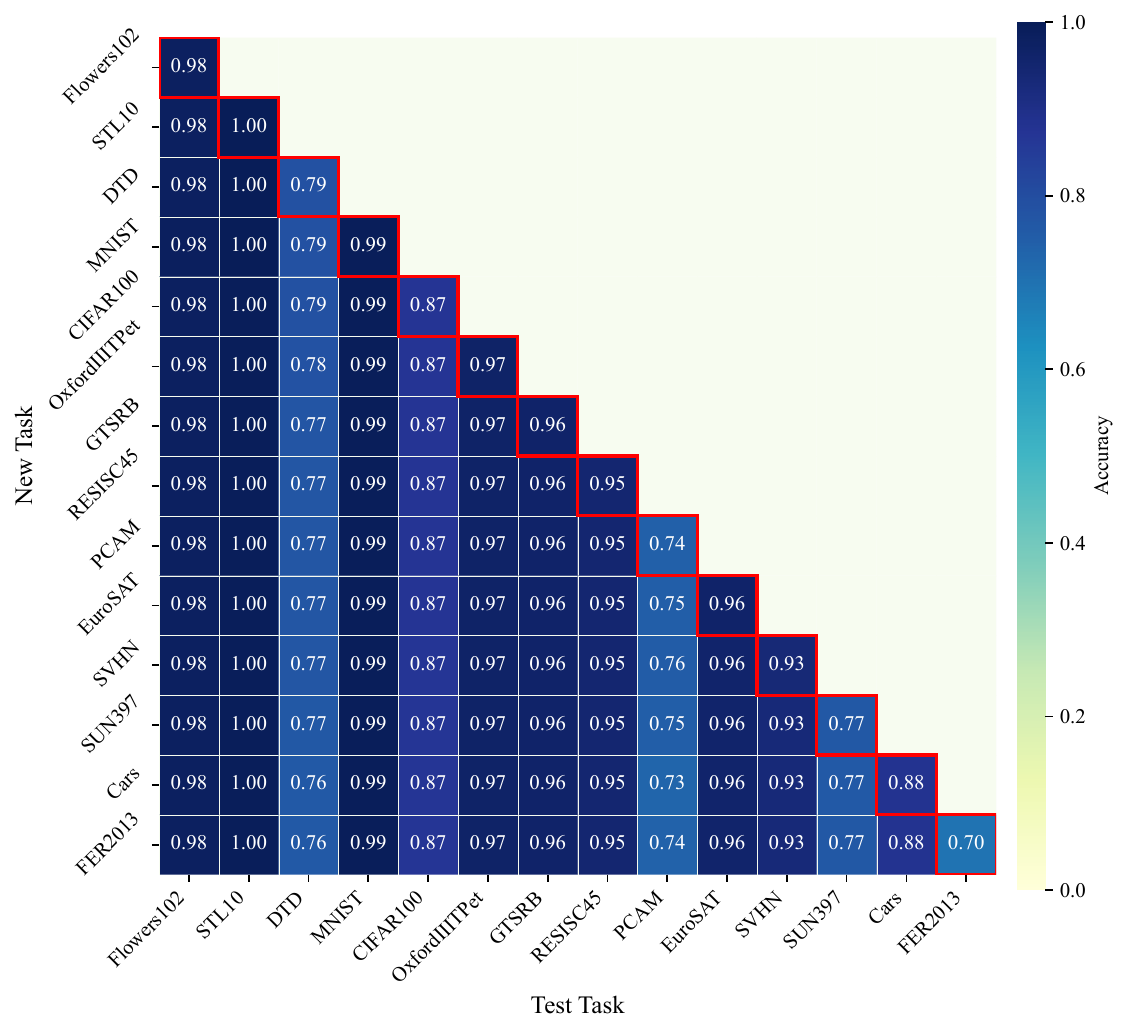}
    \caption{ViT-L/14 on 14 Tasks}
  \end{subfigure}

  \vskip1.2\baselineskip

  \begin{subfigure}[b]{0.7\textwidth}
    \centering
    \includegraphics[width=\linewidth]{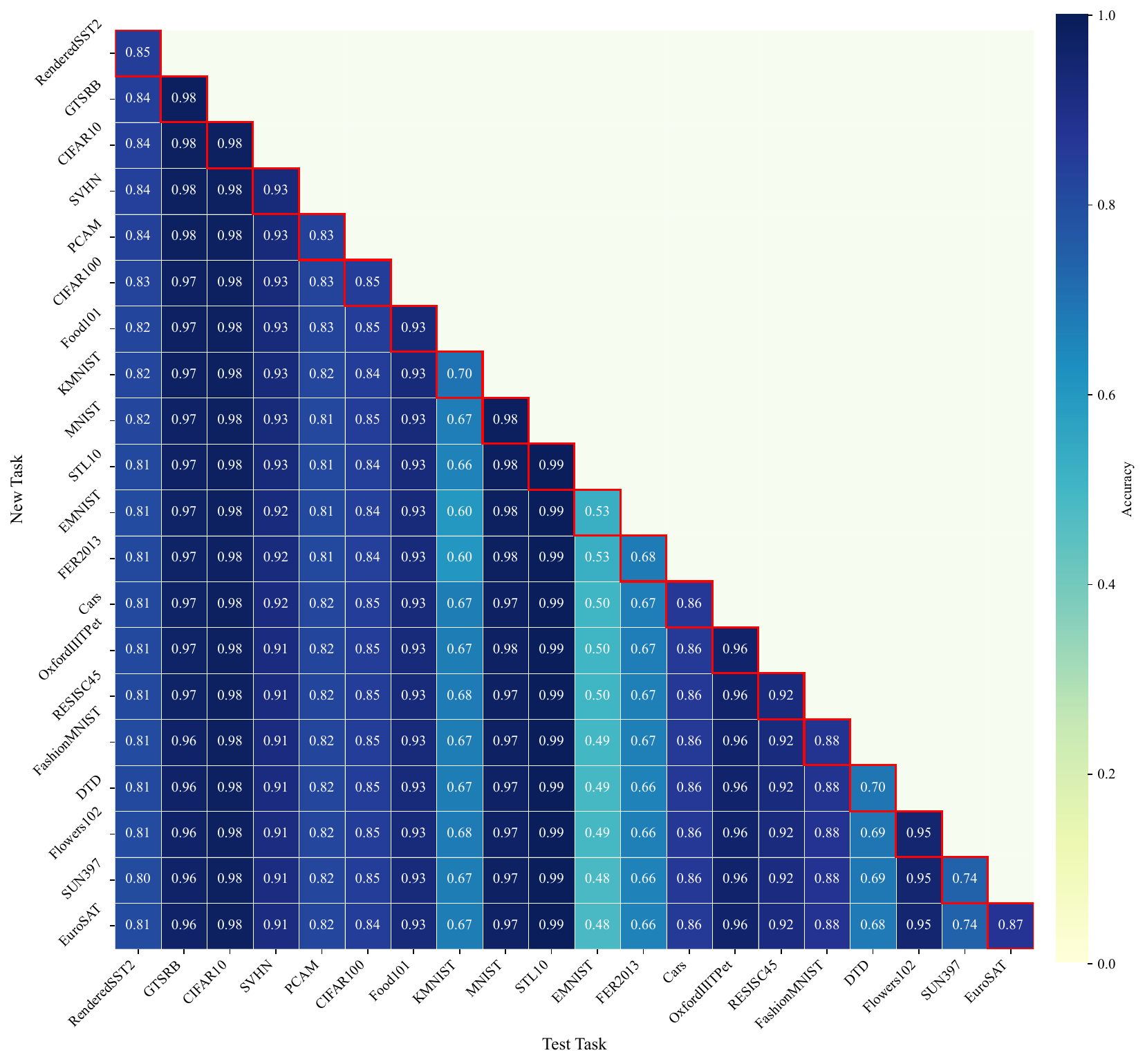}
    \caption{ViT-L/14 on 20 Tasks}
  \end{subfigure}

  \caption{Accuracy matrices of \methodshort{} (ViT-B/32, ViT-B/16, and ViT-L/14) under different task settings.}
  \label{fig:clip_task_accuracy_hierarchy_final_tuned}
\end{figure}

\begin{figure}[H]
    \centering
    \begin{subfigure}[t]{0.6\textwidth}
        \centering
        \includegraphics[width=\linewidth]{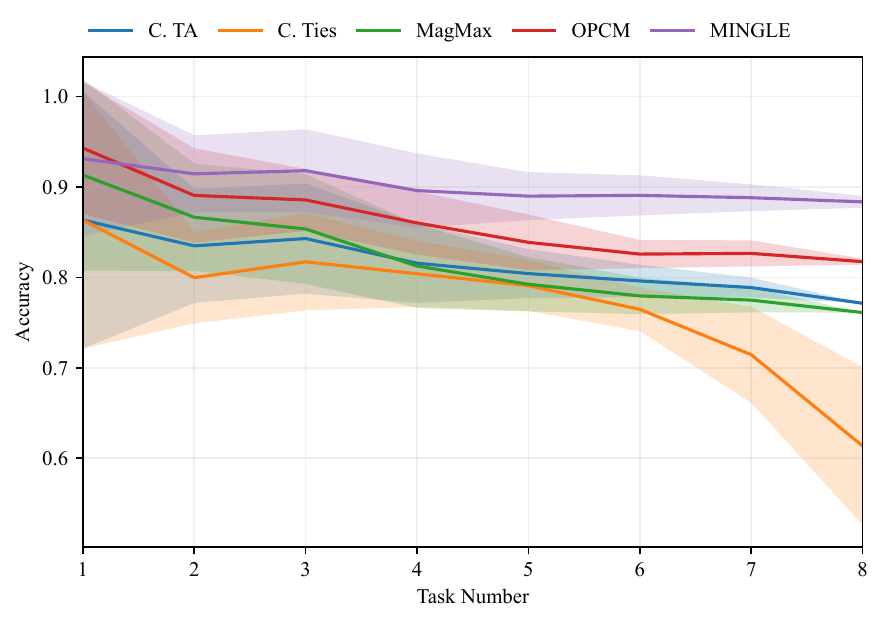}
        \caption{8 tasks}
        \label{fig:curve_t8}
    \end{subfigure}
    \hfill
    \begin{subfigure}[t]{0.6\textwidth}
        \centering
        \includegraphics[width=\linewidth]{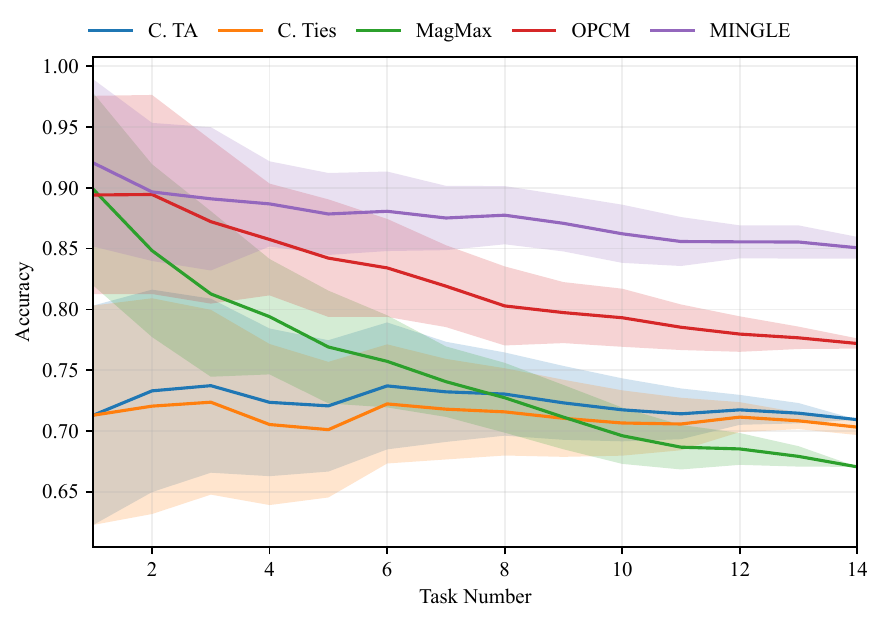}
        \caption{14 tasks}
        \label{fig:curve_t14}
    \end{subfigure}
    \hfill
    \begin{subfigure}[t]{0.6\textwidth}
        \centering
        \includegraphics[width=\linewidth]{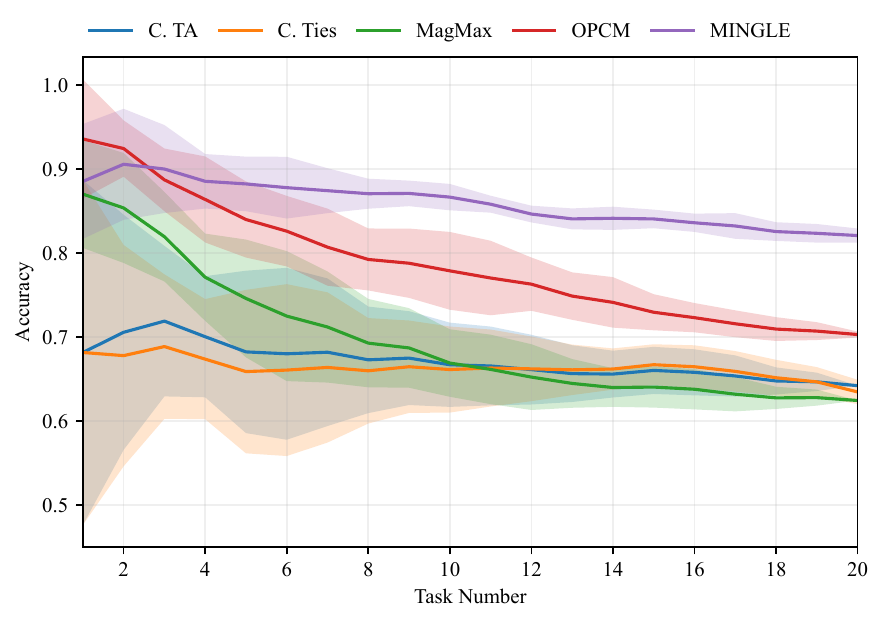}
        \caption{20 tasks}
        \label{fig:curve_t20}
    \end{subfigure}
    \caption{Sequential test accuracy curves of \methodshort{} and baselines (C.TA, C.Ties, MagMax, OPCM) under different task settings. Shaded regions indicate standard deviation across 10 task orders.}
    \label{fig:acc_curve_all}
\end{figure}

\begin{table}[H]
\centering
\setlength{\tabcolsep}{4pt}
\caption{Robustness results when merging ViT-B/32 models on four tasks.}
\label{tab:robustness_vitbase32}
\renewcommand\arraystretch{1.1}
\resizebox{\linewidth}{!}{
\begin{tabular}{l|ccccc|ccccc}
\toprule
& \multicolumn{5}{c|}{\textbf{Clean Test Set}} & \multicolumn{5}{c}{\textbf{Corruption: Motion Blur}} \\
Method & Cars & EuroSAT & RESISC45 & GTSRB & \textbf{Avg ACC} & Cars & EuroSAT & RESISC45 & GTSRB & \textbf{Avg ACC} \\
\midrule
\textsc{C. LW AdaMerging} &65.3 & 49.7 & 65.4 & 43.8 & 56.0   & 64.2 & 25.6 & 62.5 & 37.5 & 47.5  \\
\textsc{C. WEMOE} &0.5  &  8.1 &  2.6 &  2.5 &  3.4 & 0.5 & 8.0 & 1.8 & 2.3 & 3.1  \\
\textsc{C. LoRA-WEMOE} &66.1  &  84.3 &  81.0 &  83.6 & 78.7  & 64.8 & 57.9 & 82.0 & 79.2 & 71.0 \\
\textsc{C. Task Arithmetic}  &  64.6 & 90.4 & 80.2 & 74.8 & 77.5   & 62.3 & 59.4 & 78.5 & 63.3 & 65.9  \\
\textsc{MagMax-Ind} & 63.1 & 89.2 & 81.7 & 82.5 & 79.1 & 61.4 & 62.1 & 80.0 & 72.6 & 69.0  
  \\
\textsc{OPCM}    & 65.7 & 92.3 & 85.7 & 90.5 & 83.6  & 62.8 & 62.5 & 83.7 & 82.2 & 72.8 \\
\methodshort{} (Ours)  & 74.4 & 96.5 & 91.5 & 97.3 & 89.9 &  73.2 & 70.5 & 91.9 & 95.8 & 82.9 \\
\midrule
& \multicolumn{5}{c|}{\textbf{Corruption: Impulse Noise}} & \multicolumn{5}{c}{\textbf{Corruption: Gaussian Noise}} \\
Method & Cars & EuroSAT & RESISC45 & GTSRB & \textbf{Avg ACC} & Cars & EuroSAT & RESISC45 & GTSRB & \textbf{Avg ACC} \\
\midrule
\textsc{C. LW AdaMerging} &60.5 & 30.1 & 56.3 & 25.5 & 43.1& 62.3 & 25.6 & 59.7 & 25.6 & 43.3 \\
\textsc{C. WEMOE} & 0.5 & 11.2 & 2.3 & 3.2 & 4.3 & 0.5 & 8.1 & 2.4 & 2.8 & 3.4 \\
\textsc{C. LoRA-WEMOE} & 62.2 & 23.4 & 69.9 & 64.6 & 55.0 & 64.9 & 31.7 & 77.8 & 63.4 & 59.4 \\
\textsc{C. Task Arithmetic} & 59.9 & 57.7 & 72.9 & 45.0 & 58.9 & 61.8 & 51.4 & 75.1 & 50.1 & 59.6 \\
\textsc{MagMax-Ind} & 59.2 & 56.3 & 74.3 & 52.5 & 60.6 &60.6 & 51.7 & 77.0 & 56.5 & 61.5 \\
\textsc{OPCM}    &61.1 & 57.1 & 78.5 & 62.0 & 64.7 & 63.0 & 52.4 & 80.7 & 64.9 & 65.2 \\
\methodshort{} (Ours)  & 69.6 & 28.0 & 86.1 & 86.1 & 67.5 & 72.0 & 38.5 & 89.4 & 82.9 & 70.7 \\
\midrule
& \multicolumn{5}{c|}{\textbf{Corruption: Pixelate}} & \multicolumn{5}{c}{\textbf{Corruption: Spatter}} \\
Method & Cars & EuroSAT & RESISC45 & GTSRB & \textbf{Avg ACC} & Cars & EuroSAT & RESISC45 & GTSRB & \textbf{Avg ACC} \\
\midrule
\textsc{C. LW AdaMerging} & 3.4 & 16.5 & 13.5 & 39.2 & 18.1 & 61.3 & 34.1 & 58.2 & 32.8 & 46.6 \\
\textsc{C. WEMOE} & 0.5 & 6.3 & 2.5 & 2.5 & 3.0 & 0.5  & 10.1 &  2.7 &  2.6 &  4.0 \\
\textsc{C. LoRA-WEMOE} & 0.8 & 26.0 & 5.8 & 67.0 & 24.9& 62.4 & 35.4 & 71.2 & 73.0 & 60.5 \\
\textsc{C. Task Arithmetic}   &  2.5 & 31.7 & 19.1 & 65.6 & 29.7   &  61.2 & 63.1 & 72.7 & 57.0 & 63.5 \\
\textsc{MagMax-Ind} & 2.6 & 36.1 & 19.3 & 74.0 & 33.0 & 60.0 & 64.9 & 74.8 & 66.1 & 66.4 \\
\textsc{OPCM}    & 2.1 & 34.3 & 19.5 & 84.9 & 35.2  &  61.5 & 64.7 & 78.8 & 76.8 & 70.5 \\
\methodshort{} (Ours)  &2.3 & 35.6 & 18.5 & 95.1 & 37.9 & 70.1 & 57.8 & 86.2 & 93.9 & 77.0 \\
\midrule
& \multicolumn{5}{c|}{\textbf{Corruption: Contrast}} & \multicolumn{5}{c}{\textbf{Corruption: JPEG Compression}} \\
Method & Cars & EuroSAT & RESISC45 & GTSRB & \textbf{Avg ACC} & Cars & EuroSAT & RESISC45 & GTSRB & \textbf{Avg ACC} \\
\midrule
\textsc{C. LW AdaMerging} & 61.8 & 26.0 & 63.1 & 44.8 & 48.9  & 65.1 & 29.6 & 65.4 & 36.4 & 49.1\\
\textsc{C. WEMOE} & 0.5 & 7.5 & 2.3 & 3.0 & 3.3 & 0.5 & 10.5 & 2.4 & 2.7 & 4.0\\
\textsc{C. LoRA-WEMOE} & 64.3 & 46.5 & 77.7 & 85.6 & 68.5 &65.5 & 59.1 & 80.4 & 74.0 & 69.7\\
\textsc{C. Task Arithmetic} &  62.5 & 55.2 & 75.3 & 70.8 & 66.0 & 64.1 & 66.2 & 80.0 & 61.0 & 67.8 \\
\textsc{MagMax-Ind} & 61.3 & 58.0 & 76.9 & 78.2 & 68.6 & 62.5 & 67.7 & 81.1 & 68.5 & 69.9\\
\textsc{OPCM}    & 63.8 & 57.5 & 81.3 & 87.4 & 72.5 & 65.0 & 68.0 & 85.4 & 79.3 & 74.4 \\
\methodshort{} (Ours)  &  72.4 & 60.1 & 90.4 & 97.3 & 80.1 & 73.7 & 73.5 & 92.0 & 92.4 & 82.9\\
\bottomrule
\end{tabular}
}
\end{table}

\subsection{Inference Efficiency and Parameter Overhead}
\label{appdix:Inference Efficiency and Parameter Overhead}
Tab.~\ref{tab:efficiency} compares the inference efficiency and parameter overhead of all baselines on the CLIP ViT-B/32 model after merging eight tasks. 
We report the total number of parameters, additional storage and inference parameters, throughput (images per second), and accuracy. 
The results show that most static merging methods (\textit{e.g.}, Task Arithmetic, Ties-Merging, MAGMAX-Ind, OPCM) incur no extra storage or inference overhead, but typically achieve limited accuracy. 
Consensus TA and WEMOE introduce significant storage overhead, while WEMOE also scales up inference parameters considerably. 
By contrast, \methodshort{} achieves a favorable trade-off: although it introduces additional parameters for LoRA experts and the router, the effective inference overhead remains small, and throughput is only marginally reduced compared to static baselines. 
This efficiency advantage comes while delivering substantially higher accuracy.

\begin{table}[htbp]
  \centering
  \caption{Comparison of inference efficiency and parameter overhead on CLIP ViT-B/32 model after eight tasks merging.}
  \label{tab:efficiency} 
    \renewcommand\tabcolsep{4pt}  
    \resizebox{\textwidth}{!}{
  \begin{tabular}{lccccc}
    \toprule
    Method          & Total Params (M) & Extra Storage (M) & Extra Inference (M) & Throughput (img/s) & ACC (\%) \\
    \midrule
    \textsc{Task Arithmetic} & 87.5   & 0.0   & 0.0   & $\sim$910 & 67.5 \\
    \textsc{Ties-Merging}    & 87.5   & 0.0   & 0.0   & $\sim$910 & 49.0 \\
    \textsc{MAGMAX-Ind}      & 87.5   & 0.0   & 0.0   & $\sim$910 & 70.7 \\
    \textsc{OPCM}            & 87.5   & 0.0   & 0.0   & $\sim$910 & 75.5 \\
    \textsc{Consensus TA}    & 87.5   & 87.5  & 0.0   & $\sim$910 & 69.0 \\
    \textsc{LW. AdaMerging}  & 87.5   & 0.0   & 0.0   & $\sim$910 & 52.9 \\
    \textsc{WEMOE}           & 540.9  & 453.4 & 0.07  & $\sim$858 & 4.9  \\
    \textsc{WEMOE-LoRA}      & 103.7  & 16.2  & 0.07  & $\sim$848 & 66.6 \\
    \textsc{MINGLE}          & 173.1  & 85.6  & 0.6   & $\sim$841 & 85.8 \\
    \textsc{MINGLE*}         & 113.7  & 26.2  & 0.3   & $\sim$862 & 85.0 \\
    \bottomrule
  \end{tabular}}
\end{table}

\begin{table}[htbp]
  \centering
  \caption{Forward transfer (FWT) results on 8-task continual merging with CLIP ViT-B/16.}
  \label{tab:fwt}
  \renewcommand\tabcolsep{30pt}  
  \resizebox{\textwidth}{!}{
  \begin{tabular}{lccc}
    \toprule
    Method            & ACC (\%)     & BWT (\%)     & FWT (\%)     \\
    \midrule
    \textsc{Task Arithmetic}   & 77.1 $\pm$ 0.0  & -4.2 $\pm$ 1.0  & -13.4 $\pm$ 0.0 \\
    \textsc{Ties-Merging}      & 66.8 $\pm$ 3.7  & -5.5 $\pm$ 0.4  & -30.7 $\pm$ 9.9 \\
    \textsc{OPCM}              & 81.8 $\pm$ 0.3  & -4.8 $\pm$ 0.7  & -9.0  $\pm$ 0.4 \\
    \textbf{\methodshort{} (Ours)} & 88.3 $\pm$ 0.6  & -0.4 $\pm$ 0.1  & -3.8  $\pm$ 0.8 \\
    \bottomrule
  \end{tabular}}
\end{table}

\subsection{Forward Transfer Analysis}
\label{appdix:Forward Transfer Analysis}
Forward transfer (FWT) is an important metric in continual learning, as it quantifies how effectively prior knowledge facilitates the learning of future tasks. 
We adopt the standard definition:
\begin{equation}
    \text{FWT} = \frac{1}{T-1} \sum_{t=2}^{T} \left[ a_t(\theta_t^{\text{merged}}) - \bar{a}_t \right],
\end{equation}
where $a_t(\theta_t^{\text{merged}})$ denotes the test accuracy on task $t$ using the merged model after task $t$, and $\bar{a}_t$ is the accuracy of the individually fine-tuned model for task $t$. 
Positive FWT indicates beneficial transfer, while negative values suggest interference.

Tab.~\ref{tab:fwt} reports the results for the 8-task continual merging setup on CLIP ViT-B/16. 
The results demonstrate that \methodshort{} achieves nearly zero forgetting (BWT $\approx 0$) while obtaining the highest forward transfer among all baselines, showing that our adaptive gating and merging strategy not only preserves past knowledge but also enhances feature utility for future tasks.

\subsection{Additional Visualizations of Gate Activations and the Relaxation Effect}
\label{appdix:Hyper-parameter Analysis of Gate}
We provide an extended ablation study on gate hyper-parameters, including visualizations of gate activations under 14-task (Fig.~\ref{fig:null_gate_14}) and 20-task (Fig.~\ref{fig:null_gate_20}) configurations, complementing the 8-task results presented in the main paper. The visualizations demonstrate that the null-space constraint remains effective as the number of tasks increases, consistently suppressing gate responses to inputs from previously seen tasks and thereby mitigating forgetting.

\newpage
\begin{figure}[H]
    \centering
    \includegraphics[width=1\linewidth]{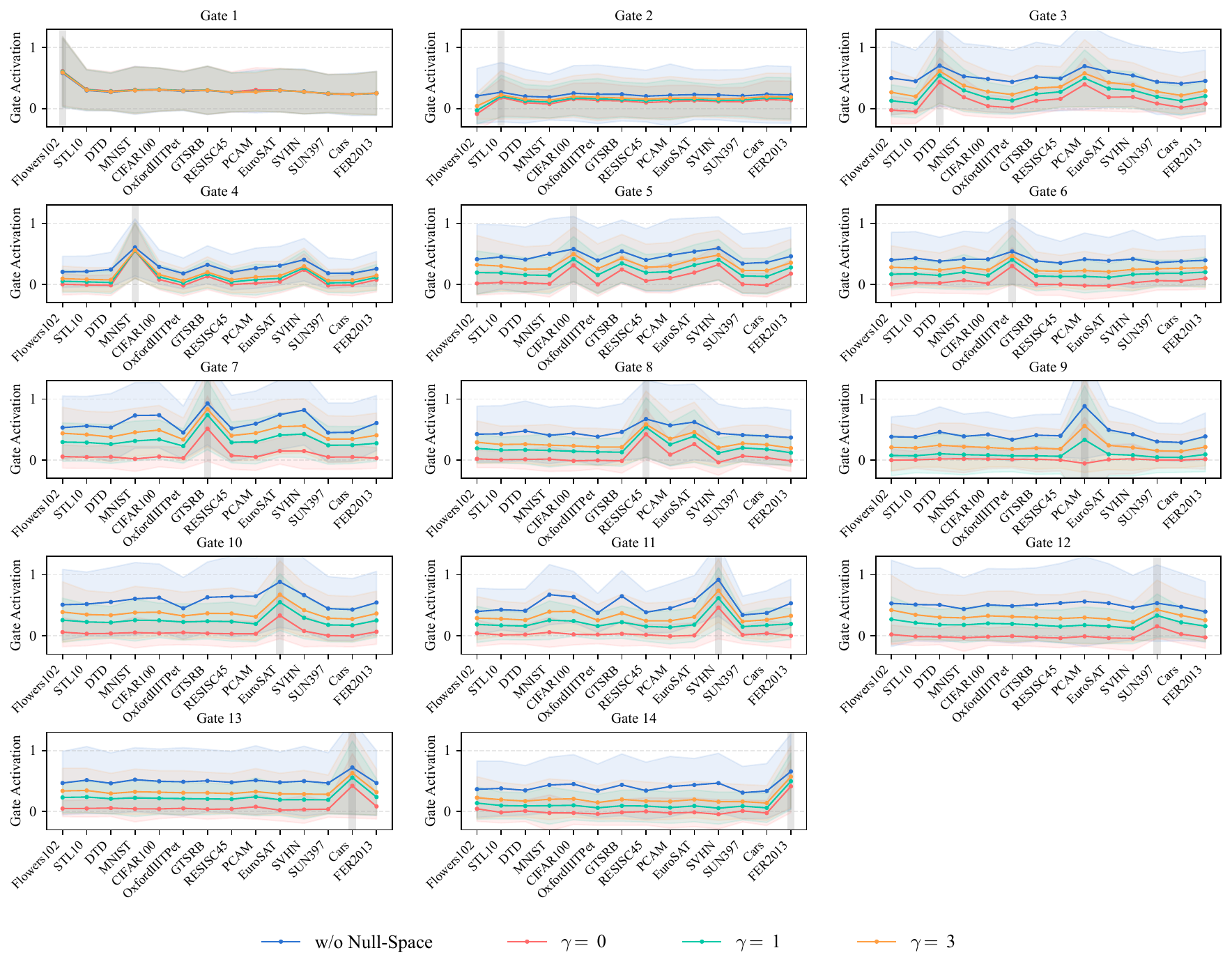}
    \caption{Visualization of gate activations across 14 tasks under varying $\gamma$ values. Each subplot corresponds to a gate, with curves and shaded regions denoting the mean and standard deviation of activations across layers. Gray bars mark the training dataset for each gate. Smaller $\gamma$ values result in stronger suppression of activations on previously learned tasks.}
    \label{fig:null_gate_14}
\end{figure}

\begin{figure}[H]
    \centering
    \includegraphics[width=1\linewidth]{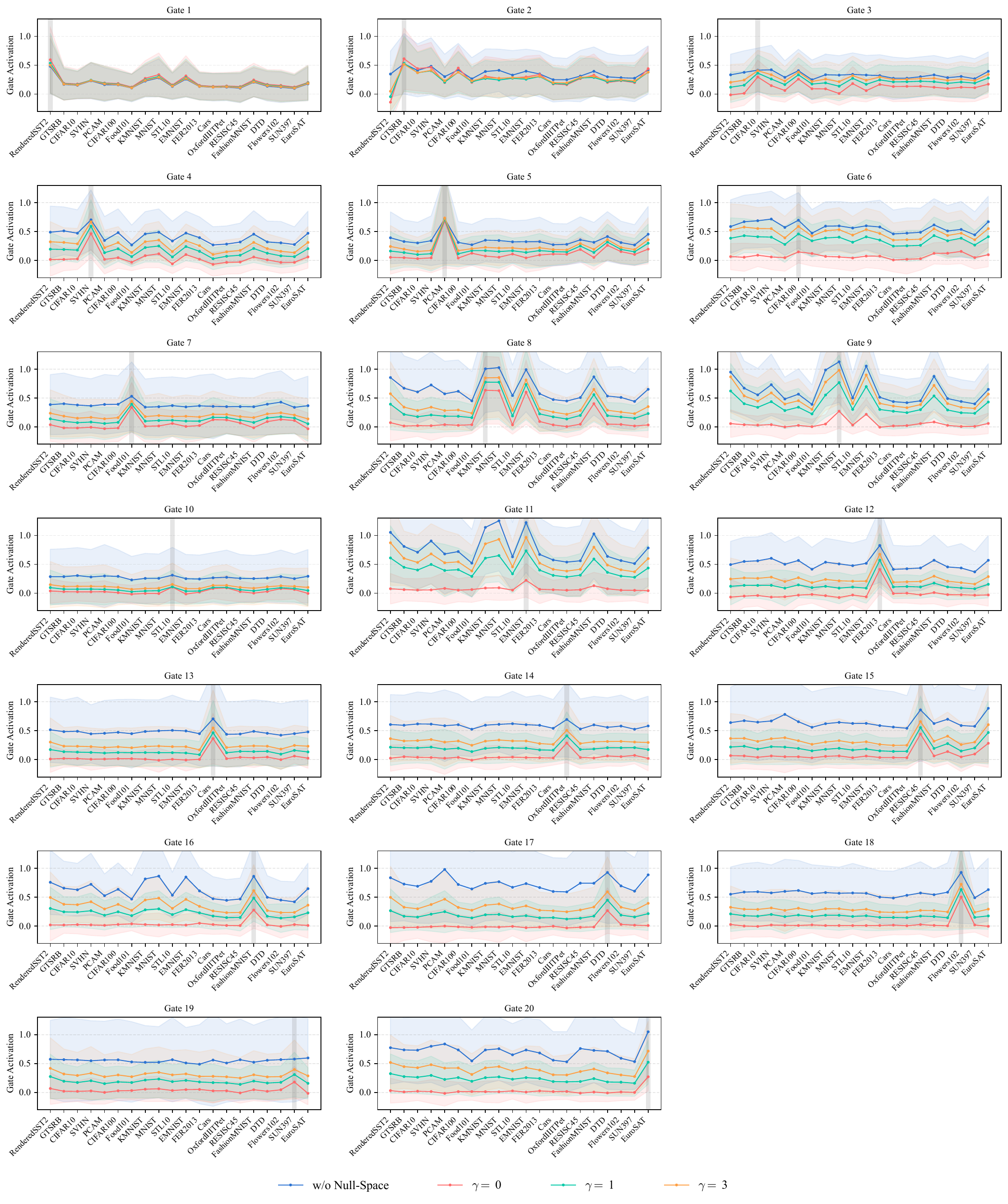}\vspace{-8pt}
    \caption{Visualization of gate activations across 20 tasks under varying $\gamma$ values.}\vspace{-5pt}
    \label{fig:null_gate_20}
\end{figure}

\newpage
\section{Discussions}
\label{appdix:Discussions}

\subsection{Use of Unlabeled Adaptation Samples}
\label{appdix:Use of Unlabeled Adaptation Samples}

In our experiments, we simulate a realistic deployment setting by randomly sampling 5 unlabeled examples per class from the test split, which serve as adaptation samples for model merging. 
Such small unlabeled buffers are practical in real-world applications and can be obtained from various sources, including (i) incoming test-time data such as recent user queries or model inputs, (ii) held-out validation inputs or small training subsets (if available), (iii) user-provided samples (\textit{e.g.}, few-shot examples) that do not raise privacy concerns, (iv) synthetically generated data, or (v) manually curated public data. 
Importantly, our method does not depend on precise sample selection, and the buffer size remains fixed and small, ensuring feasibility and robustness in deployment scenarios.

\subsection{Relation to Rehearsal-Free Continual Learning}
\label{appdix:Relation to Rehearsal-Free Continual Learning}

Test-time continual model merging (TTCMM) is closely related to the paradigm of rehearsal-free continual learning (RFCL), as both approaches share two fundamental constraints: 
(i) they do not retain past training data, and 
(ii) they avoid storing previous task models. 
The key distinction lies in the information available at each stage. 
RFCL assumes access to the training data of the current task and incrementally fine-tunes a single model over time. 
In contrast, TTCMM assumes access to independently fine-tuned models for each new task and focuses on merging these expert models rather than training them from scratch. 
Additionally, TTCMM relies on a small unlabeled buffer at test time (\textit{e.g.}, 5 samples per class) to guide the merging process.

From a privacy perspective, TTCMM provides stronger guarantees. 
Since it does not require access to full training sets, it only depends on a small set of unlabeled samples, which can be user-provided without risk, synthetically generated, or curated from public data. 
By comparison, RFCL requires access to large-scale labeled datasets for every task, raising more significant concerns regarding privacy, storage, and legal constraints (\textit{e.g.}, medical images, personal data, or copyrighted corpora). 
The reliance on a tiny unlabeled buffer makes TTCMM more practical in scenarios where data privacy is a primary consideration.

\subsection{Limitations}
\label{appdix:Limitations}
As with many model merging methods, our approach assumes that all independently fine-tuned models originate from a shared pretrained initialization. The extent to which this assumption influences merging performance remains unclear and warrants further investigation. In addition, our current experiments focus on merging models with identical backbone architectures (\textit{e.g.}, CLIP ViT-B/16). Although our use of LoRA-based expert offers some structural uniformity, which could potentially accommodate heterogeneous backbones, we have not yet explored this setting. Extending our framework to support diverse initialization points or architectural variants remains an open direction for future work.

\subsection{Broader Impacts}
\label{appdix:Broader Impacts}
This paper presents work whose goal is to advance the field of Machine Learning. There are many potential societal consequences of our work, none which we feel must be specifically highlighted here.

%% file: neurips_2025.bbl
\begin{thebibliography}{98}
\providecommand{\natexlab}[1]{#1}
\providecommand{\url}[1]{\texttt{#1}}
\expandafter\ifx\csname urlstyle\endcsname\relax
  \providecommand{\doi}[1]{doi: #1}\else
  \providecommand{\doi}{doi: \begingroup \urlstyle{rm}\Url}\fi

\bibitem[Ainsworth et~al.(2022)Ainsworth, Hayase, and Srinivasa]{ainsworth2022git}
S.~K. Ainsworth, J.~Hayase, and S.~Srinivasa.
\newblock Git re-basin: Merging models modulo permutation symmetries.
\newblock \emph{arXiv preprint arXiv:2209.04836}, 2022.

\bibitem[Aljundi et~al.(2018)Aljundi, Babiloni, Elhoseiny, Rohrbach, and Tuytelaars]{aljundi2018memory}
R.~Aljundi, F.~Babiloni, M.~Elhoseiny, M.~Rohrbach, and T.~Tuytelaars.
\newblock Memory aware synapses: Learning what (not) to forget.
\newblock In \emph{Proceedings of the European conference on computer vision (ECCV)}, pages 139--154, 2018.

\bibitem[Bossard et~al.(2014)Bossard, Guillaumin, and Van~Gool]{bossard2014food}
L.~Bossard, M.~Guillaumin, and L.~Van~Gool.
\newblock Food-101--mining discriminative components with random forests.
\newblock In \emph{European Conference on Computer Vision}, pages 446--461. Springer, 2014.

\bibitem[Cheng et~al.(2017)Cheng, Han, and Lu]{cheng2017remote}
G.~Cheng, J.~Han, and X.~Lu.
\newblock Remote sensing image scene classification: Benchmark and state of the art.
\newblock \emph{Proceedings of the IEEE}, 105\penalty0 (10):\penalty0 1865--1883, 2017.

\bibitem[Chitale et~al.(2023)Chitale, Vaidya, Kane, and Ghotkar]{chitale2023task}
R.~Chitale, A.~Vaidya, A.~Kane, and A.~S. Ghotkar.
\newblock Task arithmetic with lo{RA} for continual learning.
\newblock In \emph{Workshop on Advancing Neural Network Training: Computational Efficiency, Scalability, and Resource Optimization (WANT@NeurIPS 2023)}, 2023.

\bibitem[Chung et~al.(2024)Chung, Hou, Longpre, Zoph, Tay, Fedus, Li, Wang, Dehghani, Brahma, et~al.]{chung2024scaling}
H.~W. Chung, L.~Hou, S.~Longpre, B.~Zoph, Y.~Tay, W.~Fedus, Y.~Li, X.~Wang, M.~Dehghani, S.~Brahma, et~al.
\newblock Scaling instruction-finetuned language models.
\newblock \emph{Journal of Machine Learning Research}, 25\penalty0 (70):\penalty0 1--53, 2024.

\bibitem[Cimpoi et~al.(2014)Cimpoi, Maji, Kokkinos, Mohamed, and Vedaldi]{cimpoi2014describing}
M.~Cimpoi, S.~Maji, I.~Kokkinos, S.~Mohamed, and A.~Vedaldi.
\newblock Describing textures in the wild.
\newblock In \emph{Proceedings of the IEEE Conference on Computer Vision and Pattern Recognition}, pages 3606--3613, 2014.

\bibitem[Clanuwat et~al.(2018)Clanuwat, Bober-Irizar, Kitamoto, Lamb, Yamamoto, and Ha]{clanuwat2018deep}
T.~Clanuwat, M.~Bober-Irizar, A.~Kitamoto, A.~Lamb, K.~Yamamoto, and D.~Ha.
\newblock Deep learning for classical japanese literature.
\newblock In \emph{NeurIPS Workshop on Machine Learning for Creativity and Design}, 2018.

\bibitem[Coates et~al.(2011)Coates, Ng, and Lee]{coates2011analysis}
A.~Coates, A.~Y. Ng, and H.~Lee.
\newblock An analysis of single-layer networks in unsupervised feature learning.
\newblock In \emph{Proceedings of the Fourteenth International Conference on Artificial Intelligence and Statistics}, pages 215--223. JMLR Workshop and Conference Proceedings, 2011.

\bibitem[Cohen et~al.(2017)Cohen, Afshar, Tapson, and Van~Schaik]{cohen2017emnist}
G.~Cohen, S.~Afshar, J.~Tapson, and A.~Van~Schaik.
\newblock Emnist: Extending mnist to handwritten letters.
\newblock In \emph{2017 International Joint Conference on Neural Networks (IJCNN)}, pages 2921--2926. IEEE, 2017.

\bibitem[D{\"o}bler et~al.(2023)D{\"o}bler, Marsden, and Yang]{dobler2023robust}
M.~D{\"o}bler, R.~A. Marsden, and B.~Yang.
\newblock Robust mean teacher for continual and gradual test-time adaptation.
\newblock In \emph{Proceedings of the IEEE/CVF Conference on Computer Vision and Pattern Recognition}, pages 7704--7714, 2023.

\bibitem[Douillard et~al.(2020)Douillard, Cord, Ollion, Robert, and Valle]{Douillard2020PODNetPO}
A.~Douillard, M.~Cord, C.~Ollion, T.~Robert, and E.~Valle.
\newblock Podnet: Pooled outputs distillation for small-tasks incremental learning.
\newblock \emph{European Conference on Computer Vision}, 2020.

\bibitem[Entezari et~al.(2021)Entezari, Sedghi, Saukh, and Neyshabur]{entezari2021role}
R.~Entezari, H.~Sedghi, O.~Saukh, and B.~Neyshabur.
\newblock The role of permutation invariance in linear mode connectivity of neural networks.
\newblock \emph{arXiv preprint arXiv:2110.06296}, 2021.

\bibitem[Fang et~al.(2024)Fang, Dziedzic, Zhang, Oliva, Verma, Razak, Papernot, and Wang]{fang2024decentralised}
C.~Fang, A.~Dziedzic, L.~Zhang, L.~Oliva, A.~Verma, F.~Razak, N.~Papernot, and B.~Wang.
\newblock Decentralised, collaborative, and privacy-preserving machine learning for multi-hospital data.
\newblock \emph{EBioMedicine}, 101, 2024.

\bibitem[Feng et~al.(2023)Feng, Yu, Liu, Khan, and Zuo]{feng2023diverse}
C.-M. Feng, K.~Yu, Y.~Liu, S.~Khan, and W.~Zuo.
\newblock Diverse data augmentation with diffusions for effective test-time prompt tuning.
\newblock In \emph{Proceedings of the IEEE/CVF International Conference on Computer Vision}, pages 2704--2714, 2023.

\bibitem[Fukuda et~al.(2024)Fukuda, Kera, and Kawamoto]{fukuda2024adapter}
T.~Fukuda, H.~Kera, and K.~Kawamoto.
\newblock Adapter merging with centroid prototype mapping for scalable class-incremental learning.
\newblock \emph{arXiv preprint arXiv:2412.18219}, 2024.

\bibitem[Gao et~al.(2023)Gao, Zhang, Liu, Darrell, Shelhamer, and Wang]{gao2023back}
J.~Gao, J.~Zhang, X.~Liu, T.~Darrell, E.~Shelhamer, and D.~Wang.
\newblock Back to the source: Diffusion-driven adaptation to test-time corruption.
\newblock In \emph{Proceedings of the IEEE/CVF Conference on Computer Vision and Pattern Recognition}, pages 11786--11796, 2023.

\bibitem[Goodfellow et~al.(2013)Goodfellow, Erhan, Carrier, Courville, Mirza, Hamner, Cukierski, Tang, Thaler, Lee, et~al.]{goodfellow2013challenges}
I.~J. Goodfellow, D.~Erhan, P.-L. Carrier, A.~Courville, M.~Mirza, B.~Hamner, W.~Cukierski, Y.~Tang, D.~Thaler, D.-H. Lee, et~al.
\newblock Challenges in representation learning: A report on three machine learning contests.
\newblock In \emph{International Conference on Neural Information Processing}, pages 117--124. Springer, 2013.

\bibitem[Helber et~al.(2019)Helber, Bischke, Dengel, and Borth]{helber2019eurosat}
P.~Helber, B.~Bischke, A.~Dengel, and D.~Borth.
\newblock Eurosat: A novel dataset and deep learning benchmark for land use and land cover classification.
\newblock \emph{IEEE Journal of Selected Topics in Applied Earth Observations and Remote Sensing}, 12\penalty0 (7):\penalty0 2217--2226, 2019.

\bibitem[Hou et~al.(2019)Hou, Pan, Loy, Wang, and Lin]{Hou2019LearningAU}
S.~Hou, X.~Pan, C.~C. Loy, Z.~Wang, and D.~Lin.
\newblock Learning a unified classifier incrementally via rebalancing.
\newblock \emph{2019 IEEE/CVF Conference on Computer Vision and Pattern Recognition (CVPR)}, pages 831--839, 2019.

\bibitem[Hu et~al.(2022)Hu, Shen, Wallis, Allen-Zhu, Li, Wang, Wang, Chen, et~al.]{hu2022lora}
E.~J. Hu, Y.~Shen, P.~Wallis, Z.~Allen-Zhu, Y.~Li, S.~Wang, L.~Wang, W.~Chen, et~al.
\newblock Lora: Low-rank adaptation of large language models.
\newblock \emph{ICLR}, 1\penalty0 (2):\penalty0 3, 2022.

\bibitem[Huang et~al.(2024{\natexlab{a}})Huang, Liu, Lin, Pang, Du, and Lin]{huang2024lorahub}
C.~Huang, Q.~Liu, B.~Y. Lin, T.~Pang, C.~Du, and M.~Lin.
\newblock Lorahub: Efficient cross-task generalization via dynamic lo{RA} composition.
\newblock In \emph{First Conference on Language Modeling}, 2024{\natexlab{a}}.

\bibitem[Huang et~al.(2024{\natexlab{b}})Huang, Cao, Lu, and Liu]{huang2024class}
L.~Huang, X.~Cao, H.~Lu, and X.~Liu.
\newblock Class-incremental learning with clip: Adaptive representation adjustment and parameter fusion.
\newblock In \emph{European Conference on Computer Vision}, pages 214--231. Springer, 2024{\natexlab{b}}.

\bibitem[Ilharco et~al.(2023)Ilharco, Ribeiro, Wortsman, Schmidt, Hajishirzi, and Farhadi]{ilharco2023editing}
G.~Ilharco, M.~T. Ribeiro, M.~Wortsman, L.~Schmidt, H.~Hajishirzi, and A.~Farhadi.
\newblock Editing models with task arithmetic.
\newblock In \emph{The Eleventh International Conference on Learning Representations}, 2023.

\bibitem[Iwasawa and Matsuo(2021)]{iwasawa2021test}
Y.~Iwasawa and Y.~Matsuo.
\newblock Test-time classifier adjustment module for model-agnostic domain generalization.
\newblock \emph{Advances in Neural Information Processing Systems}, 34:\penalty0 2427--2440, 2021.

\bibitem[Izmailov et~al.(2018)Izmailov, Podoprikhin, Garipov, Vetrov, and Wilson]{izmailov2018averaging}
P.~Izmailov, D.~Podoprikhin, T.~Garipov, D.~Vetrov, and A.~G. Wilson.
\newblock Averaging weights leads to wider optima and better generalization.
\newblock In \emph{34th Conference on Uncertainty in Artificial Intelligence 2018, UAI 2018}, pages 876--885. Association For Uncertainty in Artificial Intelligence (AUAI), 2018.

\bibitem[Jacobs et~al.(1991)Jacobs, Jordan, Nowlan, and Hinton]{jacobs1991adaptive}
R.~A. Jacobs, M.~I. Jordan, S.~J. Nowlan, and G.~E. Hinton.
\newblock Adaptive mixtures of local experts.
\newblock \emph{Neural computation}, 3\penalty0 (1):\penalty0 79--87, 1991.

\bibitem[Jin et~al.(2023)Jin, Ren, Preotiuc-Pietro, and Cheng]{jin2023dataless}
X.~Jin, X.~Ren, D.~Preotiuc-Pietro, and P.~Cheng.
\newblock Dataless knowledge fusion by merging weights of language models.
\newblock In \emph{The Eleventh International Conference on Learning Representations}, 2023.

\bibitem[Jordan and Jacobs(1994)]{jordan1994hierarchical}
M.~I. Jordan and R.~A. Jacobs.
\newblock Hierarchical mixtures of experts and the em algorithm.
\newblock \emph{Neural computation}, 6\penalty0 (2):\penalty0 181--214, 1994.

\bibitem[Jung et~al.(2020)Jung, Ahn, Cha, and Moon]{jung2020continual}
S.~Jung, H.~Ahn, S.~Cha, and T.~Moon.
\newblock Continual learning with node-importance based adaptive group sparse regularization.
\newblock \emph{Advances in neural information processing systems}, 33:\penalty0 3647--3658, 2020.

\bibitem[Kirkpatrick et~al.(2016)Kirkpatrick, Pascanu, Rabinowitz, Veness, Desjardins, Rusu, Milan, Quan, Ramalho, Grabska-Barwinska, Hassabis, Clopath, Kumaran, and Hadsell]{Kirkpatrick2016OvercomingCF}
J.~Kirkpatrick, R.~Pascanu, N.~C. Rabinowitz, J.~Veness, G.~Desjardins, A.~A. Rusu, K.~Milan, J.~Quan, T.~Ramalho, A.~Grabska-Barwinska, D.~Hassabis, C.~Clopath, D.~Kumaran, and R.~Hadsell.
\newblock Overcoming catastrophic forgetting in neural networks.
\newblock \emph{Proceedings of the National Academy of Sciences}, 114:\penalty0 3521 -- 3526, 2016.

\bibitem[Krause et~al.(2013)Krause, Stark, Deng, and Fei-Fei]{krause20133d}
J.~Krause, M.~Stark, J.~Deng, and L.~Fei-Fei.
\newblock 3d object representations for fine-grained categorization.
\newblock In \emph{2013 IEEE International Conference on Computer Vision Workshops}, pages 554--561. IEEE, 2013.

\bibitem[Krizhevsky and Hinton(2009)]{krizhevsky2009learning}
A.~Krizhevsky and G.~Hinton.
\newblock Learning multiple layers of features from tiny images.
\newblock Technical report, University of Toronto, 2009.

\bibitem[LeCun et~al.(1998)LeCun, Bottou, Bengio, and Haffner]{lecun1998mnist}
Y.~LeCun, L.~Bottou, Y.~Bengio, and P.~Haffner.
\newblock Gradient-based learning applied to document recognition.
\newblock \emph{Proceedings of the IEEE}, 86\penalty0 (11):\penalty0 2278--2324, 1998.

\bibitem[Lee et~al.(2017)Lee, Kim, Jun, Ha, and Zhang]{lee2017overcoming}
S.-W. Lee, J.-H. Kim, J.~Jun, J.-W. Ha, and B.-T. Zhang.
\newblock Overcoming catastrophic forgetting by incremental moment matching.
\newblock \emph{Advances in neural information processing systems}, 30, 2017.

\bibitem[Lee et~al.(2025)Lee, Kim, Kang, Bang, Song, and Lee]{lee2025ratta}
Y.~Lee, D.~Kim, J.~Kang, J.~Bang, H.~Song, and J.-G. Lee.
\newblock {RA}-{TTA}: Retrieval-augmented test-time adaptation for vision-language models.
\newblock In \emph{The Thirteenth International Conference on Learning Representations}, 2025.

\bibitem[Lin et~al.(2022)Lin, Yang, Fan, and Zhang]{lin2022beyond}
S.~Lin, L.~Yang, D.~Fan, and J.~Zhang.
\newblock Beyond not-forgetting: Continual learning with backward knowledge transfer.
\newblock \emph{Advances in Neural Information Processing Systems}, 35:\penalty0 16165--16177, 2022.

\bibitem[Liu and Soatto(2023)]{liu2023tangent}
T.~Y. Liu and S.~Soatto.
\newblock Tangent model composition for ensembling and continual fine-tuning.
\newblock In \emph{Proceedings of the IEEE/CVF International Conference on Computer Vision}, pages 18676--18686, 2023.

\bibitem[Liu et~al.(2021)Liu, Schiele, and Sun]{liu2021rmm}
Y.~Liu, B.~Schiele, and Q.~Sun.
\newblock Rmm: Reinforced memory management for class-incremental learning.
\newblock \emph{Advances in Neural Information Processing Systems}, 34:\penalty0 3478--3490, 2021.

\bibitem[Lu et~al.(2024)Lu, Fan, Wei, Qu, Chen, and Cheng]{lu2024twin}
Z.~Lu, C.~Fan, W.~Wei, X.~Qu, D.~Chen, and Y.~Cheng.
\newblock Twin-merging: Dynamic integration of modular expertise in model merging.
\newblock \emph{Advances in Neural Information Processing Systems}, 37:\penalty0 78905--78935, 2024.

\bibitem[Marczak et~al.(2024)Marczak, Twardowski, Trzciński, and Cygert]{marczak2024magmax}
D.~Marczak, B.~Twardowski, T.~Trzciński, and S.~Cygert.
\newblock Magmax: Leveraging model merging for seamless continual learning.
\newblock In \emph{European Conference on Computer Vision (ECCV)}, 2024.

\bibitem[Marouf et~al.(2024)Marouf, Roy, Tartaglione, and Lathuilière]{marouf2024weightedensemblemodelsstrong}
I.~E. Marouf, S.~Roy, E.~Tartaglione, and S.~Lathuilière.
\newblock Weighted ensemble models are strong continual learners, 2024.

\bibitem[McCloskey and Cohen(1989)]{mccloskey1989catastrophic}
M.~McCloskey and N.~J. Cohen.
\newblock Catastrophic interference in connectionist networks: The sequential learning problem.
\newblock In \emph{Psychology of learning and motivation}, volume~24, pages 109--165. Elsevier, 1989.

\bibitem[McMahan et~al.(2017)McMahan, Moore, Ramage, Hampson, and y~Arcas]{mcmahan2017communication}
B.~McMahan, E.~Moore, D.~Ramage, S.~Hampson, and B.~A. y~Arcas.
\newblock Communication-efficient learning of deep networks from decentralized data.
\newblock In \emph{Artificial intelligence and statistics}, pages 1273--1282. PMLR, 2017.

\bibitem[Mu and Lin(2025)]{mu2025comprehensive}
S.~Mu and S.~Lin.
\newblock A comprehensive survey of mixture-of-experts: Algorithms, theory, and applications.
\newblock \emph{arXiv preprint arXiv:2503.07137}, 2025.

\bibitem[Netzer et~al.(2011)Netzer, Wang, Coates, Bissacco, Wu, and Ng]{netzer2011reading}
Y.~Netzer, T.~Wang, A.~Coates, A.~Bissacco, B.~Wu, and A.~Y. Ng.
\newblock Reading digits in natural images with unsupervised feature learning.
\newblock In \emph{NIPS Workshop on Deep Learning and Unsupervised Feature Learning}, 2011.

\bibitem[Nilsback and Zisserman(2008)]{nilsback2008automated}
M.-E. Nilsback and A.~Zisserman.
\newblock Automated flower classification over a large number of classes.
\newblock In \emph{2008 Sixth Indian Conference on Computer Vision, Graphics \& Image Processing}, pages 722--729. IEEE, 2008.

\bibitem[Niu et~al.(2022)Niu, Wu, Zhang, Chen, Zheng, Zhao, and Tan]{niu2022efficient}
S.~Niu, J.~Wu, Y.~Zhang, Y.~Chen, S.~Zheng, P.~Zhao, and M.~Tan.
\newblock Efficient test-time model adaptation without forgetting.
\newblock In \emph{International conference on machine learning}, pages 16888--16905. PMLR, 2022.

\bibitem[Ortiz-Jimenez et~al.(2023)Ortiz-Jimenez, Favero, and Frossard]{ortiz2023task}
G.~Ortiz-Jimenez, A.~Favero, and P.~Frossard.
\newblock Task arithmetic in the tangent space: Improved editing of pre-trained models.
\newblock \emph{Advances in Neural Information Processing Systems}, 36:\penalty0 66727--66754, 2023.

\bibitem[Parkhi et~al.(2012)Parkhi, Vedaldi, Zisserman, and Jawahar]{parkhi2012cats}
O.~M. Parkhi, A.~Vedaldi, A.~Zisserman, and C.~V. Jawahar.
\newblock Cats and dogs.
\newblock \emph{2012 IEEE Conference on Computer Vision and Pattern Recognition}, pages 3498--3505, 2012.

\bibitem[Porrello et~al.(2025)Porrello, Bonicelli, Buzzega, Millunzi, Calderara, and Cucchiara]{porrello2025a}
A.~Porrello, L.~Bonicelli, P.~Buzzega, M.~Millunzi, S.~Calderara, and R.~Cucchiara.
\newblock A second-order perspective on model compositionality and incremental learning.
\newblock In \emph{The Thirteenth International Conference on Learning Representations}, 2025.

\bibitem[Qiu et~al.(2023)Qiu, Xu, Wang, Wu, Meng, and Li]{qiu2023ism}
Z.~Qiu, L.~Xu, Z.~Wang, Q.~Wu, F.~Meng, and H.~Li.
\newblock Ism-net: Mining incremental semantics for class incremental learning.
\newblock \emph{Neurocomputing}, 523:\penalty0 130--143, 2023.

\bibitem[Qiu et~al.(2024)Qiu, Xu, Meng, Li, Xu, and Wu]{qiu2024dual}
Z.~Qiu, Y.~Xu, F.~Meng, H.~Li, L.~Xu, and Q.~Wu.
\newblock Dual-consistency model inversion for non-exemplar class incremental learning.
\newblock In \emph{Proceedings of the IEEE/CVF conference on computer vision and pattern recognition}, pages 24025--24035, 2024.

\bibitem[Radford et~al.(2021)Radford, Kim, Hallacy, Ramesh, Goh, Agarwal, Sastry, Askell, Mishkin, Clark, et~al.]{radford2021learning}
A.~Radford, J.~W. Kim, C.~Hallacy, A.~Ramesh, G.~Goh, S.~Agarwal, G.~Sastry, A.~Askell, P.~Mishkin, J.~Clark, et~al.
\newblock Learning transferable visual models from natural language supervision.
\newblock In \emph{International conference on machine learning}, pages 8748--8763. PmLR, 2021.

\bibitem[Rebuffi et~al.(2016)Rebuffi, Kolesnikov, Sperl, and Lampert]{Rebuffi2016iCaRLIC}
S.-A. Rebuffi, A.~Kolesnikov, G.~Sperl, and C.~H. Lampert.
\newblock icarl: Incremental classifier and representation learning.
\newblock \emph{2017 IEEE Conference on Computer Vision and Pattern Recognition (CVPR)}, pages 5533--5542, 2016.

\bibitem[Schneider et~al.(2020)Schneider, Rusak, Eck, Bringmann, Brendel, and Bethge]{schneider2020improving}
S.~Schneider, E.~Rusak, L.~Eck, O.~Bringmann, W.~Brendel, and M.~Bethge.
\newblock Improving robustness against common corruptions by covariate shift adaptation.
\newblock \emph{Advances in neural information processing systems}, 33:\penalty0 11539--11551, 2020.

\bibitem[Shen et~al.(2023)Shen, Song, Tan, Li, Lu, and Zhuang]{shen2023hugginggpt}
Y.~Shen, K.~Song, X.~Tan, D.~Li, W.~Lu, and Y.~Zhuang.
\newblock Hugginggpt: Solving ai tasks with chatgpt and its friends in hugging face.
\newblock \emph{Advances in Neural Information Processing Systems}, 36:\penalty0 38154--38180, 2023.

\bibitem[Shoemake(1985)]{shoemake1985animating}
K.~Shoemake.
\newblock Animating rotation with quaternion curves.
\newblock In \emph{Proceedings of the 12th annual conference on Computer graphics and interactive techniques}, pages 245--254, 1985.

\bibitem[Shu et~al.(2022)Shu, Nie, Huang, Yu, Goldstein, Anandkumar, and Xiao]{shu2022test}
M.~Shu, W.~Nie, D.-A. Huang, Z.~Yu, T.~Goldstein, A.~Anandkumar, and C.~Xiao.
\newblock Test-time prompt tuning for zero-shot generalization in vision-language models.
\newblock \emph{Advances in Neural Information Processing Systems}, 35:\penalty0 14274--14289, 2022.

\bibitem[Simon et~al.(2021)Simon, Koniusz, and Harandi]{simon2021learning}
C.~Simon, P.~Koniusz, and M.~Harandi.
\newblock On learning the geodesic path for incremental learning.
\newblock In \emph{Proceedings of the IEEE/CVF conference on Computer Vision and Pattern Recognition}, pages 1591--1600, 2021.

\bibitem[Smith et~al.(2023)Smith, Karlinsky, Gutta, Cascante-Bonilla, Kim, Arbelle, Panda, Feris, and Kira]{smith2023coda}
J.~S. Smith, L.~Karlinsky, V.~Gutta, P.~Cascante-Bonilla, D.~Kim, A.~Arbelle, R.~Panda, R.~Feris, and Z.~Kira.
\newblock Coda-prompt: Continual decomposed attention-based prompting for rehearsal-free continual learning.
\newblock In \emph{Proceedings of the IEEE/CVF Conference on Computer Vision and Pattern Recognition}, pages 11909--11919, 2023.

\bibitem[Socher et~al.(2013)Socher, Perelygin, Wu, Chuang, Manning, Ng, and Potts]{socher2013recursive}
R.~Socher, A.~Perelygin, J.~Wu, J.~Chuang, C.~D. Manning, A.~Ng, and C.~Potts.
\newblock Recursive deep models for semantic compositionality over a sentiment treebank.
\newblock In \emph{Proceedings of the 2013 Conference on Empirical Methods in Natural Language Processing}, pages 1631--1642, 2013.

\bibitem[Song et~al.(2023)Song, Lee, Kweon, and Choi]{song2023ecotta}
J.~Song, J.~Lee, I.~S. Kweon, and S.~Choi.
\newblock Ecotta: Memory-efficient continual test-time adaptation via self-distilled regularization.
\newblock In \emph{Proceedings of the IEEE/CVF Conference on Computer Vision and Pattern Recognition}, pages 11920--11929, 2023.

\bibitem[Stallkamp et~al.(2012)Stallkamp, Schlipsing, Salmen, and Igel]{stallkamp2012man}
J.~Stallkamp, M.~Schlipsing, J.~Salmen, and C.~Igel.
\newblock Man vs. computer: Benchmarking machine learning algorithms for traffic sign recognition.
\newblock \emph{Neural Networks}, 32:\penalty0 323--332, 2012.

\bibitem[Stoica et~al.(2025)Stoica, Ramesh, Ecsedi, Choshen, and Hoffman]{stoica2025model}
G.~Stoica, P.~Ramesh, B.~Ecsedi, L.~Choshen, and J.~Hoffman.
\newblock Model merging with {SVD} to tie the knots.
\newblock In \emph{The Thirteenth International Conference on Learning Representations}, 2025.

\bibitem[Sun et~al.(2020)Sun, Wang, Liu, Miller, Efros, and Hardt]{sun2020test}
Y.~Sun, X.~Wang, Z.~Liu, J.~Miller, A.~Efros, and M.~Hardt.
\newblock Test-time training with self-supervision for generalization under distribution shifts.
\newblock In \emph{International conference on machine learning}, pages 9229--9248. PMLR, 2020.

\bibitem[Tang et~al.(2024{\natexlab{a}})Tang, Shen, Luo, Hu, Du, and Tao]{tangFusionBenchComprehensiveBenchmark2024}
A.~Tang, L.~Shen, Y.~Luo, H.~Hu, B.~Du, and D.~Tao.
\newblock {{FusionBench}}: {{A Comprehensive Benchmark}} of {{Deep Model Fusion}}, June 2024{\natexlab{a}}.

\bibitem[Tang et~al.(2024{\natexlab{b}})Tang, Shen, Luo, Yin, Zhang, and Tao]{tang2024merging}
A.~Tang, L.~Shen, Y.~Luo, N.~Yin, L.~Zhang, and D.~Tao.
\newblock Merging multi-task models via weight-ensembling mixture of experts.
\newblock In \emph{International Conference on Machine Learning}, pages 47778--47799. PMLR, 2024{\natexlab{b}}.

\bibitem[Tang et~al.(2024{\natexlab{c}})Tang, Shen, Luo, Zhan, Hu, Du, Chen, and Tao]{tang2024parameterefficient}
A.~Tang, L.~Shen, Y.~Luo, Y.~Zhan, H.~Hu, B.~Du, Y.~Chen, and D.~Tao.
\newblock Parameter-efficient multi-task model fusion with partial linearization.
\newblock In \emph{The Twelfth International Conference on Learning Representations}, 2024{\natexlab{c}}.

\bibitem[Tang et~al.(2025)Tang, Yang, Shen, Luo, Hu, Du, and Tao]{tang2025merging}
A.~Tang, E.~Yang, L.~Shen, Y.~Luo, H.~Hu, B.~Du, and D.~Tao.
\newblock Merging models on the fly without retraining: A sequential approach to scalable continual model merging.
\newblock \emph{arXiv preprint arXiv:2501.09522}, 2025.

\bibitem[Tang et~al.(2024{\natexlab{d}})Tang, Tian, Li, He, Zhou, Zhao, Li, and Jia]{tang2024mind}
L.~Tang, Z.~Tian, K.~Li, C.~He, H.~Zhou, H.~Zhao, X.~Li, and J.~Jia.
\newblock Mind the interference: Retaining pre-trained knowledge in parameter efficient continual learning of vision-language models.
\newblock In \emph{European Conference on Computer Vision}, pages 346--365. Springer, 2024{\natexlab{d}}.

\bibitem[Utans(1996)]{utans1996weight}
J.~Utans.
\newblock Weight averaging for neural networks and local resampling schemes.
\newblock In \emph{Proc. AAAI-96 Workshop on Integrating Multiple Learned Models. AAAI Press}, pages 133--138. Citeseer, 1996.

\bibitem[Veeling et~al.(2018)Veeling, Linmans, Winkens, Cohen, and Welling]{veeling2018rotation}
B.~S. Veeling, J.~Linmans, J.~Winkens, T.~S. Cohen, and M.~Welling.
\newblock Rotation equivariant cnns for digital pathology.
\newblock \emph{arXiv preprint arXiv:1806.03962}, 2018.

\bibitem[Wang et~al.(2019)Wang, Singh, Michael, Hill, Levy, and Bowman]{wang2019glue}
A.~Wang, A.~Singh, J.~Michael, F.~Hill, O.~Levy, and S.~R. Bowman.
\newblock Glue: A multi-task benchmark and analysis platform for natural language understanding.
\newblock In \emph{7th International Conference on Learning Representations, ICLR 2019}, 2019.

\bibitem[Wang et~al.(2020)Wang, Shelhamer, Liu, Olshausen, and Darrell]{wang2020tent}
D.~Wang, E.~Shelhamer, S.~Liu, B.~Olshausen, and T.~Darrell.
\newblock Tent: Fully test-time adaptation by entropy minimization.
\newblock \emph{arXiv preprint arXiv:2006.10726}, 2020.

\bibitem[Wang et~al.(2024)Wang, Ping, Wang, Han, Chen, Liu, and Sun]{wang2024lora}
H.~Wang, B.~Ping, S.~Wang, X.~Han, Y.~Chen, Z.~Liu, and M.~Sun.
\newblock Lora-flow: Dynamic lora fusion for large language models in generative tasks.
\newblock In \emph{Proceedings of the 62nd Annual Meeting of the Association for Computational Linguistics (Volume 1: Long Papers)}, pages 12871--12882, 2024.

\bibitem[Wang et~al.()Wang, Dimitriadis, Ortiz-Jimenez, Fleuret, and Frossard]{wanglocalizing}
K.~Wang, N.~Dimitriadis, G.~Ortiz-Jimenez, F.~Fleuret, and P.~Frossard.
\newblock Localizing task information for improved model merging and compression.
\newblock In \emph{Forty-first International Conference on Machine Learning}.

\bibitem[Wang et~al.(2023)Wang, Chen, Ge, Xia, Bao, Zheng, Zhang, Gui, and Huang]{wang2023orthogonal}
X.~Wang, T.~Chen, Q.~Ge, H.~Xia, R.~Bao, R.~Zheng, Q.~Zhang, T.~Gui, and X.~Huang.
\newblock Orthogonal subspace learning for language model continual learning.
\newblock In \emph{The 2023 Conference on Empirical Methods in Natural Language Processing}, 2023.

\bibitem[Wang et~al.(2022{\natexlab{a}})Wang, Zhang, Ebrahimi, Sun, Zhang, Lee, Ren, Su, Perot, Dy, et~al.]{wang2022dualprompt}
Z.~Wang, Z.~Zhang, S.~Ebrahimi, R.~Sun, H.~Zhang, C.-Y. Lee, X.~Ren, G.~Su, V.~Perot, J.~Dy, et~al.
\newblock Dualprompt: Complementary prompting for rehearsal-free continual learning.
\newblock In \emph{European Conference on Computer Vision}, pages 631--648. Springer, 2022{\natexlab{a}}.

\bibitem[Wang et~al.(2022{\natexlab{b}})Wang, Zhang, Lee, Zhang, Sun, Ren, Su, Perot, Dy, and Pfister]{wang2022learning}
Z.~Wang, Z.~Zhang, C.-Y. Lee, H.~Zhang, R.~Sun, X.~Ren, G.~Su, V.~Perot, J.~Dy, and T.~Pfister.
\newblock Learning to prompt for continual learning.
\newblock In \emph{Proceedings of the IEEE/CVF conference on computer vision and pattern recognition}, pages 139--149, 2022{\natexlab{b}}.

\bibitem[Wortsman et~al.(2022)Wortsman, Ilharco, Kim, Li, Kornblith, Roelofs, Lopes, Hajishirzi, Farhadi, Namkoong, et~al.]{wortsman2022robust}
M.~Wortsman, G.~Ilharco, J.~W. Kim, M.~Li, S.~Kornblith, R.~Roelofs, R.~G. Lopes, H.~Hajishirzi, A.~Farhadi, H.~Namkoong, et~al.
\newblock Robust fine-tuning of zero-shot models.
\newblock In \emph{Proceedings of the IEEE/CVF conference on computer vision and pattern recognition}, pages 7959--7971, 2022.

\bibitem[Wu et~al.(2024)Wu, Huang, Wang, Meng, and Wei]{wu2024meta}
Y.~Wu, L.-K. Huang, R.~Wang, D.~Meng, and Y.~Wei.
\newblock Meta continual learning revisited: Implicitly enhancing online hessian approximation via variance reduction.
\newblock In \emph{The Twelfth International Conference on Learning Representations}, 2024.

\bibitem[Xiao et~al.(2017)Xiao, Rasul, and Vollgraf]{xiao2017fashion}
H.~Xiao, K.~Rasul, and R.~Vollgraf.
\newblock Fashion-mnist: a novel image dataset for benchmarking machine learning algorithms.
\newblock \emph{arXiv preprint arXiv:1708.07747}, 2017.

\bibitem[Xiao et~al.(2010)Xiao, Hays, Ehinger, Oliva, and Torralba]{xiao2010sun}
J.~Xiao, J.~Hays, K.~A. Ehinger, A.~Oliva, and A.~Torralba.
\newblock Sun database: Large-scale scene recognition from abbey to zoo.
\newblock In \emph{2010 IEEE Conference on Computer Vision and Pattern Recognition}, pages 3485--3492. IEEE, 2010.

\bibitem[Xu et~al.(2024)Xu, Chen, Nie, Wang, Zhuang, and Okumura]{xu2024advancing}
Y.~Xu, Y.~Chen, J.~Nie, Y.~Wang, H.~Zhuang, and M.~Okumura.
\newblock Advancing cross-domain discriminability in continual learning of vision-language models.
\newblock In \emph{The Thirty-eighth Annual Conference on Neural Information Processing Systems}, 2024.

\bibitem[Yadav et~al.(2023)Yadav, Tam, Choshen, Raffel, and Bansal]{yadav2023ties}
P.~Yadav, D.~Tam, L.~Choshen, C.~A. Raffel, and M.~Bansal.
\newblock Ties-merging: Resolving interference when merging models.
\newblock \emph{Advances in Neural Information Processing Systems}, 36:\penalty0 7093--7115, 2023.

\bibitem[Yang et~al.()Yang, Wang, Shen, Liu, Guo, Wang, and Tao]{yang2023adamerging}
E.~Yang, Z.~Wang, L.~Shen, S.~Liu, G.~Guo, X.~Wang, and D.~Tao.
\newblock Adamerging: Adaptive model merging for multi-task learning.
\newblock In \emph{The Twelfth International Conference on Learning Representations}.

\bibitem[Yang et~al.(2024)Yang, Shen, Wang, Guo, Chen, Wang, and Tao]{RepresentationSurgery_ICML_2024}
E.~Yang, L.~Shen, Z.~Wang, G.~Guo, X.~Chen, X.~Wang, and D.~Tao.
\newblock Representation surgery for multi-task model merging.
\newblock \emph{Forty-first International Conference on Machine Learning}, 2024.

\bibitem[Yu et~al.(2024{\natexlab{a}})Yu, Zhuge, Zhang, Hu, Wang, Lu, and He]{yu2024boosting}
J.~Yu, Y.~Zhuge, L.~Zhang, P.~Hu, D.~Wang, H.~Lu, and Y.~He.
\newblock Boosting continual learning of vision-language models via mixture-of-experts adapters.
\newblock In \emph{Proceedings of the IEEE/CVF Conference on Computer Vision and Pattern Recognition}, pages 23219--23230, 2024{\natexlab{a}}.

\bibitem[Yu et~al.(2024{\natexlab{b}})Yu, Yu, Yu, Huang, and Li]{yu2024language}
L.~Yu, B.~Yu, H.~Yu, F.~Huang, and Y.~Li.
\newblock Language models are super mario: Absorbing abilities from homologous models as a free lunch.
\newblock In \emph{Forty-first International Conference on Machine Learning}, 2024{\natexlab{b}}.

\bibitem[Yuan et~al.(2023)Yuan, Xie, and Li]{yuan2023robust}
L.~Yuan, B.~Xie, and S.~Li.
\newblock Robust test-time adaptation in dynamic scenarios.
\newblock In \emph{Proceedings of the IEEE/CVF Conference on Computer Vision and Pattern Recognition}, pages 15922--15932, 2023.

\bibitem[Zenke et~al.(2017)Zenke, Poole, and Ganguli]{zenke2017continual}
F.~Zenke, B.~Poole, and S.~Ganguli.
\newblock Continual learning through synaptic intelligence.
\newblock \emph{International conference on machine learning}, pages 3987--3995, 2017.

\bibitem[Zhang et~al.(2022)Zhang, Levine, and Finn]{zhang2022memo}
M.~Zhang, S.~Levine, and C.~Finn.
\newblock Memo: Test time robustness via adaptation and augmentation.
\newblock \emph{Advances in neural information processing systems}, 35:\penalty0 38629--38642, 2022.

\bibitem[Zhao et~al.(2024)Zhao, Gan, Wang, Zhou, Yang, Kuang, and Wu]{zhao2024loraretriever}
Z.~Zhao, L.~Gan, G.~Wang, W.~Zhou, H.~Yang, K.~Kuang, and F.~Wu.
\newblock Loraretriever: Input-aware lora retrieval and composition for mixed tasks in the wild.
\newblock In \emph{Findings of the Association for Computational Linguistics ACL 2024}, pages 4447--4462, 2024.

\bibitem[Zhao et~al.(2025)Zhao, Shen, Zhu, Li, Su, Wang, and Wu]{zhao2025merging}
Z.~Zhao, T.~Shen, D.~Zhu, Z.~Li, J.~Su, X.~Wang, and F.~Wu.
\newblock Merging lo{RA}s like playing {LEGO}: Pushing the modularity of lo{RA} to extremes through rank-wise clustering.
\newblock In \emph{The Thirteenth International Conference on Learning Representations}, 2025.

\bibitem[Zheng et~al.(2024)Zheng, Tang, Hao, Han, Wang, and Xu]{zheng2024adapt}
M.~Zheng, Y.~Tang, Z.~Hao, K.~Han, Y.~Wang, and C.~Xu.
\newblock Adapt without forgetting: Distill proximity from dual teachers in vision-language models.
\newblock In \emph{European Conference on Computer Vision}, pages 109--125. Springer, 2024.

\bibitem[Zheng et~al.(2023)Zheng, Ma, Wang, Qin, Yue, and You]{zheng2023preventing}
Z.~Zheng, M.~Ma, K.~Wang, Z.~Qin, X.~Yue, and Y.~You.
\newblock Preventing zero-shot transfer degradation in continual learning of vision-language models.
\newblock In \emph{Proceedings of the IEEE/CVF international conference on computer vision}, pages 19125--19136, 2023.

\bibitem[Zhou et~al.(2024)Zhou, Sun, Ye, and Zhan]{zhou2024expandable}
D.-W. Zhou, H.-L. Sun, H.-J. Ye, and D.-C. Zhan.
\newblock Expandable subspace ensemble for pre-trained model-based class-incremental learning.
\newblock In \emph{Proceedings of the IEEE/CVF Conference on Computer Vision and Pattern Recognition}, pages 23554--23564, 2024.

\end{thebibliography}
